\documentclass[10pt,twocolumn,letterpaper]{article}

\usepackage[pagenumbers]{cvpr} 

\usepackage{graphicx}
\usepackage{amsmath}
\usepackage{amssymb}
\usepackage{booktabs}

\usepackage[pagebackref,breaklinks,colorlinks]{hyperref}

\usepackage[capitalize]{cleveref}
\crefname{section}{Sec.}{Secs.}
\Crefname{section}{Section}{Sections}
\Crefname{table}{Table}{Tables}
\crefname{table}{Tab.}{Tabs.}

\usepackage[T1]{fontenc}  
\newcommand{\pmscore}[2]{{#1}\footnotesize\,$\pm$\,{#2}}
\newcommand{\boldpmscore}[2]{\textbf{#1}\footnotesize\,$\pm$\,\textbf{#2}}
\newsavebox\CBox
\def\textBF#1{\sbox\CBox{#1}\resizebox{\wd\CBox}{\ht\CBox}{\textbf{#1}}}
\usepackage[table]{xcolor}
\definecolor{tabhighlight}{HTML}{e5e5e5}

\usepackage{makecell}

\usepackage[export]{adjustbox}
\usepackage{subcaption}
\usepackage{multirow}
\usepackage{amsthm}

\newtheorem{proposition}{Proposition}

\usepackage[numbers,sort&compress]{natbib}

\def\cvprPaperID{*****} 
\def\confName{CVPR}
\def\confYear{2022}

\begin{document}

\title{Prompt Tuning with Soft Context Sharing for Vision-Language Models}

\author{Kun Ding\textsuperscript{13}\quad Ying Wang\textsuperscript{13}\quad Pengzhang Liu\textsuperscript{4}\quad Qiang Yu\textsuperscript{3}\quad Haojian Zhang\textsuperscript{2}\\Shiming Xiang\textsuperscript{13}\quad Chunhong Pan\textsuperscript{3}\\
	\textsuperscript{1}State Key Laboratory of Multimodal Artificial Intelligence Systems, CASIA, Beijing, China\\
	\textsuperscript{2}Engineering Laboratory for Intelligent Industrial Vision, CASIA, Beijing, China\\
	\textsuperscript{3}Research Center of Aerospace Information, CASIA, Beijing, China\\
	\textsuperscript{4}JD.com, Beijing, China\\
	{\tt\small \{kun.ding,\,qiang.yu,\,zhanghaojian2014\}@ia.ac.cn}\\
	{\tt\small \{ywang,\,smxiang,\,chpan\}@nlpr.ia.ac.cn\quad liupengzhang@jd.com}
}

\maketitle

\begin{abstract}
	Vision-language models have recently shown great potential on many tasks in computer vision. Meanwhile, prior work demonstrates prompt tuning designed for vision-language models could acquire superior performance on few-shot image recognition compared to linear probe, a strong baseline. In practice, many few-shot tasks are inherently correlated, particularly within specialized domains. However, such information is overlooked previously. Inspired by the fact that modeling task relationship by multi-task learning can usually boost performance, we propose a novel method SoftCPT (Soft Context Sharing for Prompt Tuning) to tune pre-trained vision-language models on multiple target few-shot tasks jointly. Specifically, we design a task-shared meta network to generate prompt context for each task using task name together with a learnable task context as input. The parameters of this meta network as well as the task context are tuned on the joint training set of all tasks. As such, the prompt context of all tasks will be shared in a soft manner. Extensive experiments across four multi-task few-shot datasets covering 44 tasks and 1593 categories demonstrate that SoftCPT significantly outperforms single-task prompt tuning methods, highlighting the effectiveness of multi-task learning for vision-language prompt tuning.
\end{abstract}

\section{Introduction}
Vision-language models (VLMs)~\cite{ALIGN,CLIP} especially the contrastive learning based ones~\cite{SLIP,CLIP} pre-trained on large-scale image-text pairs collected from the web have recently obtained extensive attention. Increasing evidences indicate that improved performance can be achieved by transferring knowledge from VLMs to various computer vision tasks, such as zero/few shot recognition~\cite{LI2023103497,CLIP,CoOp,DMPT}, detection~\cite{OpenVocaOD_VLM,PromptDet,gu2022openvocabulary} and segmentation~\cite{lueddecke22_cvpr,DenseCLIP,FISCHER2024103024}.

\begin{figure}[!t]
	\centering
	\includegraphics[width=\columnwidth]{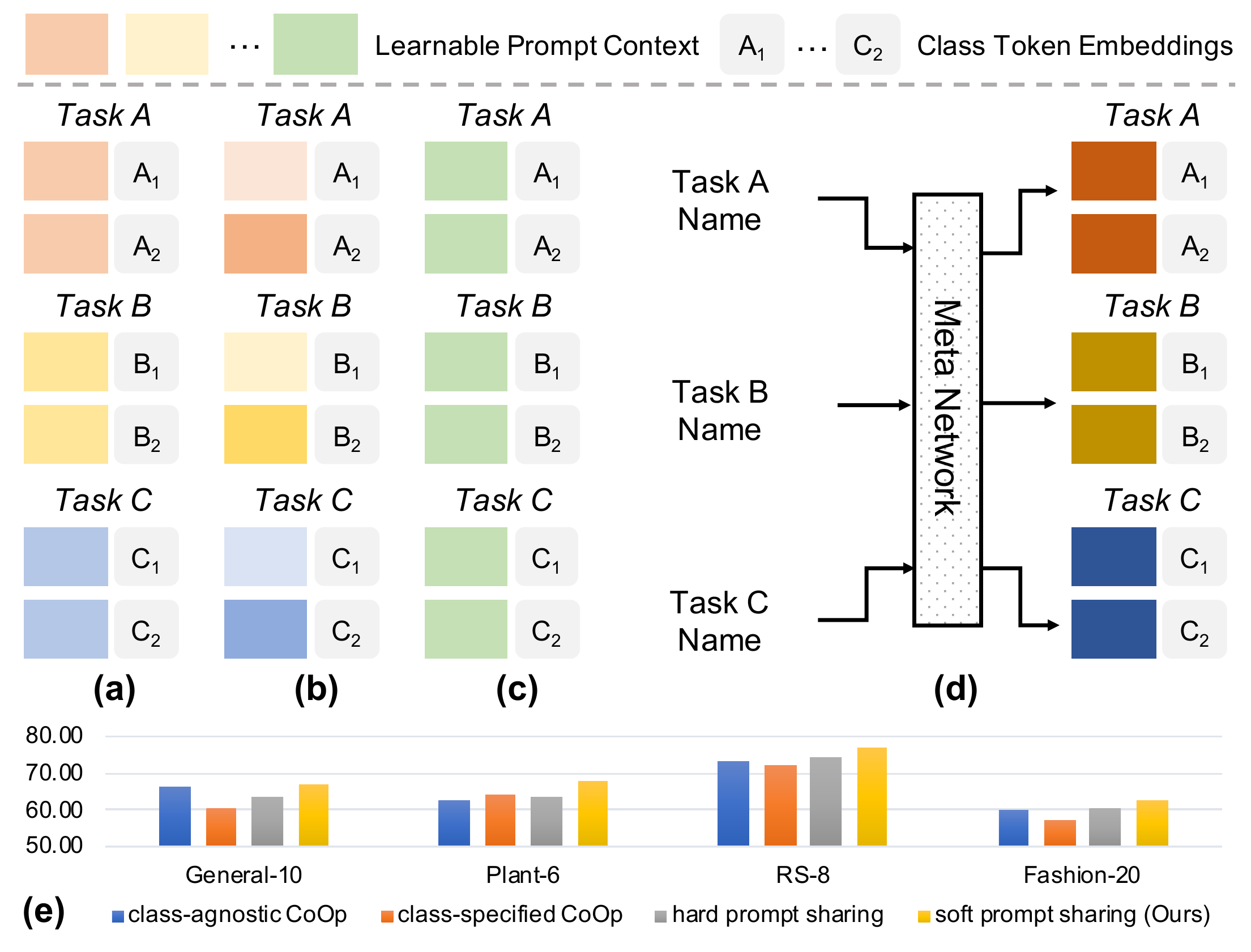}
	\caption{A conceptual comparison of different prompt tuning methods. (a) class-agnostic CoOp, (b) class-specified CoOp, (c) hard prompt sharing for CoOp, (d) our soft prompt sharing, (e) average performances on four datasets. In (a)-(d), we assume there are 2 classes per task and shared prompt contexts are in the same color.}
	\label{fig:concept}
	\vspace{-6pt}
\end{figure}

Provided with pre-trained VLMs, how to adapt their implicit knowledge to various downstream tasks becomes an essential problem. Recently, Zhou et al.~\cite{CoOp} introduced CoOp, a method that incorporates prompt tuning, originally developed in NLP~\cite{PowerofScale,PrefixTuning,GPT_Understands}, into computer vision to tackle the challenge of few-shot image recognition. In CoOp, a learnable context comprising several vectors is concatenated with the token embeddings of class names. These resultant embeddings are then fed into a text encoder to generate classifier weights. Notably, both the parameters of image and text encoder are frozen during training. Meanwhile, for a certain few-shot task, all classes can share a unified prompt context in a class-agnostic way or each class can learn its own prompt context in a class-specified way. According to their findings, CoOp can achieve improved accuracy against a strong baseline of linear probe~\cite{linear_probe}, thus paving a novel path in addressing the visual few-shot recognition problem.

The above ``pre-train then prompt-tune'' paradigm suggests that for some specialized areas we only need to keep one copy of pre-trained model and activate it to accommodate to various downstream tasks. In real-life applications, it is very natural to assume there are some relationships between these tasks. From prior research, jointly training a model on some related tasks in a multi-task manner is believed to be useful~\cite{MT_DNN_Survey,KendallGC18}. Thus, we naturally pose a question: \emph{can prompt tuning of VLMs further benefit from multi-task learning?} Although multi-task learning has already been explored in the context of prompt tuning for pure language models~\cite{AttnMixSoftPromptTune,Intrinsic_Task_Subspace}, a systematic investigation in the vision-language domain is still absent.

To study the effectiveness of multi-task prompt tuning, our initial attempt involves adapting CoOp by sharing a class-agnostic prompt context across all target tasks and tuning this context on joint multi-task few-shot training set (see Fig.~\ref{fig:concept}(c)). However, we observe that this approach of hard prompt sharing exhibited unstable performance (see Fig.~\ref{fig:concept}(e)). It underperforms on the General-10 dataset, probably due to the loose relations between the tasks on this generalized dataset. To tackle this problem, we relax the hard constraint by sharing the prompt contexts in a soft manner, leading to a novel multi-task prompt tuning technique called \textbf{SoftCPT} (\textbf{Soft} \textbf{C}ontext Sharing for \textbf{P}rompt \textbf{T}uning), as visualized by Fig.~\ref{fig:concept}(d). Specifically, a shared meta network is devised to generate per-task soft (i.e. continuous) prompt context by using task names as inputs. Owing to the capability of extracting semantic features of CLIP's text encoder, similar task names would yield analogous soft prompt contexts, enabling more effective knowledge transfer across tasks.

We conduct extensive experiments on four multi-task few-shot datasets. The results show that \textbf{SoftCPT} outperforms CoOp by \textbf{0.73\%}, \textbf{5.09\%}, \textbf{3.63\%} and \textbf{2.80\%} on four datasets, which hint multi-task prompts are beneficial. Our contributions are summarized as follows:
\begin{enumerate}
	\item A softly-shared multi-task prompt tuning method is proposed for VLMs, as far as we know, which is the first attempt to explore the effectiveness of multi-task learning in prompt tuning for VLMs.
	\item A new few-shot fashion classification dataset is constructed to test the effectiveness of multi-task prompt tuning in real industrial scenario, which will be made publicly available to facilitate future researches.
	\item Experiments on four datasets, ranging from generalized to specialized area, are conducted to study the effectiveness of SoftCPT. Our results reveal that multi-task prompts if learned in a soft sharing manner are indeed useful for the scenarios with multiple related tasks.
\end{enumerate}

The remainder of the paper is organized in the following manner: In Section~\ref{sec:related_work}, we review related work including pre-trained models, vision-language models, parameter-efficient fine-tuning, prompt tuning for vision-language models, and multi-task learning. In Section~\ref{sec:method}, we detail the proposed SoftCPT. In Section~\ref{sec:exp}, we describe the experiments and present the main findings. In Section~\ref{sec:conclusion}, we provide the overall conclusions.

\section{Related Work}
\label{sec:related_work}

\subsection{Vision-Language Models}
Existing methods in modeling vision and language signals can be roughly categorized as one-stream~\cite{VisualBert,Unicoder_VL,ViLT} and two-stream methods~\cite{CLIP,SLIP,ALIGN}. The recent contrastive learning based two-stream methods like CLIP~\cite{CLIP} have garnered significant attention. With CLIP, strong zero-shot and few-shot performance is achieved, which demonstrates the extracted visual features are of superior generalization ability.

\subsection{Parameter-Efficient Fine-Tuning}
With the rapid development of pre-training techniques, lots of pre-trained models are available, which poses a new challenge -- how to fine-tune the models in a parameter-efficient way on new tasks. Traditional fine-tuning methods~\cite{BERT,RuderH18,FineTuning_Pretrained_Language_Models,RevisitingBertFT} add task-specific heads and tunes all parameters. Although it is simple and extensively adopted, it has several obvious deficiencies~\cite{SMART,FTDistortFeats,AttnMixSoftPromptTune}. In view of this, two kinds of parameter-efficient fine-tuning methods are proposed, i.e.  prompt-based~\cite{AutoPrompt,GPT_Understands,PrefixTuning} and adapter-based~\cite{HoulsbyGJMLGAG19,RuckleGGBP0G21,CLIP_Adapter,sung2022vladapter}. The former freezes all parameters and only designs or optimizes the inputs of models on different tasks, which can be further categorized into prompt design~\cite{RaffelSRLNMZLL20,BrownMRSKDNSSAA20}, prompt search~\cite{AutoPrompt,GaoFC20} and prompt tuning~\cite{GPT_Understands,PrefixTuning,ProGrad}. Compared to prompt design and prompt search, prompt tuning learns continuous prompts which is more favorable~\cite{AutoPrompt,GaoFC20}. The latter freezes the model parameters and inserts or attaches some tuneable layers, whose parameters will be tuned on target task. Prompt-based and adapter-based method can now achieve comparable or even better performance against classical fine-tuning method~\cite{PPT,ConvBypass,CoOp,ProDA}. Recently, more and more evidences suggested that different parameter-efficient fine-tuning methods are complementary~\cite{NPS,UPT}.

\subsection{Prompt Tuning for Vision-Language Models}
CoOp proposed by Zhou et al.~\cite{CoOp} is the first method that successfully introduces prompt tuning to VLMs. Following this line, several prompt tuning methods were proposed to further improve the effectiveness and versatility for few-shot recognition task. CoCoOp~\cite{zhou2022cocoop} was proposed to address the class shift problem by introducing a meta network, which has a distinct structure to ours and its input is conditional on image feature instead of text feature. It improves the accuracy on new classes but at the expense of the accuracy on base classes, which is quite important in practical use. Subsequently, ProGrad~\cite{ProGrad} and KgCoOp~\cite{KgCoOp} were proposed to further improve the generalization ability from the viewpoints of gradient and knowledge, respectively. To further enhance the representation ability of CoOp, Chen et al. proposed PLOT~\cite{PLOT}, which is based on optimal transport. To adapt CoOp to multi-label classification, Sun et al. proposed the dual prompt tuning method DualCoOp~\cite{DualCoOp}. Besides, Lu et al. proposed to learn the distribution of prompt context~\cite{ProDA}. The proposed method ProDA learns a collection of diverse prompts and models them by multivariate Gaussian distribution. Different from all these works, our work aims at studying the effectiveness of multi-task learning in the setting of prompt tuning for VLMs.

\subsection{Multi-task Learning} 
Multi-task learning (MTL) is an important subfield in machine learning~\cite{MTL,MT_DNN_Survey,LUO2023126836}. By exploiting task relatedness, it is able to improve the performance over single-task learning. There are two dominant methods for deep multi-task learning, hard and soft parameter sharing, which learn identical and similar features, respectively. Recently, MTL has been introduced to prompt tuning in NLP. Asai et al.~\cite{AttnMixSoftPromptTune} adopted MTL to learn the parameters of attention generator and slightly improved performances were observed on language tasks. Qin et al.~\cite{Intrinsic_Task_Subspace} proposed to learn a universal intrinsic task subspace of prompts in a multi-task manner. Compared to traditional MTL that is usually performed on features or classifiers, MTL for prompt tuning must be conducted on inputs. Thus, the research of MTL for prompt tuning is necessary. Finally, although MTL is proved promising in NLP, its efficacy on prompt tuning for VLMs has not been studied yet. Our work fills this gap.

\begin{figure*}[!t]
	\centering
	\includegraphics[width=0.9\linewidth]{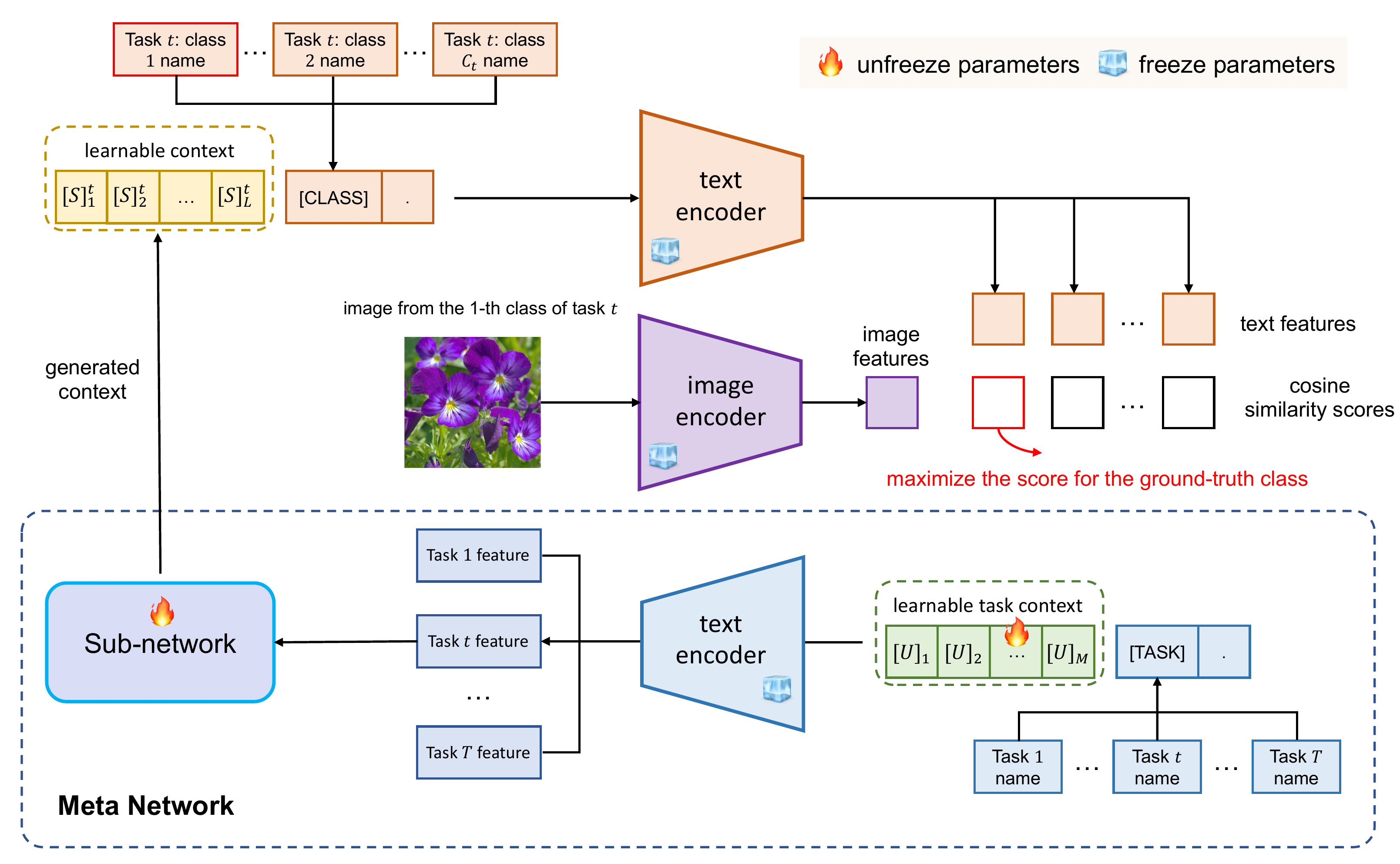}
	\caption{Illustration of the proposed multi-task prompt tuning method SoftCPT. Unlike CoOp, a meta network is introduced to produce learnable context ($[S]_1^t[S]_2^t\cdots[S]_L^t$) for each task. The meta network consists of a frozen text encoder and a learnable sub-network. The text encoder extracts task features from task names, while the sub-network transforms the task features to learnable context of class names. For model training, SoftCPT uses samples from all tasks and the loss is first computed independently for each task. The summed loss of all tasks is then used for backpropagation (ref. Eq.~\ref{eq:total_loss}). In the figure, [TASK] denotes token embeddings of task name, [CLASS] denotes token embeddings of a certain class name in a task.}
	\label{fig:network_struct}
	\vspace{-6pt}
\end{figure*}

\section{Method}
\label{sec:method}
This section first introduces the classical single-task based CoOp in Section~\ref{ssec:bg} and then presents the proposed SoftCPT to extend it for multi-task scenario in Section~\ref{ssec:softcpt}.

\subsection{Background}
\label{ssec:bg}
CoOp was proposed by Zhou et al.~\cite{CoOp}, which enables the few-shot adaption of pre-trained CLIP model for image recognition. CoOp inherits the two-stream structure from CLIP to bridge the gap between pre-training and fine-tuning. In other words, it has an image encoder denoted as $e(\cdot)$ to extract the high-level features of image and a text encoder denoted as $g(\cdot)$ to extract the text features. For zero/few-shot classification, the text features are regarded as classifier's weights. 

For few-shot classification, assume that there are $C$ classes, and the associated $C$ class name texts are available. For each class name, CoOp prepends a learnable prompt context to each class name before feeding it to text encoder. Actually, this context consists of a sequence of vectors. These vectors together with the token embeddings of the class name form the prompt template. Formally, the prompt template is represented as follows:
\begin{align}
	\boldsymbol{s}_c = [\text{S}]_1[\text{S}]_2\cdots[\text{S}]_L[\text{CLASS}]_c,
\end{align}
where each $[\text{S}]_l$ ($l\in \{1,\cdots, L\}$) is a vector with the same dimension as word embeddings of CLIP's text encoder, $[\text{CLASS}]_c$ is the token embeddings of the $c$-th ($c\in\{1,\cdots,C\}$) class name and $L$ is the number of the context tokens. By feeding $\boldsymbol{s}_c$ to the text encoder, we can obtain its text features $g(\boldsymbol{s}_c)$. Given an input image $\boldsymbol{x}$, the extracted image features can be denoted as $e(\boldsymbol{x})$. CoOp classifies this image according to the following prediction probabilities:
\begin{equation}
p(c|\boldsymbol{x})=\frac{\exp(\langle e(\boldsymbol{x}), g(\boldsymbol{s}_c) \rangle/\tau)}{\sum\nolimits_{i=1}^{C} \exp(\langle e(\boldsymbol{x}), g(\boldsymbol{s}_i) \rangle/\tau)}, c\in\{1,\cdots, C\}, 
\label{eq:pred_prob0}
\end{equation}
where $\langle\cdot,\cdot\rangle$ denotes cosine similarity and $\tau$ is a temperature parameter~\cite{CLIP}.

During model training, CoOp solely focuses on learning the context by minimizing the cross-entropy loss on the training set while maintaining all other parameters constant. Depending on the setup, we can obtain two variants of CoOp: a class-agnostic version where all classes share a common context, and a class-specific version where each class possesses its unique context. Notably, CoOp was originally designed to operate in a single-task mode, assuming that a dataset corresponds to a single classification task. However, in scenarios where multiple classification tasks are presented, a straightforward approach to extend CoOp for multi-task learning is enforcing all classes across all tasks to utilize the same context. Nevertheless, as previously mentioned, this method falls short due to its inability to effectively model task relatedness, rendering it suboptimal.

\subsection{SoftCPT}
\label{ssec:softcpt}
This work introduces SoftCPT, an extension of CoOp that effectively adapts to the multi-task scenario while facilitating the modeling of task relatedness. Fig.~\ref{fig:network_struct} illustrates the overall structure of SoftCPT. Unlike CoOp, SoftCPT incorporates a meta network, shared across all tasks, to generate task-specific contexts. This meta network enables the capture of inter-task relationships. In the following sections, we delve into the meta network's structure (Section~\ref{sssec:meta_net}), detail the prediction probability computation (Section~\ref{sssec:pred_prob}), present the multi-task training method (Section~\ref{sssec:opt}), and conduct in-depth discussions (Section~\ref{sssec:discuss}).

\subsubsection{Context Generation via Meta Network}
\label{sssec:meta_net}
The meta network serves as the cornerstone for context generation. It initiates by extracting task-specific features from the descriptive text of each task, subsequently transforming these features into the requisite context.

\textbf{Task Feature Extraction.} To extract task features, we assign a concise text description as the task name for each task, ensuring that the text effectively captures the task's essence. For instance, for the DTD dataset~\cite{DTD},  the task name ``texture classification'' is more descriptive than ``image classification''. Leveraging the strong language modeling capabilities of CLIP's text encoder, we harness it to extract text features, which serve as task features. The rationality of this method is that pre-trained language models can extract rich semantic features. Consequently, similar text descriptions yield comparable text features, translating to similar task features. Besides, similar to the motivation of CoOp, we prepend learnable task context\footnote{Note that, we use context or prompt context to denote the context combined with different class names, while task context is the context combined with different task names.} to each task name before feeding it to the text encoder, significantly boosting the plasticity of text encoder.

Specifically, assume there are $T$ tasks, let $[\text{TASK}]_t$ denote the token embeddings of the task name of the $t$-th task ($t\in \{1,\cdots,T\}$). The learnable task context consists of $M$ learnable vectors, i.e., $[\text{U}]_1[\text{U}]_2\cdots[\text{U}]_M$. The prompt template for task names can be expressed as
\begin{align}
	\boldsymbol{u}_t=[\text{U}]_1[\text{U}]_2\cdots[\text{U}]_M[\text{TASK}]_t.
	\label{eq:ta_context}
\end{align}
It is worthy nothing that, the task context in this equation does not depend on $t$. In other words, the prompt template is task-agnostic. Similar to class-specific CoOp, the above task-agnostic prompt template can also be made task-specific. For this aim, define $[\text{U}]_m^t$ as the $m$-th ($m\in\{1,\cdots, M\}$) learnable vector for task $t$, the task-specific prompt template can be expressed as
\begin{align}
	\boldsymbol{u}_t=[\text{U}]_1^t[\text{U}]_2^t\cdots[\text{U}]_M^t[\text{TASK}]_t.
	\label{eq:ts_context}
\end{align}
By feeding $\boldsymbol{u}_t$ to $g(\cdot)$, the task features for the $t$-th task is $g(\boldsymbol{u}_t)$. Note that, $g(\boldsymbol{u}_t)$ is the pooled features of the last layer of the text encoder. 

\textbf{From Task Feature to Context.} Given the vector $g(\boldsymbol{u}_t)$, we need to obtain the context for class names, which consists of $L$ vectors. To achieve this goal, we design a lightweight sub-network comprising of a linear layer and a reshaping layer. Formally, the generated context of length $L$ for the $t$-th task by the sub-network can be expressed as
\begin{align}
	[\text{S}]_1^t[\text{S}]_2^t\cdots[\text{S}]_L^t=\text{Reshape}(\boldsymbol{W}^T\cdot g(\boldsymbol{u}_t)),
	\label{eq:task_feat_to_context}
\end{align}
where $[\text{S}]_l^t$ ($l\in \{1,\cdots, L\}$) denotes the $l$-th token embedding in the context for the $t$-th task, $\boldsymbol{W}\in\mathbb{R}^{d_{txt}\times (d_{embed} L)}$ is the linear layer's transformation matrix and $\text{Reshape}(\cdot)$ converts the vector of length $d_{embed}L$ to a matrix of shape $d_{embed}\times L$. Here, $g(\boldsymbol{u}_t)$ is assumed to have a length of $d_{txt}$ and the word embedding layer's dimension of CLIP's text encoder is assumed to be $d_{embed}$.

\subsubsection{Prediction Probability Computation}
\label{sssec:pred_prob}
Based on the generated context $[\text{S}]_1^t[\text{S}]_2^t\cdots[\text{S}]_L^t$, we compute the prediction probabilities. First, the prompt template of class names for the $t$-th task can be represented as
\begin{align}
	\boldsymbol{s}_{t,c} = [\text{S}]_1^t[\text{S}]_2^t\cdots[\text{S}]_L^t[\text{CLASS}]_c,
\end{align}
where $c$ is in $\{1, \cdots, C_t\}$ with $C_t$ the number of classes of task $t$. After that, by feeding $\boldsymbol{s}_{t,c}$ to the text encoder we can obtain the corresponding text features for the $c$-th class of task $t$, i.e., $g(\boldsymbol{s}_{t,c})$. Finally, the prediction probability of an input image $\boldsymbol{x}$ belonging to class $c$ of task $t$ is
\begin{equation}
p(c|\boldsymbol{x})=\frac{\exp(\langle e(\boldsymbol{x}), g(\boldsymbol{s}_{t,c}) \rangle/\tau)}{\sum\nolimits_{i=1}^{C_t} \exp(\langle e(\boldsymbol{x}), g(\boldsymbol{s}_{t,i}) \rangle/\tau)},
\label{eq:pred_prob1}
\end{equation}
where the symbol $\tau$ carries the same meaning as it does in Eq.~\ref{eq:pred_prob0}.

\subsubsection{Multi-task Optimization}
\label{sssec:opt}
Unlike CoOp, SoftCPT is trained on the joint training set of all target tasks, each task has its own training set. Assume there are $T$ tasks, each has its own data splits. We define the dataset of task $t$ as $\mathcal{D}_t=(\mathcal{D}_t^\text{train}, \mathcal{D}_t^\text{val}, \mathcal{D}_t^\text{test})$, where $\mathcal{D}_t^\text{train}$, $\mathcal{D}_t^\text{val}$ and $\mathcal{D}_t^\text{test}$ denote the train, validation and test split, respectively. Each of these splits is a set of tuple $(\boldsymbol{x}, y)$ with $\boldsymbol{x}$ an image and $y$ the corresponding label. We further define the joint dataset $\mathcal{D}$ as $\mathcal{D}=(\mathcal{D}^\text{train}, \mathcal{D}^\text{val}, \mathcal{D}^\text{test})$, where $\mathcal{D}^\text{train}$, $\mathcal{D}^\text{val}$ and $\mathcal{D}^\text{test}$ denote the joint train, validation and test set, respectively. $\mathcal{D}^{*}$ is the union of the samples from all tasks with the same split, i.e. $\mathcal{D}^{*}=\bigcup\nolimits_{t=1}^T \mathcal{D}^{*}_t$, where $*$ is ``train'', ``val'' or ``test''. For model training, we minimize the following total loss,
\begin{equation}
\mathcal{L}=-\sum\nolimits_{t=1}^T \sum\nolimits_{(\boldsymbol{x}, y)\in \mathcal{D}_t^\text{train}} \log(p(y|\boldsymbol{x})),\label{eq:total_loss}
\end{equation}
where $p(y|\boldsymbol{x})$ is given by Eq.~\ref{eq:pred_prob1}. The parameters to be optimized include all tasks' context(s) (in Eq.~\ref{eq:ta_context} or Eq.~\ref{eq:ts_context}) and the sub-network's parameters $\boldsymbol{W}$, while all other parameters are fixed.

\subsubsection{Discussions}
\label{sssec:discuss}
\textbf{Rationality of Using the Text Encoder Twice.} It's noteworthy that SoftCPT leverages the text encoder twice: once for extracting task features and another for generating classifier weights. The text encoder's role in the meta network is crucial as it provides a prior for capturing task relatedness. Without it, learning meaningful task relationships from scratch, especially for closely related tasks, becomes challenging. Ablation experiments in Section~\ref{sec:exp} further validate this point. Additionally, we opt for the CLIP's text encoder instead of other pre-trained models to maintain simplicity and avoid incorporating additional components. However, it's worth mentioning that other pre-trained language models, such as GPT~\cite{BrownMRSKDNSSAA20} and LLaMA~\cite{LLaMA}, are also compatible with our framework.

\begin{table}[t!]
	\footnotesize
	\setlength{\tabcolsep}{2pt}
	\renewcommand{\arraystretch}{1.2}
	\centering
	\caption{Prompt context's and task's proximity in different methods. `Soft w/o TE' denotes the text encoder in SoftCPT's meta network is removed. Refer to Section~\ref{ssec:ablation} for more details.}
	\begin{tabular}{lcc}
		\toprule
		\textBF{Method}& \textBF{prompt context's proximity} & \textBF{task's proximity}\\
		\midrule
		Hard & 0 & N/A\\
		\midrule
		Soft & $\|\boldsymbol{W}^T\boldsymbol{g}_k-\boldsymbol{W}^T\boldsymbol{g}_t\|$&$\|\boldsymbol{g}_k-\boldsymbol{g}_t\|,\,\boldsymbol{g}_*=g(\boldsymbol{u}_*)$\\
		\midrule
		\multirow{2}{*}{Soft w/o TE} &\multirow{2}{*}{$\| \boldsymbol{W}^T\boldsymbol{g}_k-\boldsymbol{W}^T\boldsymbol{g}_t\|$}&$\|\boldsymbol{g}_k-\boldsymbol{g}_t\|$,\\
		&&$\boldsymbol{g}_*\, \text{is unconstrained}$\\
		\bottomrule
	\end{tabular}
	\label{tab:task_similarity}
\end{table}

\begin{table*}[t!]
	\footnotesize
	\setlength{\tabcolsep}{4pt}
	\renewcommand{\arraystretch}{1.0}
	\centering
	\caption{Detailed information of General-10, Plant-6 and RS-8. ``per-class'' is the mean per-class accuracy, ``acc'' is the top-1 accuracy and ``C'' is the number of classes for a task. ``Prompt Template" is the template used in zero-shot CLIP.}
	\begin{tabular}{lrp{5cm}rr}
		\toprule
		\textBF{Task} & \textBF{C} & \textBF{Task Name} & \textBF{Metric} & \textBF{Prompt Template}\\
		\midrule
		\textBF{General-10:}\\
		Caltech101~\cite{Caltech101} & 101 & object classification & per-class & "a photo of a \{\}."\\
		DTD~\cite{DTD} & 47 & texture classification & acc & "\{\} texture."\\
		EuroSAT~\cite{EuroSAT} & 10 & land use and land cover classification& acc & "a centered satellite photo of \{\}."\\
		FGVCAircraft~\cite{FGVCAircraft} & 102 & aircraft classification & per-class & "a photo of a \{\}, a type of aircraft."\\
		Food101~\cite{Food101} & 101 & food classification & acc & "a photo of \{\}, a type of food."\\
		Flowers102~\cite{Flowers102} & 102 & flower classification & per-class & "a photo of a \{\}, a type of flower."\\
		Oxford-Pets~\cite{OxfordPets} & 37 & pets classification & per-class & "a photo of a \{\}, a type of pet."\\
		StanfordCars~\cite{StanfordCars} & 196 & car classification & acc & "a photo of a \{\}."\\
		SUN397~\cite{SUN397} & 397 & scene classification & acc & "a photo of a \{\}."\\
		UCF101~\cite{UCF101} & 101 & action classification & acc & "a photo of a person doing \{\}."\\
		\midrule
		\textBF{Plant-6:}\\
		FruitVegetable~\cite{fruit_vegetable} & 36 & fruits and vegetables image classification & acc & "a photo of a \{\}, a type of fruit or vegetable.''\\
		KaggleFlower~\cite{kaggle_flower} & 5 & flower classification & acc & "a photo of a \{\}, a type of flower.''\\
		KaggleMushroom~\cite{kaggle_mushroom} & 9 & mushroom classification & acc & "a photo of a \{\}, a type of mushroom.''\\
		KaggleVegetable~\cite{kaggle_vegetable} & 15 & vegetable classification & acc & "a photo of a \{\}, a type of vegetable.''\\
		PlantSeedling~\cite{plant_seedling} & 12 & plant seedling classification & acc & "a photo of a \{\}, a type of seedling.''\\
		PlantVillage~\cite{plant_village} & 39 & plant leaf disease classification & acc & "a photo of a plant leaf with \{\} disease.''\\
		\midrule
		\textBF{RS-8:}\\
		AID~\cite{xia2017aid}& 30 & aerial image scene classification & acc & "an aerial photo of \{\}."\\
		RESISC45~\cite{RESISC45}& 45 & remote sensing image scene classification & acc & "a satellite photo of \{\}."\\
		OPTIMAL~\cite{wang2018scene}& 31 & Google Earth image scene classification & acc & "a satellite photo of \{\}."\\
		RSICB128~\cite{RSICB} & 45 & global-scale remote sensing image scene classification & acc & "a satellite photo of \{\}."\\
		RSSCN7~\cite{RSSCN7} & 7 & Google Earth image scene classification & acc & "a satellite photo of \{\}."\\
		NaSC-TG2~\cite{TG2} & 10 & satellite image scene classification & acc & "a satellite photo of \{\}."\\
		UCMerced~\cite{UCMerced} &21& remote sensing image scene classification with high resolution overhead image& acc & "a satellite photo of \{\}."\\
		WHURS19~\cite{Dai2011WHURS19} &19& Google Earth image scene classification & acc & "a satellite photo of \{\}."\\
		\bottomrule
	\end{tabular}
	\label{tab:data_general10_plant6_rs8}
\end{table*}

\begin{table*}[t!]
	\footnotesize
	\renewcommand{\arraystretch}{1.0}
	\setlength{\tabcolsep}{1.5pt}
	\centering
	\caption{Detailed information of Fashion-20. ``acc'' is the top-1 accuracy and ``C'' is the number of classes for a task. ``Prompt Template" is the template used in zero-shot CLIP.}
	\begin{tabular}{p{1.8cm}ccp{1.5cm}p{6cm}cr}
		\toprule
		\textBF{Task} & \textBF{C} & \textBF{\#Samples} & \textBF{Task Name} & \textBF{Class Names} & \textBF{Metric} & \textBF{Prompt Template}\\
		\midrule
		pants type & 7 & 3467 & pants type & straight-legged trousers, pencil pants, harem pants, flared trousers, wide leg pants, cargo pants, bib pants & acc & "a photo of \{\}, a type of pants.''\\
		pants length & 5 & 1614 & pants length &trousers, nine-point pants, seven-point pants, five-point pants, short pants & acc & "a photo of \{\}, a type of pants.''\\
		waist type & 3 & 564 & waist type &low waist pants, mid waist pants, high waist pants& acc&"a photo of \{\}, a type of pants.''\\
		collar type & 5 & 1838 & collar type &round collar, V-shape collar, square collar, stand collar, lapel collar& acc&"a photo of \{\}, a type of tops.''\\
		sleeve type & 5 & 921 & sleeve type &bat sleeves, puff sleeves, lantern sleeves, mandarin sleeves, flying sleeves& acc&"a photo of \{\}, a type of tops.''\\
		sleeve length & 4 & 1963 & sleeve length &sleeveless, short sleeve, mid sleeve, long sleeve& acc&"a photo of \{\}, a type of tops.''\\
		top pattern & 5 & 981 & top pattern &stripes, plaid, solid color, hand painted, broken flowers& acc&"a photo of \{\}, a type of tops.''\\
		shoe material & 4 & 1886 & shoe material &canvas shoes, leather shoes, rubber-soled shoes, plastic shoes& acc&"a photo of \{\}, a type of shoes.''\\
		shoe style & 4 & 1872 & shoe style &sneakers, boots, mules, platform shoes& acc&"a photo of \{\}, a type of shoes.''\\
		heel shape & 6 & 1326 & heel shape &square heel, horseshoe heel, slope heel, wine cup heel, muffin heel, tapered heels& acc&"a photo of \{\}, a type of shoes.''\\
		heel thickness & 2 & 354 & heel thickness &thin heel, thick heel& acc&"a photo of \{\}, a type of shoes.''\\
		heel height & 3 & 712 & heel height &low heel, medium heel, high heel& acc&"a photo of \{\}, a type of shoes.''\\
		upper height & 3 & 703 & upper height &low top shoes, medium top shoes, high top shoes& acc&"a photo of \{\}, a type of shoes.''\\
		toe cap style & 3 & 963 & toe cap style &round head, square head, pointed head& acc&"a photo of \{\}, a type of shoes.''\\
		hat style & 6 & 1791 & hat style &berets, casquette, clochehat, sailor cap, octagon hat, Chinese skullcap& acc&"a photo of \{\}, a type of hat.''\\
		socks length & 3 & 491 & socks length &stockings, mid socks, short socks& acc&"a photo of \{\}, a type of socks.''\\
		socks type & 2 & 342 & socks type &foot socks, pantyhose& acc&"a photo of \{\}, a type of socks.''\\
		skirt length & 4 & 1493 & skirt length & short skirt, knee length skirt, over knee skirt, long skirt& acc&"a photo of \{\}, a type of skirt.''\\
		number of button rows & 2 & 394 & number of button rows&single-breasted, double-breasted& acc&"a photo of \{\}, a type of tops.''\\
		underwear style & 2 & 434 & underwear style & briefs, boxer shorts& acc&"a photo of \{\}, a type of underwear.''\\
		\bottomrule
	\end{tabular}
	\label{tab:data_fashion20}
\end{table*}

\textbf{Improving Generalization Ability.} SoftCPT exploits multi-task learning, which helps to improve the generalization ability of prompts. This point can be explained from the viewpoint of gradient descent. In the following, we consider the task-specified case with a linear sub-network and obtain the following proposition.
\begin{proposition}
	For task-specified case with a linear sub-network, the new context for class names after one step update of SGD can be represented as
	\begin{equation}
	\boldsymbol{S}'_t\simeq\boldsymbol{S}_t-\eta\sum\nolimits_{k=1}^T \boldsymbol{d}_{k} \langle\boldsymbol{g}_{k},\boldsymbol{g}_t\rangle + \boldsymbol{C}_t, \label{eq:st_new}
	\end{equation}
	where $\boldsymbol{S}'_t$, $\boldsymbol{S}_t$, $\eta$, $\boldsymbol{d}_{t}$, $\boldsymbol{g}_t$ and $\boldsymbol{C}_t$ are the new context of task $t$, old context of task $t$, learning rate, gradient of loss of task $t$ with respect to $\boldsymbol{S}_t$, task features of task $t$ and a constant matrix only related to task $t$, respectively. 
\end{proposition}
\begin{proof}
	Let us denote the task features for the $t$-th task as $\boldsymbol{g}_t$, which is generated by transforming the task context of the $t$-th task $[\text{U}]_1^t[\text{U}]_2^t\cdots[\text{U}]_M^t \triangleq \boldsymbol{U}_t$ with a vector-valued function $\phi(\cdot)$, i.e., $\boldsymbol{g}_t=\phi(\boldsymbol{U}_t)$. Note that, task names are implicitly included in $\phi$. With a linear sub-network, the generated context for class names can be represented as $\boldsymbol{W}^T\boldsymbol{g}_t\triangleq \boldsymbol{S}_t$, where $\boldsymbol{W}$ is a transform matrix. Without loss of generality, we assume that $\boldsymbol{g}_t$ is $L_2$-normalized and ignore the reshaping operator as both normalization and reshaping can be absorbed into the other parts of loss computation\footnote{Under this assumption, the relation between $\boldsymbol{S}_t$ and $[\text{S}]_1^t[\text{S}]_2^t\cdots[\text{S}]_L^t$ is $[\text{S}]_1^t[\text{S}]_2^t\cdots[\text{S}]_L^t=\text{Reshape}(\boldsymbol{S}_t)$.}. The total loss function can be represented as
	\begin{equation}
	\mathcal{L}=\sum\nolimits_{t=1}^{T} \sum\nolimits_{(\boldsymbol{x}, y)\in \mathcal{D}^\text{train}_t} loss(\boldsymbol{x}, y; \{\boldsymbol{\omega}_{t,c}\}),
	\end{equation}
	where $\boldsymbol{\omega}_{t,c}=g(\boldsymbol{s}_{t,c})$ is the classifier weight vector of the $c$-th class of task $t$, and $loss()$ denotes the loss function defined on one sample $(\boldsymbol{x}, y)$ sampling from the train set $\mathcal{D}^\text{train}_t$ of task $t$. $T$ is the total number of tasks. 
	
	The gradient of $\mathcal{L}$ with respect to $\boldsymbol{W}$ is
	\begin{align}
		\frac{\partial \mathcal{L}}{\partial \boldsymbol{W}}=\sum\nolimits_{t=1}^T \boldsymbol{g}_t \boldsymbol{d}_t^T,\;
		\boldsymbol{d}_t=\sum\nolimits_{(\boldsymbol{x}, y)\in \mathcal{D}^\text{train}_t} \frac{\partial loss}{\partial \boldsymbol{S}_t}.\label{eq:LwrtQ}
	\end{align}
	The gradient of $\mathcal{L}$ with respect to $\boldsymbol{U}_t$ is
	\begin{align}
		\frac{\partial \mathcal{L}}{\partial \boldsymbol{U}_t}=\sum\nolimits_{(\boldsymbol{x}, y)\in \mathcal{D}^\text{train}_t} \frac{\partial loss}{\partial \boldsymbol{U}_t} \triangleq \boldsymbol{m}_t,
	\end{align}
	Note that, given $\boldsymbol{W}$, this gradient is only dependent on data from the $t$-th task.
	
	The one step update equation of SGD in terms of $\boldsymbol{W}$ and $\boldsymbol{U}_t$ can be represented as
	\begin{align}
		\boldsymbol{W}'&=\boldsymbol{W}-\eta \sum\nolimits_{t=1}^T \boldsymbol{g}_t\boldsymbol{d}_t^T, \label{eq:Q_new}\\
		\boldsymbol{U}'_t&=\boldsymbol{U}_t-\eta \boldsymbol{m}_t,\label{eq:ut_update}
	\end{align}
	where $\eta$ is learning rate. By substituting Eq.~\ref{eq:ut_update} into $\boldsymbol{g}'_t=\phi(\boldsymbol{U}'_t)$, the new task features of task $t$ are represented as
	\begin{align}
		\boldsymbol{g}'_t&=\phi\left(\boldsymbol{U}_t-\eta \boldsymbol{m}_t\right)\approx \boldsymbol{g}_t - \eta \boldsymbol{A}_t \boldsymbol{m}_t.\label{eq:gt_new}
	\end{align}
	The approximation in Eq.~\ref{eq:gt_new} is based on the first-order Taylor expansion and $\boldsymbol{A}_t$ denotes the first-order gradient matrix of $\boldsymbol{g}_t$ with respect to $\boldsymbol{U}_t$. By substituting Eq.~\ref{eq:Q_new} and Eq.~\ref{eq:gt_new} into $\boldsymbol{S}'_t=\boldsymbol{W}'^T\boldsymbol{g}'_t$, we obtain
	\begin{align}
		\hspace{-14pt}
		\boldsymbol{S}'_t&\approx \left(\boldsymbol{W}^T-\eta \sum\nolimits_{k=1}^T \boldsymbol{d}_k\boldsymbol{g}_k^T\right)\left(\boldsymbol{g}_t -\eta \boldsymbol{A}_t \boldsymbol{m}_t\right)\nonumber\\
		&=\boldsymbol{S}_t-\eta \sum\limits_{k=1}^T \boldsymbol{d}_k\boldsymbol{g}_k^T\boldsymbol{g}_t-\eta\boldsymbol{W}^T\boldsymbol{A}_t \boldsymbol{m}_t+\eta^2 \sum\limits_{k=1}^T \boldsymbol{d}_k\boldsymbol{g}_k^T\boldsymbol{A}_t\boldsymbol{m}_t\nonumber\\
		&\approx \boldsymbol{S}_t-\eta\boldsymbol{W}^T\boldsymbol{A}_t\boldsymbol{m}_t-\eta\sum\nolimits_{k=1}^T \boldsymbol{d}_k\boldsymbol{g}_k^T\boldsymbol{g}_t.\label{eq:st_new1}
	\end{align}
	The second approximation in Eq.~\ref{eq:st_new1} holds as $\eta^2$ is much smaller than $\eta$. When the task context $\boldsymbol{U}_t$ is not learned or of length zero, Eq.~\ref{eq:st_new1} can be simplified into
	\begin{equation}
	\boldsymbol{S}'_t=\boldsymbol{S}_t-\eta\sum\nolimits_{k=1}^T \boldsymbol{d}_k\boldsymbol{g}_k^T\boldsymbol{g}_t.\label{eq:st_new2}
	\end{equation}
	Moreover, the equation is accurate. The proof is completed. 
\end{proof}
From Eq.~\ref{eq:st_new}, the learned context for class names of a certain task is not only determined by its own data but also by data from the other tasks. Actually, tasks with more similar task features will contribute more to this task. By multi-task training, a task can borrow information from related tasks to regularize its context, which is expected to be helpful for better generalization.

\textbf{Modeling Task Relationship.} In Table~\ref{tab:task_similarity}, we compare the prompt context's and task's proximity to facilitate more in-depth understanding of the proposed method. For two different tasks $k$ and $t$, the learned contexts of our method (denoted as `Soft' in Table~\ref{tab:task_similarity}) are $\text{Reshape}(\boldsymbol{W}^T\boldsymbol{g}_k)$ and $\text{Reshape}(\boldsymbol{W}^T\boldsymbol{g}_t)$, respectively. We use the $L_2$ distance to denote the distance between two contexts, i.e.,  $\|\text{Reshape}(\boldsymbol{W}^T\boldsymbol{g}_k)-\text{Reshape}(\boldsymbol{W}^T\boldsymbol{g}_t)\|=\|\boldsymbol{W}^T\boldsymbol{g}_t-\boldsymbol{W}^T\boldsymbol{g}_t\|$. Similarly, the distance between two contexts learned by our method without using text encoder (denoted as `Soft w/o TE' in Table~\ref{tab:task_similarity}) is also $\|\boldsymbol{W}^T\boldsymbol{g}_t-\boldsymbol{W}^T\boldsymbol{g}_t\|$. As for hard prompt sharing (denoted as `Hard' in Table~\ref{tab:task_similarity}), the prompt contexts for all tasks are identical. As such, the distance between any two contexts is 0. In the last column of the table, we use the $L_2$ distance to measure the proximity between two task features $\boldsymbol{g}_k$ and $\boldsymbol{g}_t$ ($k\neq t$), i.e., $\|\boldsymbol{g}_k-\boldsymbol{g}_t\|$. The difference between `Soft' and `Soft w/o TE' is that the former enforces the task feature to be in the output space of the text encoder ($\boldsymbol{g}_t=g(\boldsymbol{u}_t)$) while the latter does not impose any constraint on task features. We have two observations about SoftCPT.

First, SoftCPT imposes effective soft constraint between contexts. According to~\cite{VirmauxS18}, any MLPs with 1-Lipschitz activation functions are Lipschitz continuous. In other words, there exists a constant $\kappa$ such that $\|\boldsymbol{W}^T\boldsymbol{g}_k-\boldsymbol{W}^T\boldsymbol{g}_t\|\leq \kappa\|\boldsymbol{g}_k-\boldsymbol{g}_t\|$. This implies that the more similar the two tasks are, the more similar the prompt context will be. This can be explained as a soft constraint between prompt context of different tasks. As for `Soft w/o TE' (ref. Table~\ref{tab:task_similarity}), there is no constraint on task features, thus, the distance $\|\boldsymbol{g}_k-\boldsymbol{g}_t\|$ can be arbitrarily large, which will make the above soft constraint less effective. For `Soft', the task features are outputs of text encoders, thus, the distance between two task features is bounded, making the above soft constraint more effective.

Second, SoftCPT can model task relatedness. As pre-trained language model can effectively extract semantic features of text, task names with higher similarity yield closer task features, reflected in smaller distances between their respective feature vectors, i.e. $\|\boldsymbol{g}_k-\boldsymbol{g}_t\|$. Consequently, the generated contexts tend to cluster closely together. Conversely, when tasks exhibit greater dissimilarity, the upper bound on the distance between their contexts increases, enabling the two learned contexts move away from each other. Overall, the modeling of SoftCPT can effectively capture and reflect the task relatedness.

\section{Experiment}
\label{sec:exp}
We conduct extensive experiments on four datasets. In Section~\ref{ssec:datasets}, the datasets are detailed. In Section~\ref{ssec:compared_method}, we detail the various methods used for comparison. In Section~\ref{ssec:imp_detail}, the implementation details of different methods are presented. After that, Section~\ref{ssec:ablation} conducts ablation study. Lastly, in Section~\ref{ssec:main_results}, the comparison between our method and existing methods is presented.

\subsection{Datasets}
\label{ssec:datasets}
To more comprehensively verify the effectiveness of the proposed method, four multi-task few-shot datasets are constructed, ranging from generalized to specialized domain.

\textbf{General-10} is a generalized dataset built on ten publicly available classification datasets, including both coarse-grained and fine-grained tasks: Caltech101~\cite{Caltech101}, DTD~\cite{DTD}, EuroSAT~\cite{EuroSAT}, FGVCAircraft~\cite{FGVCAircraft}, Food101~\cite{Food101}, Flowers102~\cite{Flowers102}, Oxford-Pets~\cite{OxfordPets}, SUN397~\cite{SUN397}, StanfordCars~\cite{StanfordCars}, and UCF101~\cite{UCF101}. These datasets are also adopted in CoOp.  

\textbf{Plant-6} is a specialized dataset built on six public plant-related datasets: FruitVegetable~\cite{fruit_vegetable}, KaggleFlower~\cite{kaggle_flower}, KaggleMushroom~\cite{kaggle_mushroom}, PlantSeedling~\cite{plant_seedling}, KaggleVegetable~\cite{kaggle_vegetable}, and PlantVillage~\cite{plant_village}. 

\textbf{RS-8} is a specialized dataset built on eight public satellite/aerial image classification datasets: AID~\cite{xia2017aid}, RESISC45~\cite{RESISC45},  OPTIMAL~\cite{wang2018scene}, RSICB128~\cite{RSICB}, RSSCN7~\cite{RSSCN7}, NaSC-TG2~\cite{TG2}, UCMerced~\cite{UCMerced}, and WHURS19~\cite{Dai2011WHURS19}. 

\textbf{Fashion-20} is a specialized dataset for fashion classification (a key technique for product data governance on E-commerce platform), which is collected by us and has about 24K images in 20 tasks. All the images are obtained by searching on web using pre-defined keywords. Before human labeling, data cleaning is performed, e.g., removing similar or too small images. Example images of Fashion-20 are shown in Fig.~\ref{fig:fashion20_example_images}. 

All datasets are split according to the method in CoOp. More details of the datasets are given in Table~\ref{tab:data_general10_plant6_rs8} and Table~\ref{tab:data_fashion20}.
\begin{figure}[t!]
	\centering
	\includegraphics[width=0.95\columnwidth]{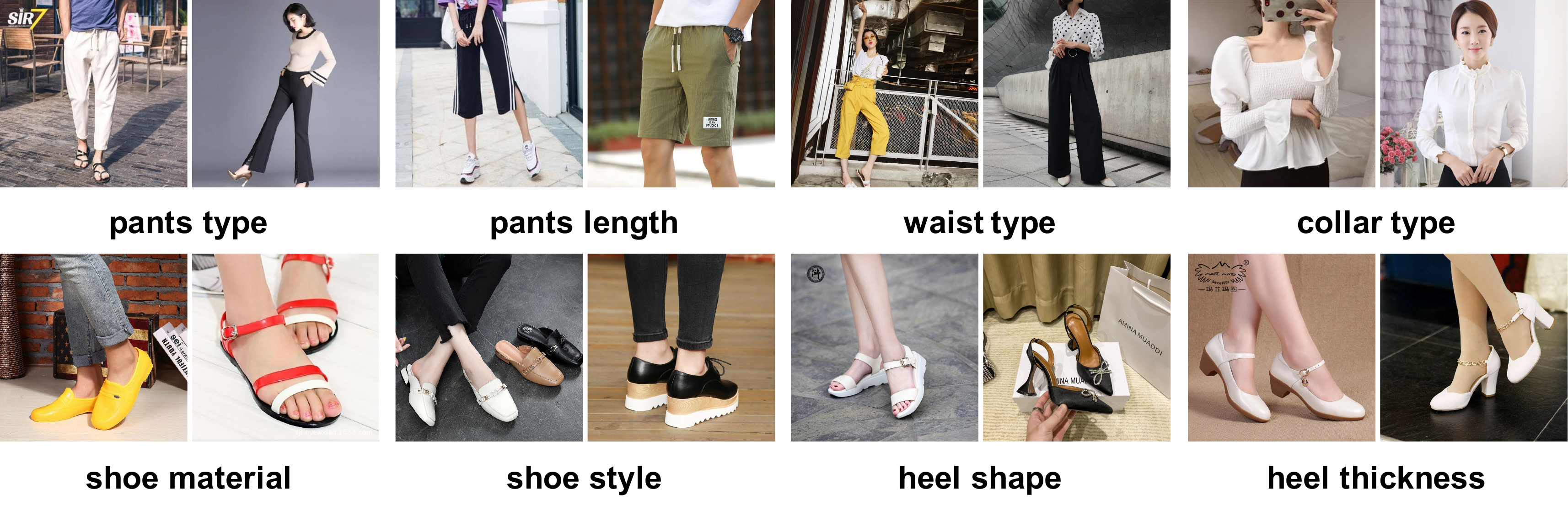}
	\caption{Some example images from the Fashion-20 dataset. Note that only 8 out of the 20 tasks are being displayed.}
	\label{fig:fashion20_example_images}
	\vspace{-5pt}
\end{figure}

\begin{figure}[t!]
	\centering
	\includegraphics[width=0.95\columnwidth]{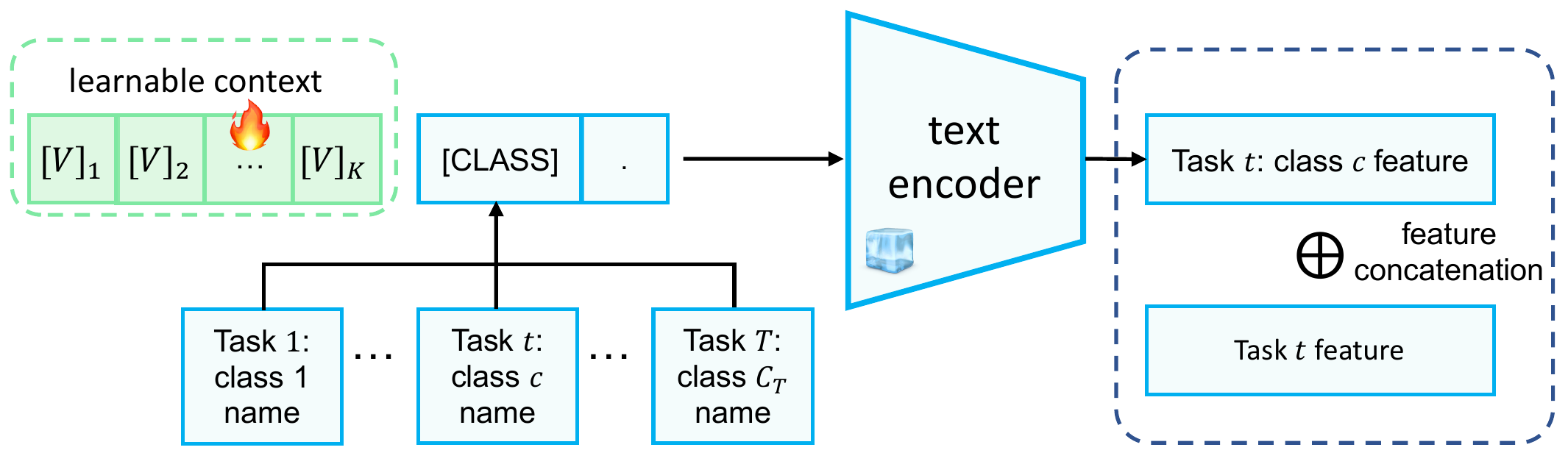}
	\caption{Illustration of adding class features to task features. The context $[V]_1[V]_2\cdots[V]_K$ of length $K$ is learned. For the $t$-th task, there are $C_t$ classes.}
	\label{fig:add_clss_feat}
\end{figure}

\subsection{Compared Methods}
\label{ssec:compared_method}
For clarity, we list the involved methods here: 1) \textbf{LP-CLIP} denotes the linear probe CLIP reported in CLIP~\cite{CLIP}, which is believed to be a strong baseline; 2) \textbf{CoOp-CA} denotes the class-agnostic CoOp applied to each task separately; 3) \textbf{CoOp-CS} denotes the class-specific CoOp applied to each task separately; 4) \textbf{CoOp-MT} denotes a shared prompt context for all tasks and all classes is learned in CoOp in a multi-task manner; 5) \textbf{ZS-CLIP} denotes zero-shot CLIP~\cite{CLIP}; 6) \textbf{CoCoOp}~\cite{zhou2022cocoop}; 7) \textbf{ProGrad}~\cite{ProGrad} is an optimized method of CoOp with good generalization ability; 8) \textbf{KgCoOp}~\cite{KgCoOp} is another improved method of CoOp; 9) \textbf{PLOT}~\cite{PLOT} is a state-of-the-art prompt tuning method based on optimal transport; 10) \textbf{ProtoNet}~\cite{ProtoNet} is a classical few-shot learning method.

As for our method, except the task-agnostic and task-specific SoftCPT mentioned in Section~\ref{sssec:meta_net}, we can also concatenate class name's features to task features before sending it to the sub-network. Fig.~\ref{fig:add_clss_feat} visualizes the process, in which the class features are extracted by CLIP's text encoder with a learnable context $[{V}]_1[{V}]_2\cdots[{V}]_K$ with $K$ the context length. Note that, this context can be class-agnostic or class-specific. Combined with the previous two cases, there are six possible cases for SoftCPT in total. Table~\ref{tab:config} lists all the configurations of SoftCPT. Note that the values in brackets are the number of different contexts ($[V]_1\cdots[V]_{K}$) or different task contexts ($[U]_1\cdots[U]_{M}$) to be learned. When $[{V}]_1[{V}]_2\cdots[{V}]_K$ is class-agnostic, only a single context of length $K$ is learned. While it is class-specific, $C$ different contexts need to be learned given there are $C$ classes in total.

\subsection{Implementation Details}
\label{ssec:imp_detail}
For all methods except ProtoNet, CLIP is used as the pre-trained vision-language model and ResNet-50 is the default image encoder. For SoftCPT, the task names in Table~\ref{tab:data_general10_plant6_rs8} and Table~\ref{tab:data_fashion20} are used. For the initialization of task context, we use the simplest random method because task-specified initialization is unnecessary as reported in a recent work~\cite{Waywardness}. For model optimization, SGD with an initial learning rate of 0.002 and cosine decay scheduler are used. The batch size is 32 and 50 epochs are trained on all datasets.

The experimental setup of classical few-shot methods like ProtoNet is different to ours. To make a valid comparison, we use the off-shelf dataset miniImageNet (100 classes) as the large train set for ProtoNet, and use the train set in CoOp for the few-shot reference set for it. We train ProtoNet on miniImageNet with 400 epochs and 8-way $k$-shot ($k= 1, 2, 4, 8, 16$) episode.

For model evaluation, the test set performances are reported for 1, 2, 4, 8, 16 shots. For few-shot classification, to reduce the variance, the scores over three trials with different seeds are averaged like CoOp. For General-10, the same evaluation metrics to CLIP are used. For all tasks of Plant-6, RS-8 and Fashion-20, top-1 accuracy is adopted. Besides linear probe and zero-shot recognition, test set metrics of the last checkpoint are reported, thus no validation is involved.

\begin{table}[thb!]
	\setlength{\tabcolsep}{7pt}
	\centering
	\footnotesize
	\caption{Six configurations of SoftCPT. $C$ is the total number of classes in all tasks. $[{V}]_1[{V}]_2\cdots[{V}]_K$ is the context introduced in Fig.~\ref{fig:add_clss_feat}. $[{U}]_1[{U}]_2\cdots[{U}]_M$ denotes the task context in Eq.~\ref{eq:ta_context} or Eq.~\ref{eq:ts_context}. CA: class-agnostic, CS: class-specified, TA: task-agnostic, TS: task-specified.}
	\begin{tabular}{ccc}
		\toprule
		$[{V}]_1[{V}]_2\cdots[{V}]_K$ & $[{U}]_1[{U}]_2\cdots[{U}]_M$ & Name \\
		\midrule
		class-agnostic ($1$) & task-agnostic ($1$) & SoftCPT-CATA\\
		class-specified ($C$) & task-agnostic ($1$) & SoftCPT-CSTA \\
		class-agnostic ($1$) & task-specified ($T$) & SoftCPT-CATS\\
		class-specified ($C$) & task-specified ($T$) & SoftCPT-CSTS\\
		\midrule
		N/A & task-agnostic ($1$) & SoftCPT-NATA\\
		N/A & task-specified ($T$) & SoftCPT-NATS\\
		\bottomrule
	\end{tabular}
	\label{tab:config}
\end{table}

\begin{table}[t!]
	\footnotesize
	\setlength{\tabcolsep}{5pt}
	\centering
	\caption{Average scores (\%) on there datasets under six configurations listed in Table~\ref{tab:config}. The results imply that adding class features to task features does not bring significant benefits.}
	\begin{tabular}{cccc}
		\toprule
		\textBF{Config} & \textBF{General-10} & \textBF{Plant-6} & \textBF{Fashion-20} \\
		\midrule
		SoftCPT-CATA & \pmscore{63.71}{0.69} & \pmscore{66.87}{1.33} & \pmscore{58.51}{1.39}\\
		SoftCPT-CSTA & \pmscore{63.11}{0.63} & {\boldpmscore{67.60}{0.91}} & \pmscore{57.81}{1.66}\\
		SoftCPT-CATS & \pmscore{64.40}{0.57} & \pmscore{67.04}{0.80} & \pmscore{59.38}{1.18}\\
		SoftCPT-CSTS & \pmscore{62.70}{0.72} & \pmscore{67.08}{1.77} & \pmscore{57.55}{1.16}\\
		\midrule
		SoftCPT-NATA & {\pmscore{67.16}{0.38}} & \pmscore{67.22}{1.11} & {\boldpmscore{62.14}{0.81}}\\
		SoftCPT-NATS & {\boldpmscore{67.22}{0.35}} & {\pmscore{67.47}{0.71}} & {\pmscore{61.95}{0.62}}\\
		\bottomrule
	\end{tabular}
	\label{tab:diff_configs}
\end{table}

\begin{table}[t!]
	\footnotesize
	\setlength{\tabcolsep}{13pt}
	\centering
	\caption{Average results (\%) of SoftCPT-NATA with different combinations of train and test sets. $\mathcal{D}^{\text{train/test}}_\text{G}$ and $\mathcal{D}^{\text{train/test}}_\text{P}$ denote a split of General-10 and Plant-6, respectively. ``$\mathcal{D}^\text{train}_\text{G}$ or $\mathcal{D}^\text{train}_\text{P}$'' denotes two models are trained on General-10 and Plant-6 independently.}
	\begin{tabular}{ccc}
		\toprule
		$\boldsymbol{\mathcal{D}}^\text{train}$ & $\boldsymbol{\mathcal{D}}^\text{test}$ & \textBF{Score} \\
		\midrule
		$\mathcal{D}^\text{train}_\text{G}$ & $\mathcal{D}^\text{test}_\text{G}$ & {\boldpmscore{67.16}{0.38}}\\
		$\mathcal{D}^\text{train}_\text{G}\cup\mathcal{D}^\text{train}_\text{P}$ & $\mathcal{D}^\text{test}_\text{G}$ & \pmscore{67.15}{0.34}\\
		\midrule
		$\mathcal{D}^\text{train}_\text{P}$ & $\mathcal{D}^\text{test}_\text{P}$ & \pmscore{67.22}{1.11}\\
		$\mathcal{D}^\text{train}_\text{G}\cup\mathcal{D}^\text{train}_\text{P}$ & $\mathcal{D}^\text{test}_\text{P}$ &{\boldpmscore{68.41}{0.95}}\\
		\midrule
		$\mathcal{D}^\text{train}_\text{G}$ or $\mathcal{D}^\text{train}_\text{P}$ & $\mathcal{D}^\text{test}_\text{G}\cup\mathcal{D}^\text{test}_\text{P}$ & \pmscore{67.18}{0.33}\\
		$\mathcal{D}^\text{train}_\text{G}\cup\mathcal{D}^\text{train}_\text{P}$ & $\mathcal{D}^\text{test}_\text{G}\cup\mathcal{D}^\text{test}_\text{P}$ & {\boldpmscore{67.62}{0.39}}\\
		\bottomrule
	\end{tabular}
	\label{tab:merge_more_tasks}
\end{table}

\begin{table}[t!]
	\footnotesize
	\setlength{\tabcolsep}{2.5pt}
	\centering
	\caption{Mean scores (\%) with two kinds of sub-network structures on Plant-6: Linear and MLP. $r$ denotes the reduction ratio of the first linear layer's dimension in MLP.}
	\begin{tabular}{lcccc}
		\toprule
		& \multirow{2}{*}{\textBF{Linear}} & \multicolumn{3}{c}{\textbf{MLP}}\\
		& & $r=1$ & $r=2$ & $r=4$\\
		\midrule
		Score & {\boldpmscore{67.22}{1.11}} & \pmscore{67.15}{1.03} &  \pmscore{66.65}{1.43} & \pmscore{66.06}{1.75}\\
		\bottomrule
	\end{tabular}
	\label{tab:diff_subnetwork}
\end{table}

\subsection{Ablation Study}
\label{ssec:ablation}
In this part, we investigate the effectiveness of the main components in our method.

\textbf{Is Adding Class Feature to Task Feature Necessary?}
We conduct experiments of SoftCPT with different configurations listed in Table~\ref{tab:config}. While learning the context $[{V}]_1\cdots[{V}]_{K}$, more gradients should be stored, thus more GPU memory is required. On General-10, more than 40G GPU memory is required for SoftCPT-CSTA, which is not affordable on most GPU cards. To reduce memory, 10\% of all class names are randomly sampled for loss computation. While on the other datasets, no sampling is used as there are not too many classes. The results are reported in Table~\ref{tab:diff_configs}. Clearly, without adding the class feature, SoftCPT already achieves good performances on all datasets. Considering the high computational cost, SoftCPT-CATA, SoftCPT-CSTA, SoftCPT-CATS and SoftCPT-CSTS are not recommended.

\textbf{Effect of Merging More Tasks.}
We merge General-10 and Plant-6 into a new dataset and conduct experiments to study if combining more tasks can bring extra benefits. The experimental results are listed in Table~\ref{tab:merge_more_tasks}. By excluding the influence of using task names as extra inputs, multi-task learning does contribute to improve overall performance.

\begin{table}[t!]
	\footnotesize
	\setlength{\tabcolsep}{11pt}
	\centering
	\caption{Comparison to prompt transferring. Rows 3-5 correspond to three different methods of transferring prompts between different tasks. }
	\begin{tabular}{lcc}
		\toprule
		\textBF{Method} & \textBF{General-10} (\%) & \textBF{Plant-6 (\%)} \\
		\midrule
		CoOp-CA & \pmscore{66.39}{0.52}& \pmscore{62.59}{2.33}\\
		Oracle & \pmscore{66.46}{0.44}& \pmscore{63.39}{1.63} \\
		EnsFeat & \pmscore{53.52}{0.40}& \pmscore{50.38}{0.66} \\
		EnsPred & \pmscore{54.65}{0.31}& \pmscore{51.27}{1.32}\\
		\midrule
		SoftCPT-NATA & {\boldpmscore{67.12}{0.39}}& {\boldpmscore{67.68}{1.18}}\\
		SoftCPT-NATS & {\boldpmscore{67.04}{0.23}}& {\boldpmscore{67.13}{1.15}}\\
		\bottomrule
	\end{tabular}
	\label{tab:result_transfer}
\end{table}

\begin{table}[t!]
	\footnotesize
	\setlength{\tabcolsep}{6pt}
	\centering
	\caption{Average scores (\%) over shots in the base-to-new generalization setting. `H' denotes the harmonic mean between the accuracies on base and new set.}
	\begin{tabular}{m{1.6cm}lcccc}
		\toprule
		\textBF{Data} & \textBF{Method} & \textBF{Base} & \textBF{New} & \textBF{H} \\
		\midrule
		\multirow{4}{*}{\parbox{1.6cm}{General-10}} & ZS-CLIP & 65.04 & {\textBF{69.62}}& 67.25\\
		& CoOp-CA & 71.40 & 60.15 & 65.29 \\
		& CoCoOp & 70.32 & 65.20 & {\textBF{67.66}}\\
		& SoftCPT-NATA & {\textBF{72.27}} & 62.29& 66.91\\
		\midrule
		\multirow{4}{*}{\parbox{1.6cm}{Plant-6}} & ZS-CLIP & 61.09 & {\textBF{60.64}} & 60.86\\
		& CoOp-CA & 71.80 & 58.25 & 64.32\\
		& CoCoOp & 70.03 & 58.81 & 63.93\\
		& SoftCPT-NATA & {\textBF{76.36}} & 59.40 & {\textBF{66.82}}\\
		\bottomrule
	\end{tabular}
	\label{tab:class_generalization}
\end{table}

\begin{figure*}[!t]
	\centering
	\hspace{10pt}
	\begin{subfigure}[t]{0.41\linewidth}
			\includegraphics[height=5.0cm]{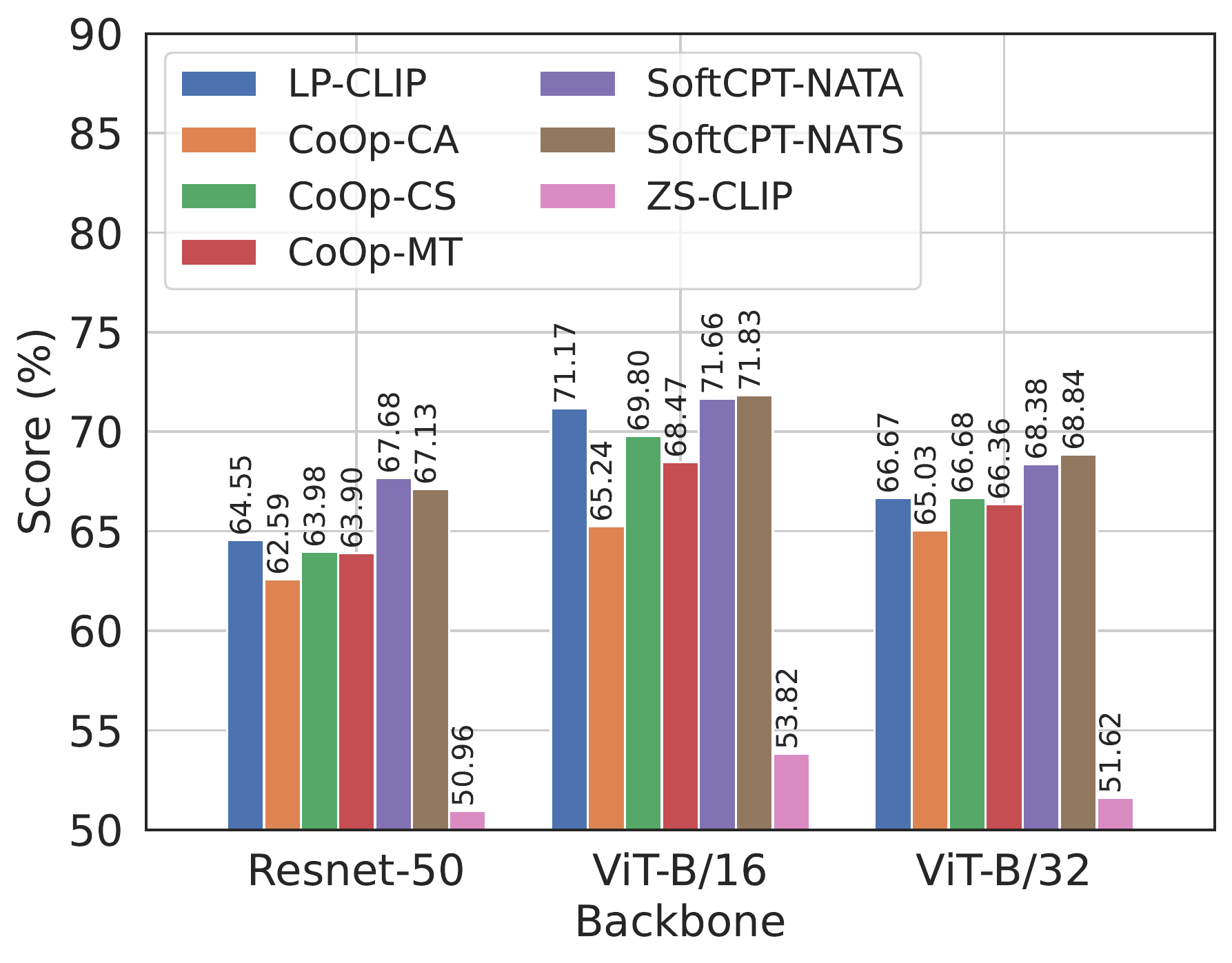}%
			\caption{}
		    \label{fig:diff_backbones_and_prompt_len_a}
	\end{subfigure}
	\begin{subfigure}[t]{0.41\linewidth}
		\includegraphics[height=5.0cm]{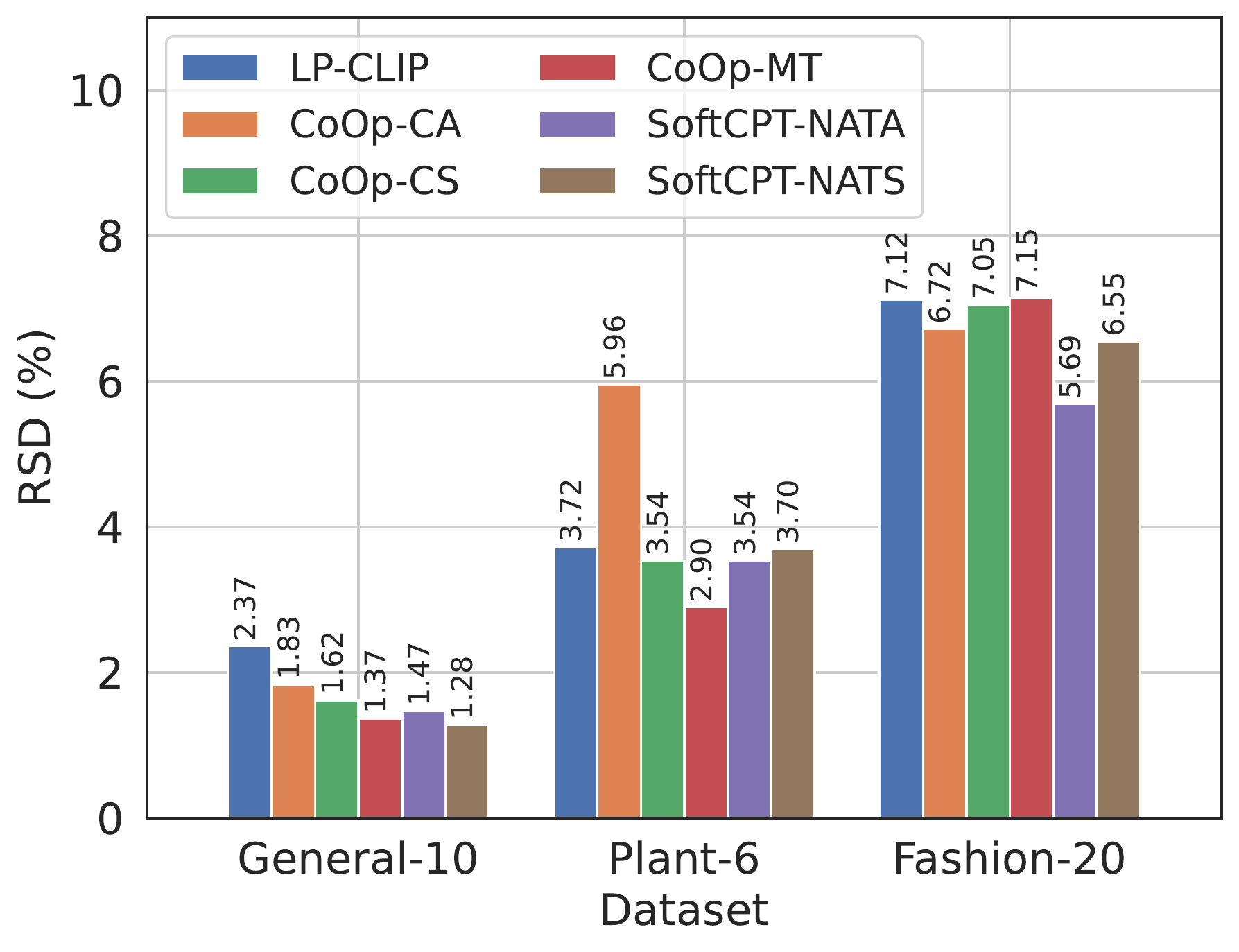}%
		\caption{}
		\label{fig:diff_backbones_and_prompt_len_b}
	\end{subfigure}
	\\
	\begin{subfigure}[t]{0.2\linewidth}
			\includegraphics[height=4.5cm,width=3.3cm]{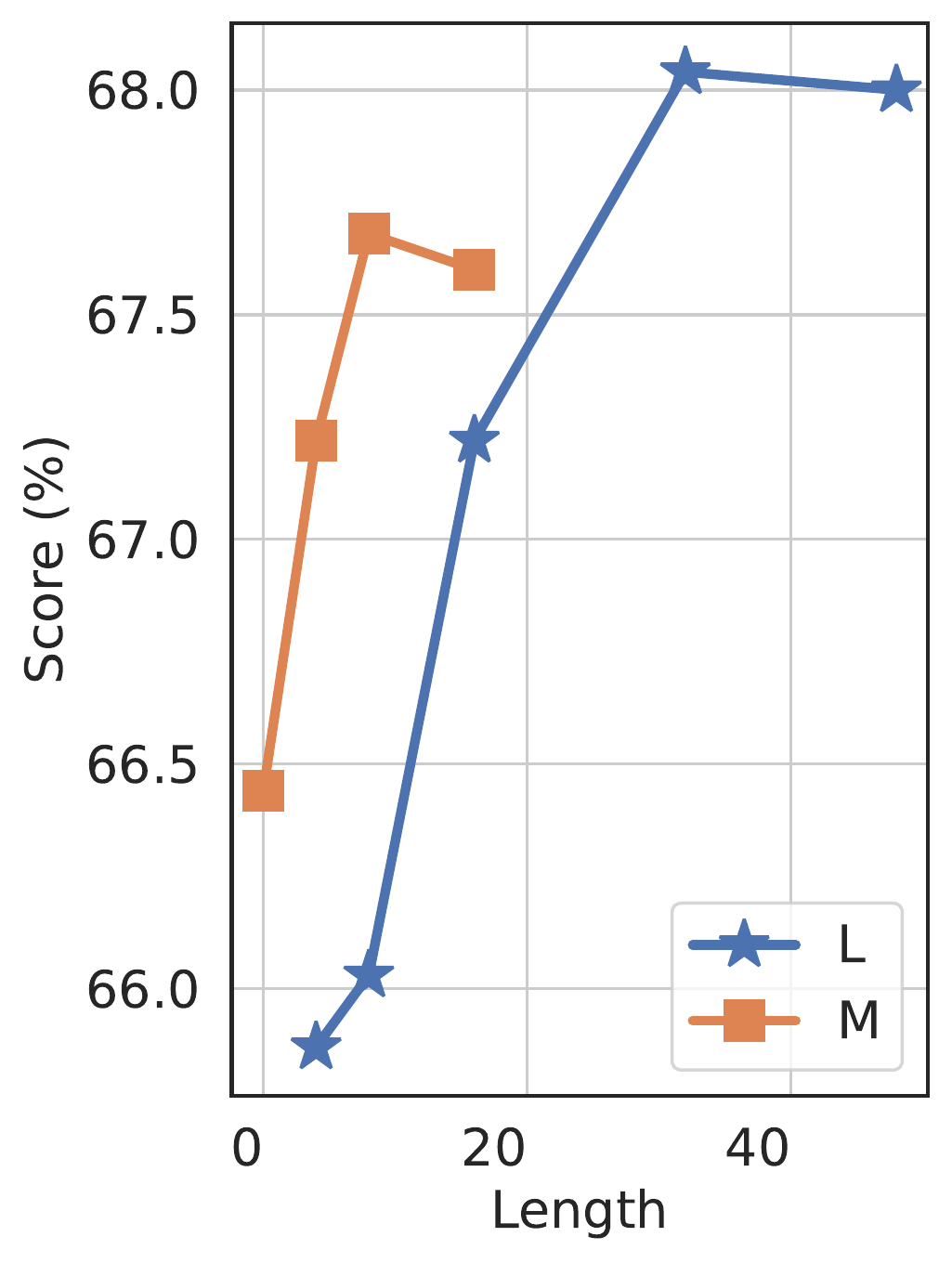}%
			\caption{}
		\label{fig:diff_backbones_and_prompt_len_c}
	\end{subfigure}
	\begin{subfigure}[t]{0.2\linewidth}
			\includegraphics[height=4.5cm,width=3.3cm]{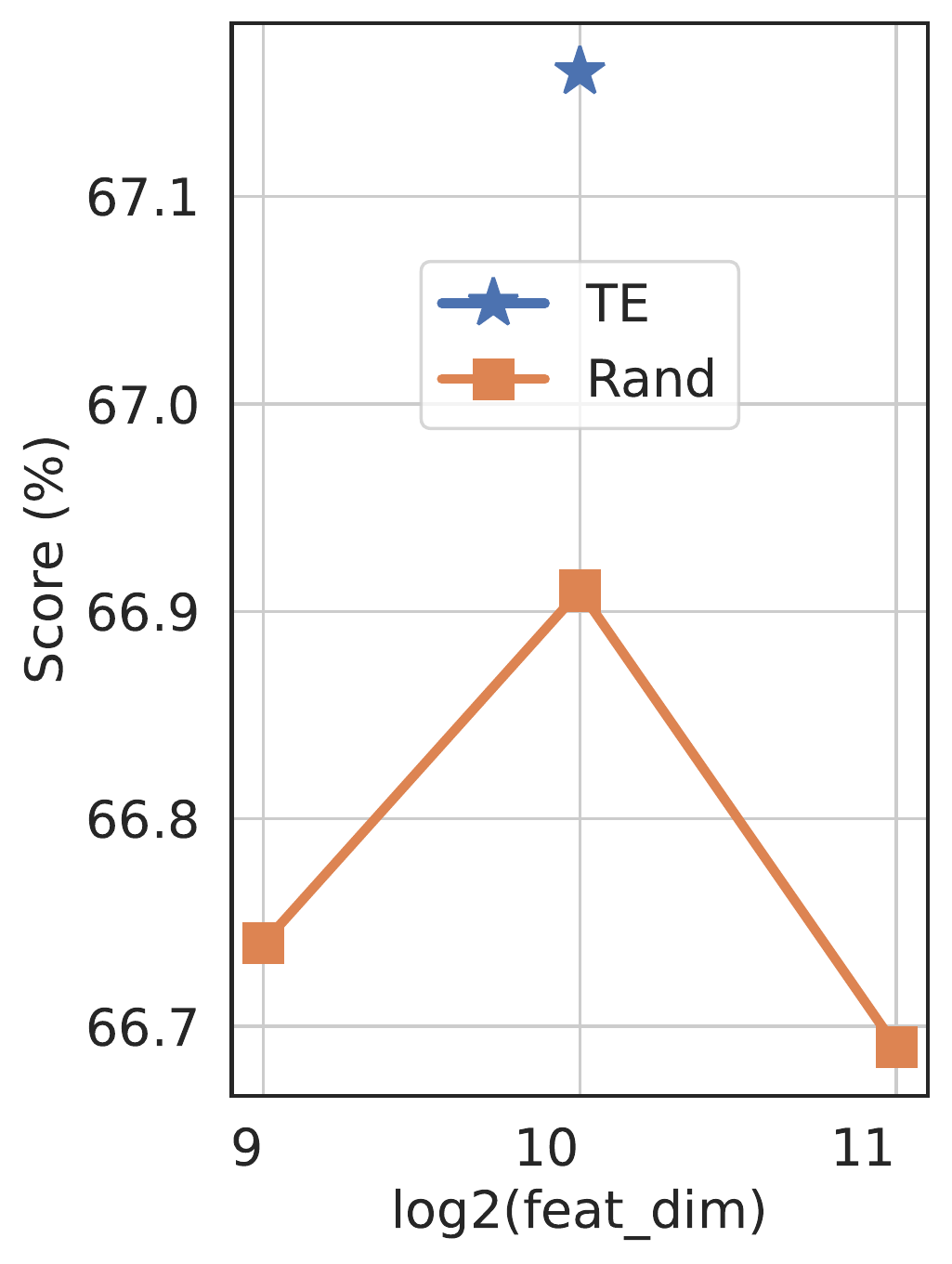}%
			\caption{}
		\label{fig:diff_backbones_and_prompt_len_d}
	\end{subfigure}
	\begin{subfigure}[t]{0.2\linewidth}
			\includegraphics[height=4.5cm,width=3.3cm]{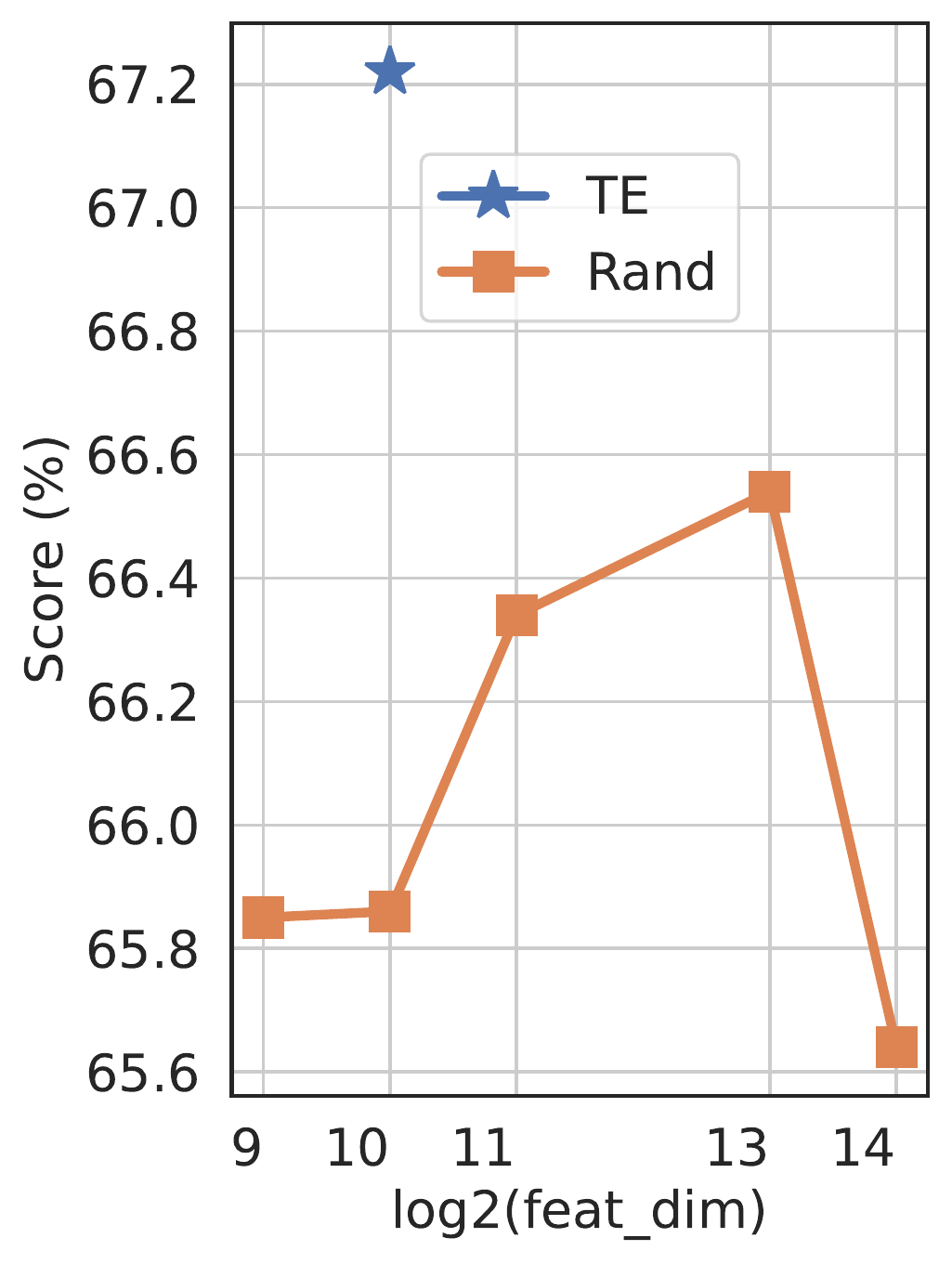}%
			\caption{}
		\label{fig:diff_backbones_and_prompt_len_e}
	\end{subfigure}
	\begin{subfigure}[t]{0.2\linewidth}
			\includegraphics[height=4.5cm,width=3.3cm]{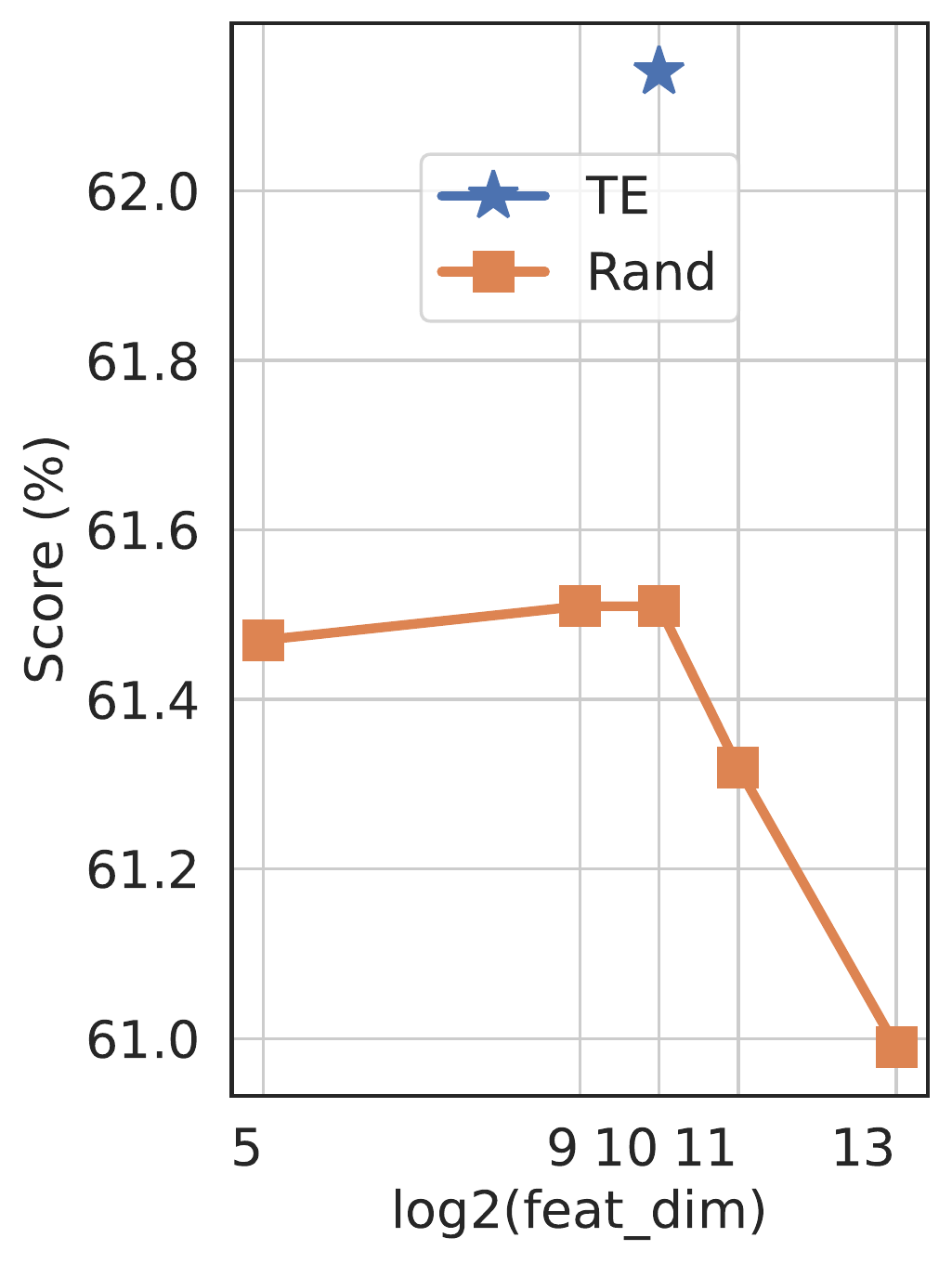}%
			\caption{}
		\label{fig:diff_backbones_and_prompt_len_f}
	\end{subfigure}
	\caption{(a) Results on Plant-6 with different backbones, (b) Relative Standard Deviation (RSD) on three datasets, (c) impact of varying prompt length on Plant-6, (d)-(f) comparison between task features learned with text encoder (`TE') and without text encoder from scratch (`Rand') on General-10, Plant-6 and Fashion-20, respectively. `feat\_dim' is the task feature's dimension.}
	\label{fig:diff_backbones_and_prompt_len}
	\vspace{-5pt}
\end{figure*}

\textbf{Different Sub-networks.} The effect of using different sub-networks for SoftCPT-NATA is studied. Two structures are tried: linear and MLP. For MLP, the structure is ``Linear-BN-ReLU-Linear''. The first linear layer reduces the input dimension by a ratio of $r$. The results are listed in Table~\ref{tab:diff_subnetwork}, which indicate that more complex sub-network does not work better than the simplest linear structure. One possible reason is that larger and deeper models are hard to fit in the few-shot setting.

\textbf{Prompt Transferring.} For CoOp, we check the transferability from the prompt context learned on one task to another. Transferring the prompt context is also a way to exploit the relation between tasks. We compare it to our method to see if it works better than multi-task prompts. Besides, if there exists the transferability, the test set score after transferring can be seen as a golden indicator for evaluating the similarity/transferability between tasks. The General-10 and Plant-6 datasets are adopted here due to their small task number. We train CoOp-CA on all $T$ (10 or 6) tasks and obtain $T$ prompt contexts. Three different transfer methods are tried: 1) \textbf{Oracle} computes $T$ test scores using all $T$ prompt contexts as initialization for each task, and returns the highest score as final test score for this task; 2) \textbf{EnsFeat} conducts a mean ensembling of $T$ classifier weights computed by using all prompt contexts as initialization and then uses the ensembled classifier for prediction; 3) \textbf{EnsPred} is similar to \textbf{EnsFeat} but the mean ensembling of prediction probabilities is used. We do not further fine-tune the prompts as this needs extra iterations compared to SoftCPT.

The results are listed in Table~\ref{tab:result_transfer}, from which the following conclusions can be drawn: 1) the prompts of CoOp are transferable; 2) transferring prompts is able to improve the performance slightly (by comparing CoOp-CA and Oracle); 3) the proposed multi-task prompt learning is superior to the simple prompt transferring mentioned above; 4) the pairwise transferring score (ref. Fig.~\ref{fig:corr_map}(a)) could be viewed as a good indicator in evaluating transferability between tasks due to the intuitive consistency to true task relatedness.

\textbf{Generalization From Base to New Classes.} We here study the generalizability across classes following the same setting in CoCoOp~\cite{zhou2022cocoop}. From the results in Table~\ref{tab:class_generalization}, we can see that SoftCPT outperforms CoOp on both base and new classes, demonstrating the generalizability of multi-task prompt tuning. Besides, SoftCPT outperforms CoCoOp significantly on Plant-6. We conjecture that when the distribution of image features deviates from the pre-trained distribution a lot, CoCoOp is hard to generate prompts of desired generalizability.

\textbf{Different Visual Backbones.} Fig.~\ref{fig:diff_backbones_and_prompt_len}(a) shows the results with different image encoders. SoftCPT acquires the best results with various backbones. Meanwhile, with ViT-B/16~\cite{ViT} and ViT-B/32 as the backbone, task-specific contexts for task name are slightly better than task-agnostic ones.

\textbf{Variance Reduction.} We compare the Relative Standard Deviation (RSD) of different methods. Due to the normalization by mean value, RSD can eliminate the effect of varying hardness of different tasks. Here, RSD is defined as $RSD=\frac{1}{5}\sum\nolimits_{i\in\{1,2,4,8,16\}} {STD_i}/{Score_i}$, where $STD_i$ and $Score_i$ are the mean over per-task standard deviations and mean over per-task mean scores for the $i$-shot setting. The results on different datasets are shown in Fig.~\ref{fig:diff_backbones_and_prompt_len}(b). We can see that, compared to other methods, SoftCPT acquires lower RSDs, which imply that it has more stable performance. We conjecture that such a stability is coming from the joint constraint on parameters from multiple tasks.

\textbf{Prompt Length.} The curves with varying prompt length $L$ and $M$ for SoftCPT-NATA on Plant-6 are shown in Fig.~\ref{fig:diff_backbones_and_prompt_len}(c). With increasing lengths, the score increases gradually at first and drops slightly after reaching the peak. For Plant-6, setting $L$ as 16 and $M$ as 8 can reach a good trade-off between speed and accuracy.

\textbf{Learning Task Features without Text Encoder.} In the proposed meta network, the task features are extracted by the text encoder. By removing the text encoder from meta network and using a learnable vector for each task, the task features can be learned from scratch. The results with and without the text encoder on three datasets are shown in Fig.~\ref{fig:diff_backbones_and_prompt_len}(d)-(f). We can observe that task features learned with text encoder are better, especially on datasets with tightly related tasks. This implies that the text encoder is important for learning the task relatedness for scarce training data. We conjecture the underlying reason is that the text encoder provides a strong prior of task relatedness.

\begin{table*}[t!]
	\footnotesize
	\setlength{\tabcolsep}{4pt}
	\renewcommand{\arraystretch}{1}
	\centering
	\caption{Comparison of state-of-the-art methods. For each dataset, parameter count and training time are computed on all its tasks. Four shots are used while measuring the training time.}
	\begin{tabular}{lcccccccc}
		\toprule
		\textBF{Method} & \#Params & Train Time/min& \textBF{1-shot} & \textBF{2-shot} & \textBF{4-shot} & \textBF{8-shot} & \textBF{16-shot} & \textBF{Avg.}\\
		
		\midrule
		\multicolumn{4}{l}{\textBF{General-10 (10 tasks, 1191 classes)}}\\
		ZS-CLIP~\cite{CLIP} &--&--&--&--&--&--&--&\pmscore{58.29}{0.00}\\
		ProtoNet~\cite{ProtoNet} &--&--& \pmscore{10.14}{0.44} & \pmscore{14.27}{0.28} & \pmscore{19.06}{0.15} & \pmscore{24.48}{0.17} & \pmscore{28.04}{0.19} & \pmscore{19.20}{0.25}\\
		LP-CLIP~\cite{CLIP} &--&--&\pmscore{37.31}{0.48}&\pmscore{49.05}{0.34}&\pmscore{59.15}{0.76}&\pmscore{66.53}{0.41}&\pmscore{72.77}{0.07}&\pmscore{56.96}{0.41} \\
		CoOp-CA~\cite{CoOp}&81,920 &49.53&{\pmscore{59.51}{0.81}}&\pmscore{61.83}{0.72}&\pmscore{66.22}{0.58}&\pmscore{70.39}{0.34}&\pmscore{73.98}{0.17}&\pmscore{66.39}{0.52}\\
		CoOp-CS~\cite{CoOp}&9,756,672&50.53&\pmscore{45.50}{0.08}&\pmscore{53.98}{0.55}&\pmscore{62.00}{0.27}&\pmscore{68.44}{0.21}&\pmscore{73.58}{0.04}&\pmscore{60.70}{0.23}\\
		CoOp-MT&8,192&66.72&\pmscore{58.35}{0.73}&\pmscore{60.86}{0.13}&\pmscore{63.04}{0.43}&\pmscore{66.06}{0.48}&\pmscore{68.83}{0.51}&\pmscore{63.43}{0.46}\\
		CoCoOp~\cite{zhou2022cocoop} &1,009,280&250.57& \pmscore{59.28}{0.31} & \pmscore{62.11}{0.50} & \pmscore{64.23}{0.56} & \pmscore{65.30}{2.22} & \pmscore{68.75}{0.20} & \pmscore{63.93}{0.76} \\
		ProGrad~\cite{ProGrad} &81,920&65.18&\underline{\pmscore{60.99}{0.31}} & \pmscore{62.95}{0.94} & \pmscore{66.67}{0.78} & \pmscore{70.77}{0.40} & \pmscore{73.69}{0.13} & \pmscore{67.01}{0.51}\\
		KgCoOp~\cite{KgCoOp} &20,480&41.53&\boldpmscore{62.73}{0.67} & {\boldpmscore{63.80}{0.13}} & \pmscore{65.40}{0.24} & \pmscore{68.64}{0.31} & \pmscore{70.82}{0.17} & \pmscore{66.28}{0.30}\\
		PLOT~\cite{PLOT} & 163,840 &60.36&{\pmscore{60.68}{0.14}} & \pmscore{63.11}{0.64} & \pmscore{67.13}{0.51} & \pmscore{70.79}{0.34} & \pmscore{74.20}{0.18}&\boldpmscore{67.18}{0.36}\\
		SoftCPT*&8,396,800&56.72&\pmscore{58.66}{0.77}&\pmscore{62.40}{0.77}&{\boldpmscore{67.79}{0.58}}&{\boldpmscore{71.46}{0.21}}&\underline{\pmscore{74.22}{0.10}}&\pmscore{66.91}{0.49}\\
		SoftCPT-NATS&8,421,376&59.03&\pmscore{59.23}{0.45}&{\pmscore{62.70}{0.21}}&\underline{\pmscore{67.67}{0.29}}&\pmscore{71.10}{0.06}&{\boldpmscore{74.51}{0.15}}&{\pmscore{67.04}{0.23}}\\
		\rowcolor{tabhighlight}
		SoftCPT-NATA &8,392,704&59.62& {\pmscore{59.47}{0.76}} & \underline{\pmscore{63.18}{0.39}} & {\pmscore{67.65}{0.53}} & \underline{\pmscore{71.15}{0.07}} & \pmscore{74.17}{0.21} & \underline{\pmscore{67.12}{0.39}}\\
		
		\midrule
		\multicolumn{4}{l}{\textBF{Plant-6 (6 tasks, 116 classes)}}\\
		ZS-CLIP~\cite{CLIP}&--&--&--&--&--&--&--&\pmscore{50.96}{0.00}\\
		ProtoNet~\cite{ProtoNet}&--&--& \pmscore{21.33}{0.58} & \pmscore{28.62}{0.60}& \pmscore{37.42}{0.62}&\pmscore{44.72}{0.28}&\pmscore{49.67}{0.17}&\pmscore{36.35}{0.45}\\
		LP-CLIP~\cite{CLIP}&--&--&\pmscore{44.54}{1.39}&\pmscore{56.45}{0.57}&\pmscore{66.91}{0.87}&\pmscore{74.30}{0.63}&{\boldpmscore{80.55}{0.21}}&\pmscore{64.55}{0.73}\\
		CoOp-CA~\cite{CoOp}&49,152&13.98&\pmscore{50.39}{4.45}&\pmscore{56.40}{1.52}&\pmscore{62.09}{2.24}&\pmscore{68.81}{2.13}&\pmscore{75.27}{1.29}&\pmscore{62.59}{2.33}\\
		CoOp-CS~\cite{CoOp}&950,272&12.53&\pmscore{43.15}{1.80}&\pmscore{55.78}{0.41}&\pmscore{66.54}{1.12}&\underline{\pmscore{74.71}{1.19}}&\underline{\pmscore{79.70}{0.18}}&\pmscore{63.98}{0.94}\\
		CoOp-MT&8,192&7.85&\pmscore{53.14}{1.11}&\pmscore{59.08}{0.76}&\pmscore{63.37}{0.85}&\pmscore{69.95}{0.51}&\pmscore{73.94}{0.61}&\pmscore{63.90}{0.77}\\
		CoCoOp~\cite{zhou2022cocoop}&605,568&25.42& {\pmscore{54.08}{2.04}} & \pmscore{57.55}{0.63} & \pmscore{60.42}{0.71} & \pmscore{65.12}{1.40} & \pmscore{69.45}{1.99} & \pmscore{61.33}{1.35} \\
		ProGrad~\cite{ProGrad}&49,152&9.08& \pmscore{52.56}{1.70} & \pmscore{58.14}{0.99} & \pmscore{60.74}{2.34} & \pmscore{67.81}{2.95} & \pmscore{73.04}{1.62} & \pmscore{62.46}{1.92}\\
		KgCoOp~\cite{KgCoOp}&12,288 &8.05&\boldpmscore{56.49}{0.56} & \pmscore{59.13}{0.84} & \pmscore{60.25}{1.00} & \pmscore{66.42}{0.70} & \pmscore{69.45}{0.50} & \pmscore{62.35}{0.72}\\
		PLOT~\cite{PLOT}& 196,608&8.57&\underline{\pmscore{55.82}{1.42}} & \pmscore{59.85}{0.88} & \pmscore{64.11}{1.44} & \pmscore{72.82}{0.44} & \pmscore{78.09}{0.36} & \pmscore{66.14}{0.91}\\
		SoftCPT*&8,396,800&5.50&\pmscore{53.99}{1.38}&\pmscore{58.82}{4.27}&\pmscore{66.65}{0.46}&\pmscore{72.19}{0.87}&\pmscore{77.67}{0.24}&\pmscore{65.86}{1.44}\\
		SoftCPT-NATS&8,421,376&7.50&\pmscore{52.28}{2.50}&{\boldpmscore{61.64}{1.42}}&\underline{\pmscore{68.28}{1.07}}&\pmscore{74.34}{0.26}&\pmscore{79.11}{0.49}&\underline{\pmscore{67.13}{1.15}}\\
		\rowcolor{tabhighlight}
		SoftCPT-NATA &8,392,704&5.83& {\pmscore{54.94}{2.33}} & \underline{\pmscore{61.58}{1.57}} & {\boldpmscore{68.72}{1.53}} & {\boldpmscore{74.75}{0.22}} & \pmscore{78.42}{0.27} & {\boldpmscore{67.68}{1.18}}\\
		
		\midrule
		\multicolumn{4}{l}{\textBF{RS-8 (8 tasks, 208 classes)}}\\
		ZS-CLIP~\cite{CLIP}&--&--&--&--&--&--&--&\pmscore{47.91}{0.00}\\
		ProtoNet~\cite{ProtoNet}&--&--& \pmscore{32.75}{0.49} & \pmscore{42.26}{0.32} & \pmscore{49.68}{0.22} & \pmscore{57.32}{0.68} & \pmscore{62.24}{0.23} & \pmscore{48.85}{0.39}\\
		LP-CLIP~\cite{CLIP}&--&--& \pmscore{54.04}{0.77} & \pmscore{64.92}{2.19} & \pmscore{76.14}{0.89} & \pmscore{82.52}{0.73} & \boldpmscore{87.39}{0.69} & \pmscore{73.00}{1.05}\\
		CoOp-CA~\cite{CoOp}&65,536&9.13& \pmscore{59.39}{1.68} & \pmscore{66.55}{0.63} & \pmscore{74.18}{0.51} & \pmscore{80.82}{0.25} & \pmscore{85.22}{0.26} & \pmscore{73.23}{0.67}\\
		CoOp-CS~\cite{CoOp}&1,703,936&9.48& \pmscore{52.09}{2.05} & \pmscore{65.45}{1.17} & \pmscore{74.74}{0.14} & \pmscore{82.19}{0.32} & \underline{\pmscore{87.05}{0.25}} & \pmscore{72.30}{0.79}\\
		CoOp-MT &8,192&7.27&\pmscore{62.44}{0.24} & \pmscore{70.39}{0.88} & \pmscore{74.74}{1.23} & \pmscore{79.41}{0.63} & \pmscore{83.62}{0.39} & \pmscore{74.12}{0.67}\\
		CoCoOp~\cite{zhou2022cocoop}&807,424&20.35&\pmscore{55.41}{0.74} & \pmscore{58.22}{2.59} & \pmscore{66.92}{0.61} & \pmscore{72.12}{0.86} & \pmscore{76.97}{0.03}&\pmscore{65.93}{0.97}\\
		ProGrad~\cite{ProGrad}&65,536&10.77& \pmscore{58.62}{0.76} & \pmscore{64.70}{1.47} & \pmscore{71.38}{2.12} & \pmscore{77.79}{1.22} & \pmscore{82.86}{0.29}&\pmscore{71.07}{1.17}\\
		KgCoOp~\cite{KgCoOp}&16,384&9.45&\pmscore{59.95}{0.59} & \pmscore{60.66}{0.36} & \pmscore{66.92}{0.15} & \pmscore{73.67}{0.49} & \pmscore{78.93}{0.07} & \pmscore{68.03}{0.33}\\
		PLOT~\cite{PLOT}& 262,144 &11.15& \underline{\pmscore{62.56}{0.43}} & \pmscore{68.16}{0.38} & \pmscore{76.27}{0.12} & \boldpmscore{83.07}{0.27} & \pmscore{87.02}{0.25} & \pmscore{75.42}{0.29}\\
		SoftCPT*&8,396,800&9.95&\pmscore{60.43}{0.38} & \pmscore{69.33}{0.80} & \pmscore{76.45}{0.67} & \pmscore{82.13}{0.34} & \pmscore{86.08}{0.20}&\pmscore{74.88}{0.48}\\
		SoftCPT-NATS &8,421,376&10.10& \pmscore{62.26}{1.28} & \underline{\pmscore{70.52}{0.46}}& \underline{\pmscore{77.77}{0.63}} & \underline{\pmscore{83.01}{0.62}} & \pmscore{86.99}{0.06}&\underline{\pmscore{76.11}{0.61}}\\
		\rowcolor{tabhighlight}
		SoftCPT-NATA &8,392,704&10.05&\boldpmscore{64.44}{0.43} & \boldpmscore{72.13}{1.02} & \boldpmscore{78.16}{0.28} & \pmscore{82.97}{0.88} &\pmscore{86.61}{0.17}&\boldpmscore{76.86}{0.56}\\
		
		\midrule
		\multicolumn{4}{l}{\textBF{Fashion-20 (20 tasks, 78 classes)}}\\
		ZS-CLIP~\cite{CLIP}&--&--&--&--&--&--&--&\pmscore{45.49}{0.00}\\
		ProtoNet~\cite{ProtoNet}&--&--& \pmscore{35.58}{0.20} & \pmscore{36.71}{1.54} & \pmscore{40.41}{0.39}& \pmscore{45.70}{0.23} & \pmscore{49.57}{0.33} & \pmscore{41.60}{0.54}\\
		LP-CLIP~\cite{CLIP}&--&--&\pmscore{48.30}{2.44}&\pmscore{54.50}{0.43}&\pmscore{61.02}{0.55}&\pmscore{65.71}{0.83}&\pmscore{71.00}{0.78}&\pmscore{60.11}{1.01}\\
		CoOp-CA~\cite{CoOp}&163,840&24.43&\pmscore{50.50}{0.63}&\pmscore{55.57}{0.40}&\pmscore{60.00}{0.56}&\pmscore{64.55}{1.66}&\pmscore{69.12}{2.32}&\pmscore{59.95}{1.11}\\
		CoOp-CS~\cite{CoOp}&638,976&23.17&\pmscore{46.49}{1.63}&\pmscore{50.16}{1.32}&\pmscore{57.04}{0.67}&\pmscore{63.28}{0.86}&\pmscore{68.85}{0.58}&\pmscore{57.16}{1.01}\\
		CoOp-MT&8,192&3.52&\pmscore{49.60}{1.99}&{\pmscore{57.01}{2.35}}&\pmscore{59.30}{0.49}&\pmscore{66.15}{1.11}&\pmscore{70.69}{0.89}&\pmscore{60.55}{1.37}\\
		CoCoOp~\cite{zhou2022cocoop}&2,018,560&11.17&\pmscore{50.67}{2.15} & \pmscore{55.44}{2.36} & \pmscore{57.53}{1.84} & \pmscore{64.42}{0.49} & \pmscore{68.69}{0.75} & \pmscore{59.35}{1.52} \\
		ProGrad~\cite{ProGrad}&163,840&20.02&\pmscore{52.38}{0.86} & \pmscore{56.17}{0.53} & {\pmscore{62.11}{0.41}}&\pmscore{65.30}{1.52}&\pmscore{68.01}{1.56}&\pmscore{60.79}{0.98}\\
		KgCoOp~\cite{KgCoOp}& 40,960 &18.95&\underline{\pmscore{54.65}{0.74}} & \boldpmscore{59.26}{1.26}&\boldpmscore{62.35}{0.69}&\pmscore{64.90}{0.71}&\pmscore{68.70}{0.37}&\underline{\pmscore{61.97}{0.75}}\\
		PLOT~\cite{PLOT}& 655,360&18.13&\pmscore{52.46}{0.81} & \underline{\pmscore{57.25}{0.77}} & \pmscore{61.75}{1.88} & \underline{\pmscore{67.46}{0.60}} & \pmscore{70.66}{0.54}&{\pmscore{61.92}{0.92}}\\
		SoftCPT*&8,396,800&6.23&{\pmscore{51.41}{0.36}}&{\pmscore{56.08}{1.54}}&{\pmscore{61.88}{0.68}}&\pmscore{66.82}{0.27}&\pmscore{71.38}{0.82}&{\pmscore{61.51}{0.73}}\\
		SoftCPT-NATS&8,421,376&6.47&\pmscore{49.49}{1.65}&\pmscore{55.60}{1.78}&\pmscore{60.29}{1.21}&{\boldpmscore{67.47}{0.59}}&\underline{\pmscore{71.57}{0.16}}&\pmscore{60.88}{1.08}\\
		\rowcolor{tabhighlight}
		SoftCPT-NATA&8,392,704&8.00&{\boldpmscore{55.90}{0.98}}&{\pmscore{56.07}{1.17}}&\underline{\pmscore{62.26}{0.26}}&{\pmscore{67.34}{0.38}}&{\boldpmscore{72.18}{0.19}}& {\boldpmscore{62.75}{0.60}}\\
		\bottomrule
	\end{tabular}
	\label{tab:main_numerical_results}
\end{table*}

\subsection{Main Results}
\label{ssec:main_results}
The main results of different methods on four datasets are listed in Table~\ref{tab:main_numerical_results}. For SoftCPT, we report the results of \textbf{SoftCPT-NATA} and \textbf{SoftCPT-NATS} as they could acquire desirable performance with lower computational cost. As a comparison, we also report a variant of SoftCPT, i.e., \textbf{SoftCPT*}. It is the method that learns task features from scratch without using the text encoder in the meta network. We have the following observations from the results.

\textbf{CoOp-CA vs CoOp-CS.} It is evident that CoOp-CA surpasses CoOp-CS in terms of average accuracy on all datasets. Given that CoOp-CA employs fewer parameters, it exhibits greater data efficiency. This is corroborated by its higher scores with fewer training samples. Nevertheless, as the amount of training data increases, the performance gap between the two methods narrows.

\begin{figure*}[t!]
	\centering
	\begin{subfigure}[t]{0.24\linewidth}
		\centering
		\includegraphics[width=1.6in]{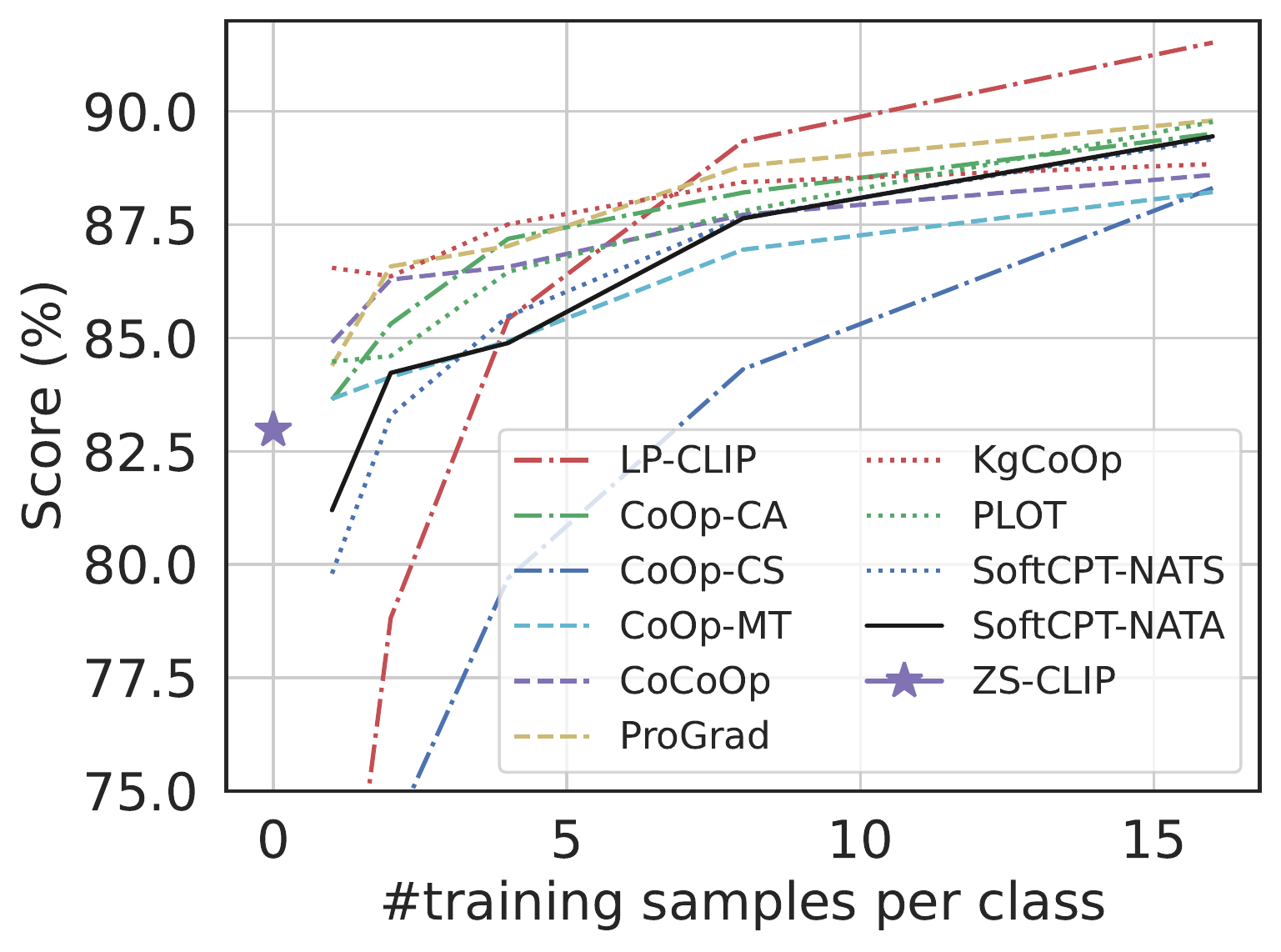}
		\caption{Caltech101}
	\end{subfigure}
	\begin{subfigure}[t]{0.24\linewidth}
		\centering
		\includegraphics[width=1.6in]{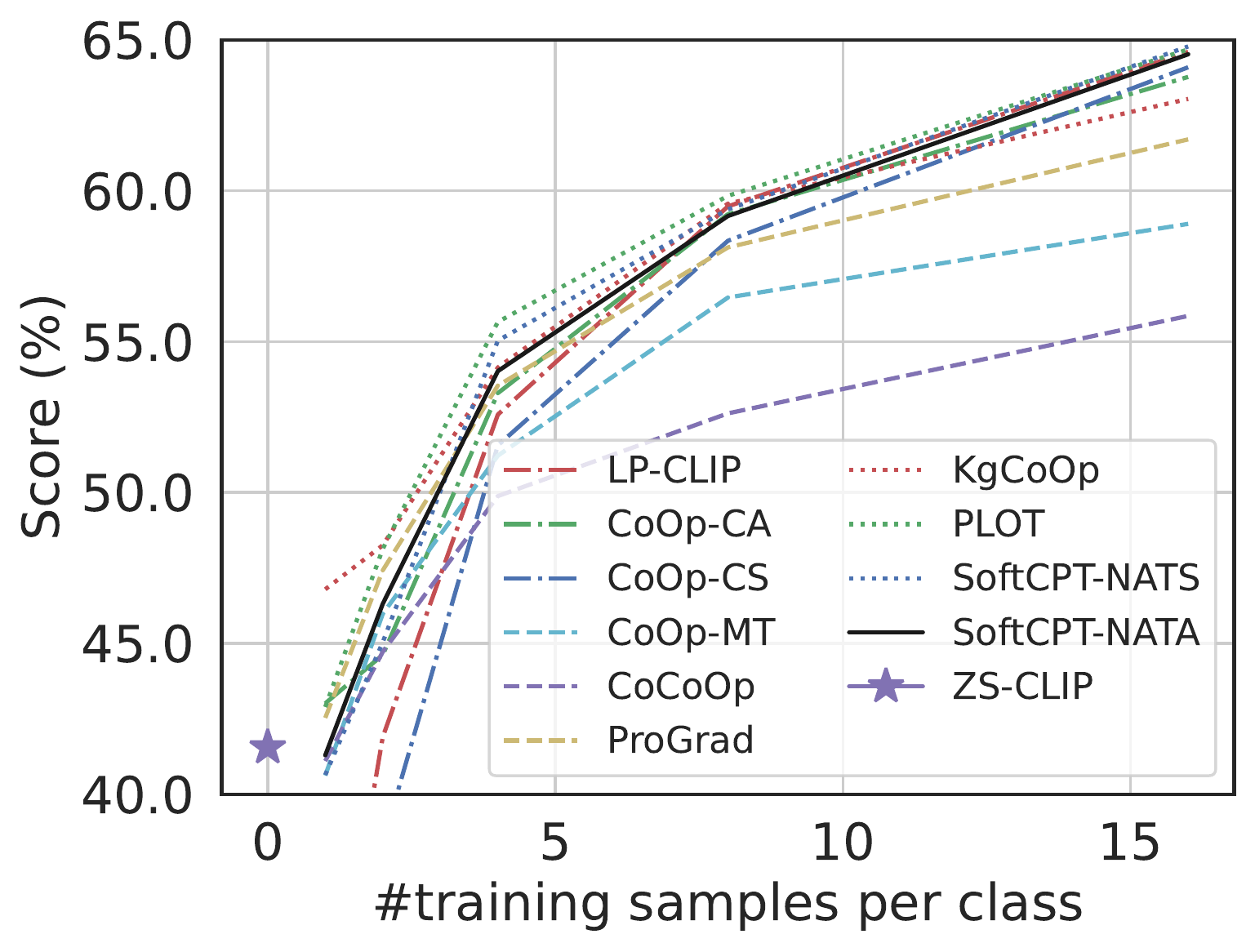}
		\caption{DTD}
	\end{subfigure}
	\begin{subfigure}[t]{0.24\linewidth}
		\centering
		\includegraphics[width=1.6in]{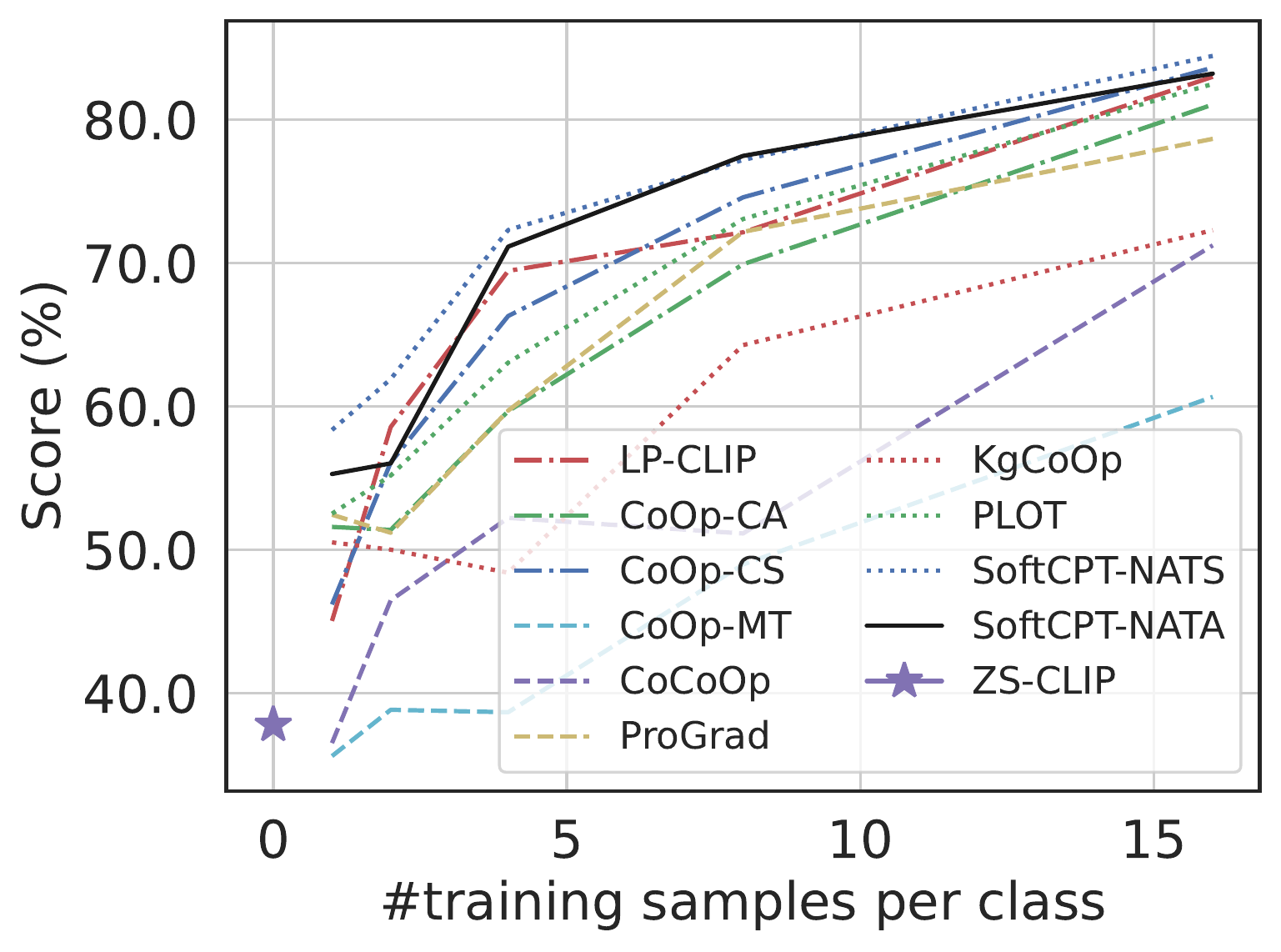}
		\caption{EuroSAT}
	\end{subfigure}
	\begin{subfigure}[t]{0.24\linewidth}
		\centering
		\includegraphics[width=1.6in]{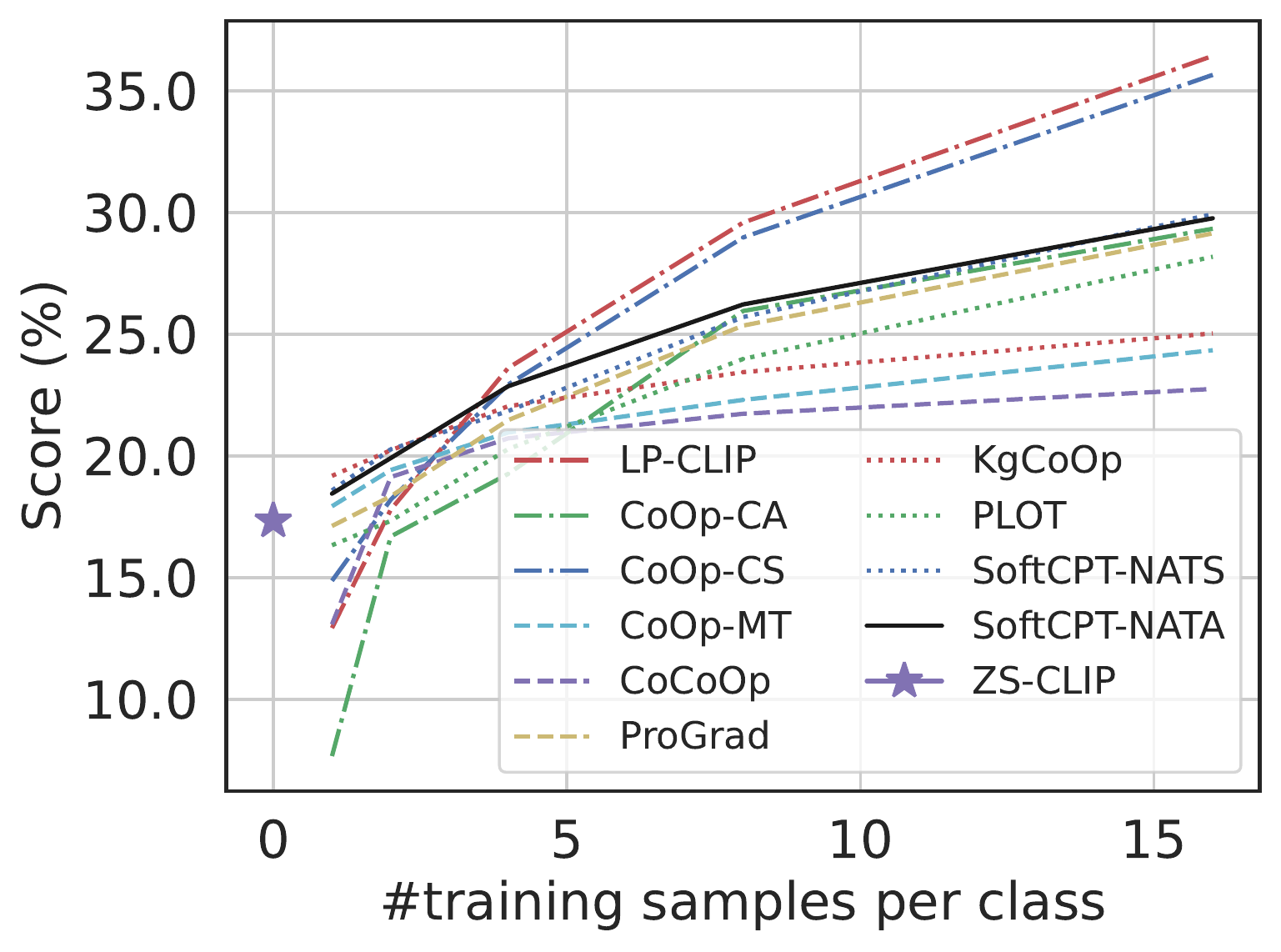}
		\caption{FGVCAircraft}
	\end{subfigure}
	\\
	
	\begin{subfigure}[t]{0.24\linewidth}
		\centering
		\includegraphics[width=1.6in]{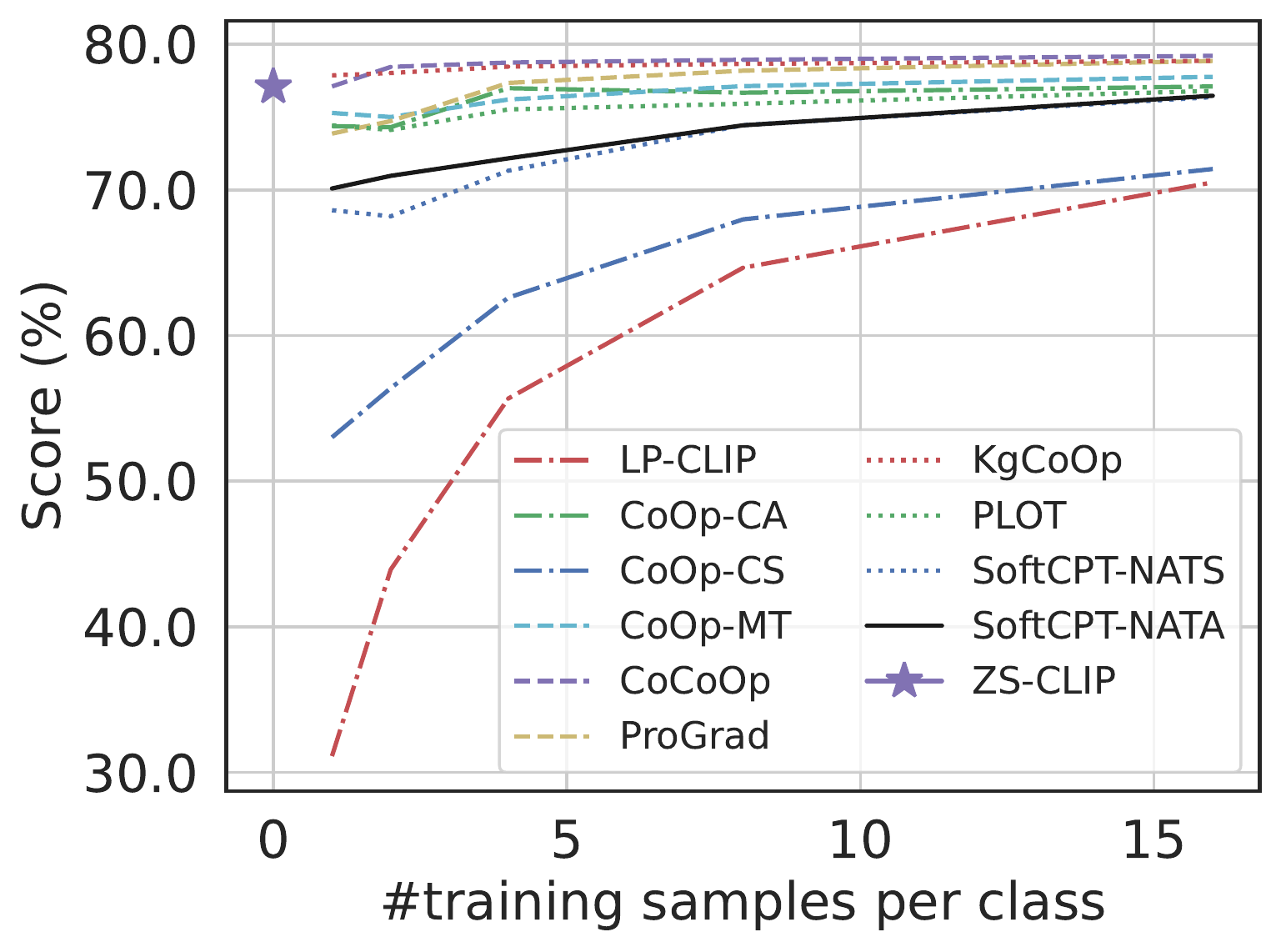}
		\caption{Food101}
	\end{subfigure}
	\begin{subfigure}[t]{0.24\linewidth}
		\centering
		\includegraphics[width=1.6in]{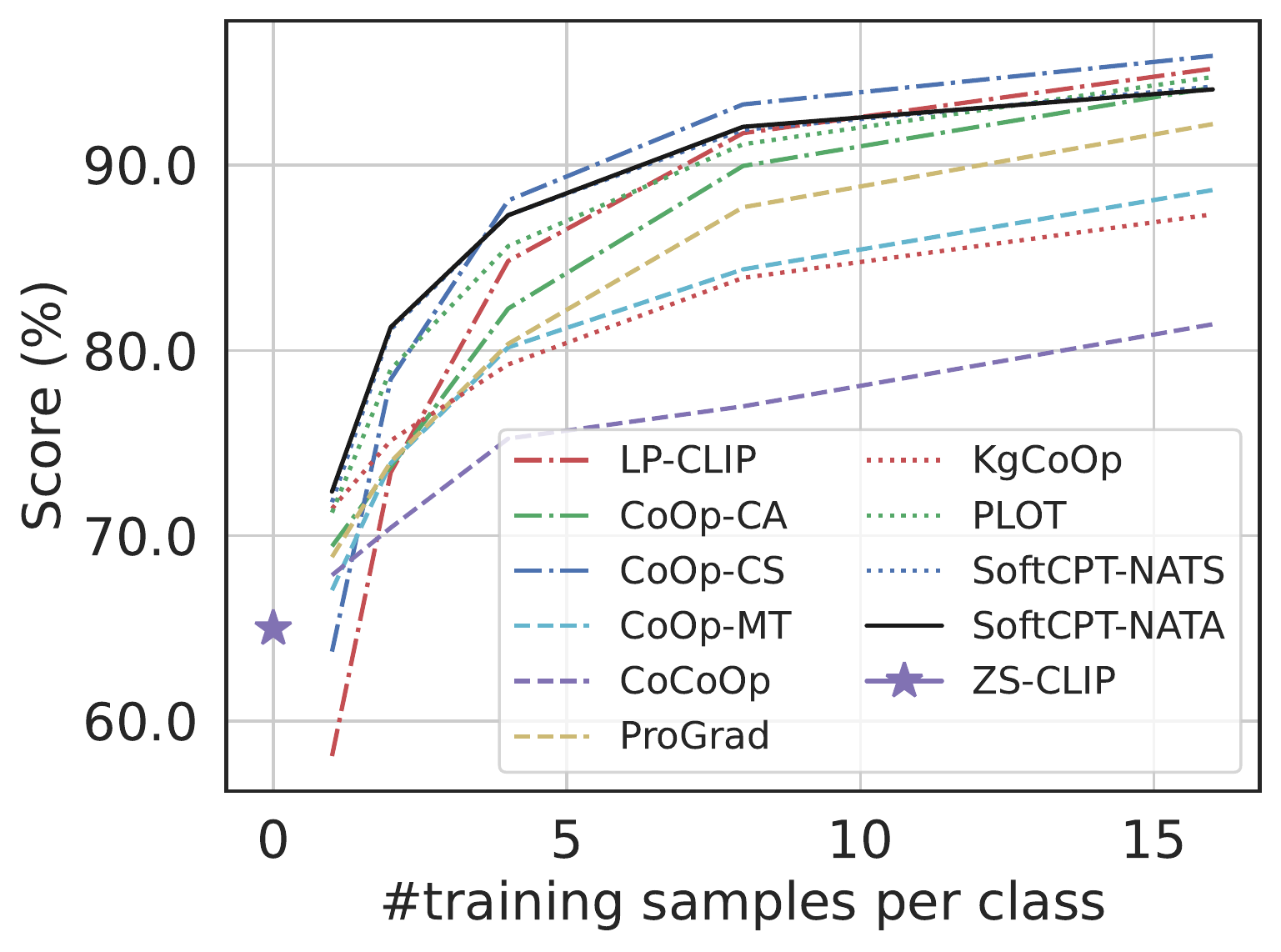}
		\caption{Flowers102}
	\end{subfigure}
	\begin{subfigure}[t]{0.24\linewidth}
		\centering
		\includegraphics[width=1.6in]{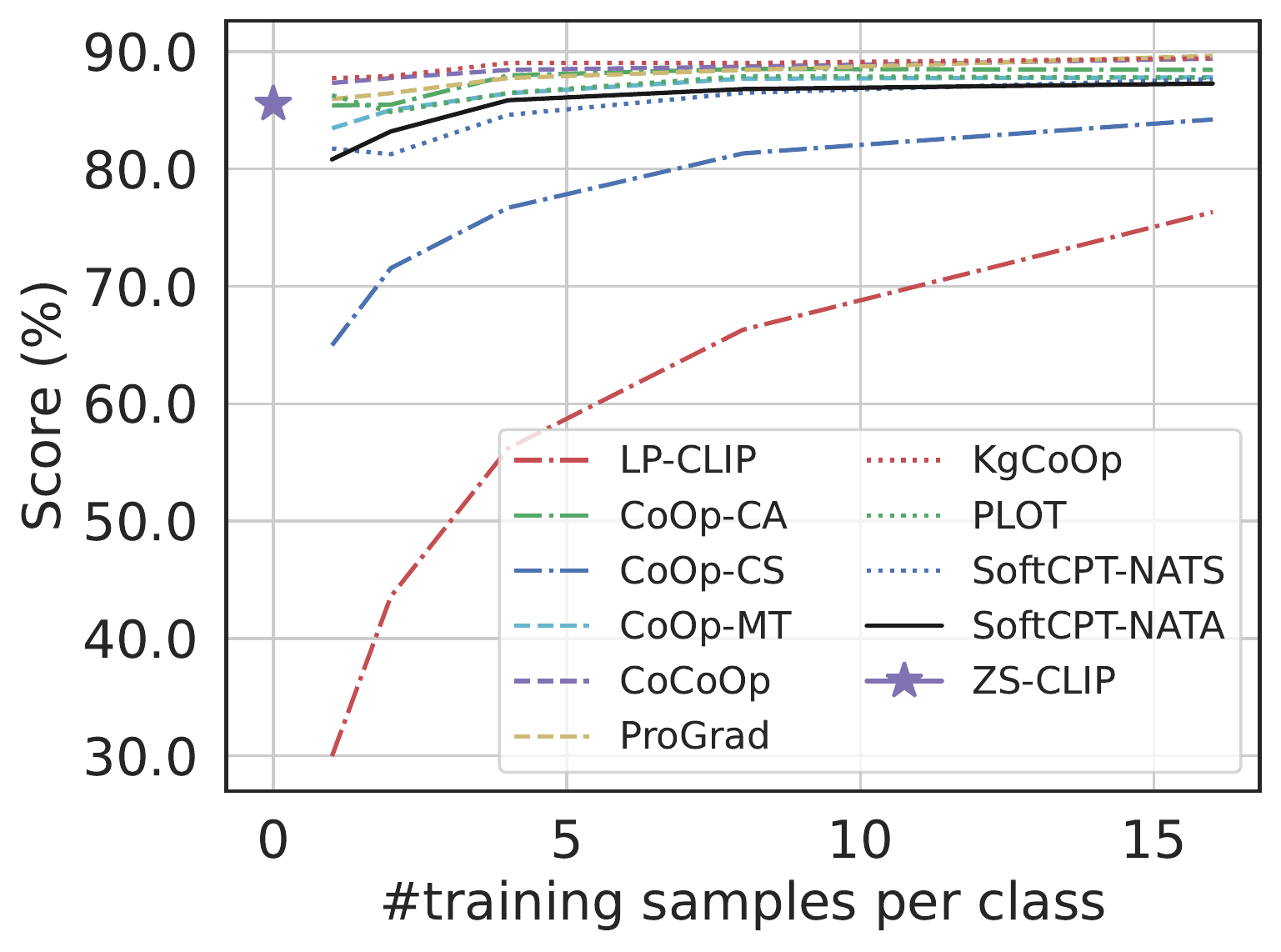}
		\caption{Oxford-Pets}
	\end{subfigure}
	\begin{subfigure}[t]{0.24\linewidth}
		\centering
		\includegraphics[width=1.6in]{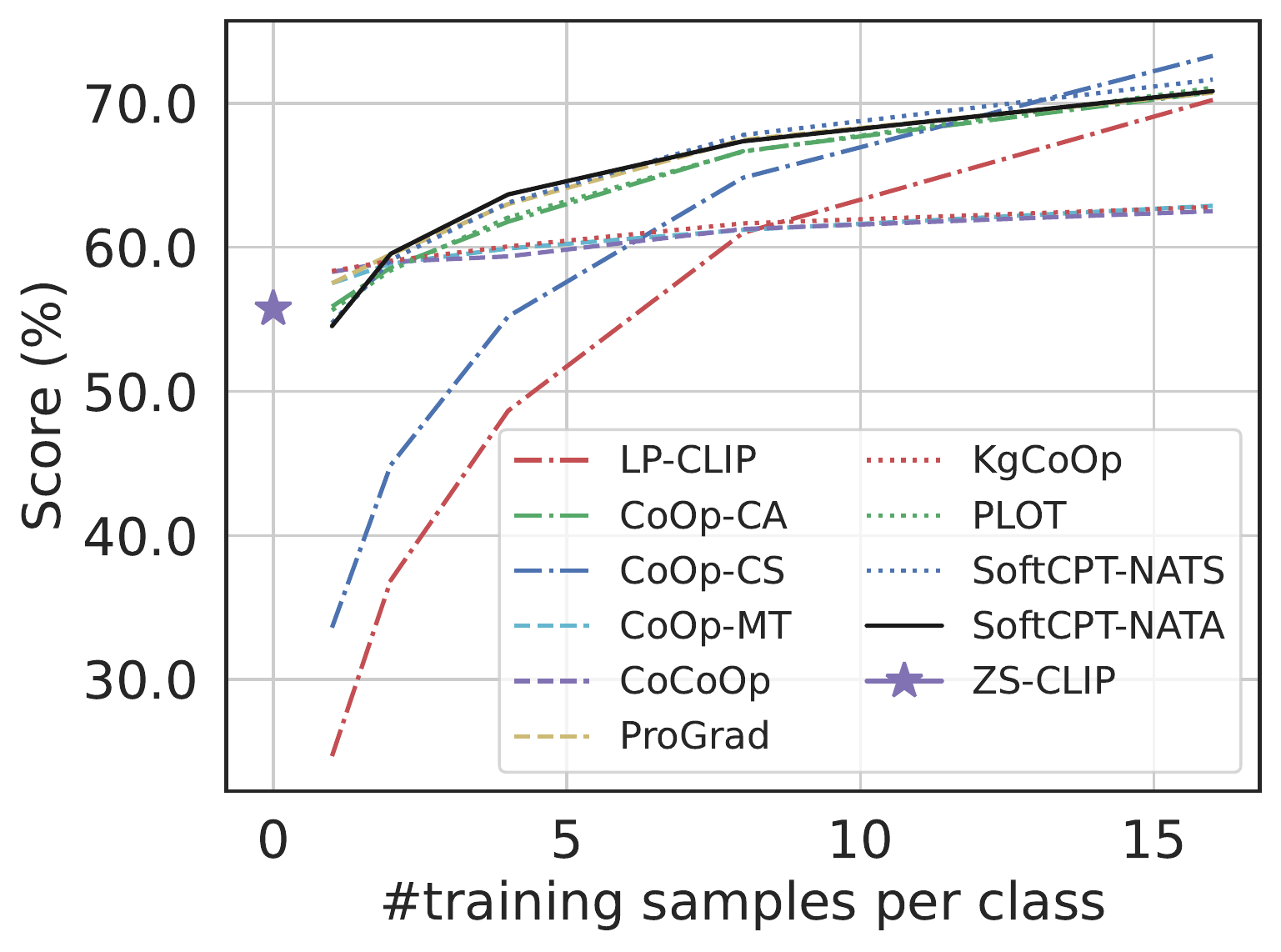}
		\caption{StanfordCars}
	\end{subfigure}
	\\
	
	\begin{subfigure}[t]{0.24\linewidth}
		\centering
		\includegraphics[width=1.6in]{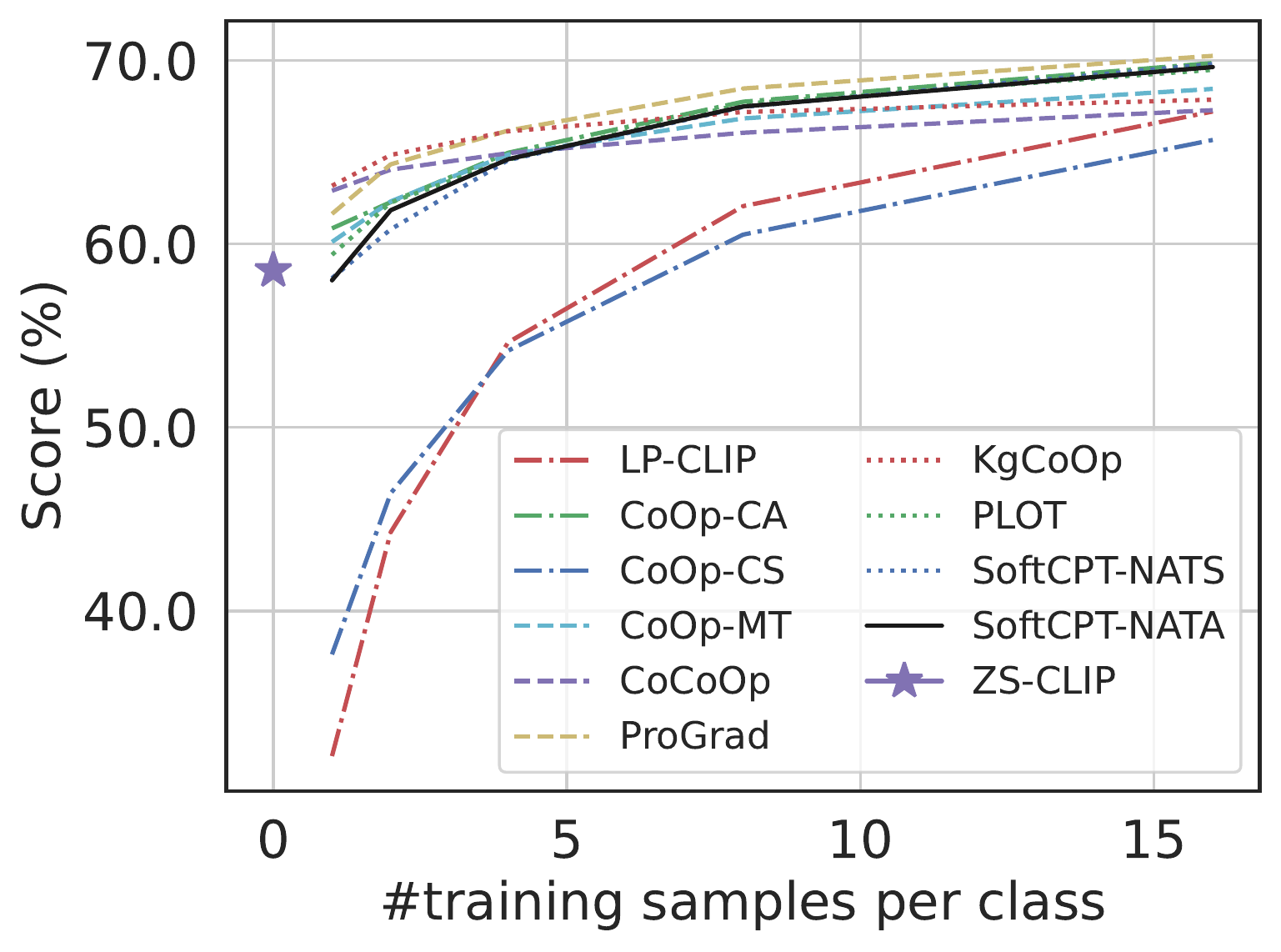}
		\caption{SUN397}
	\end{subfigure}
	\begin{subfigure}[t]{0.24\linewidth}
		\centering
		\includegraphics[width=1.6in]{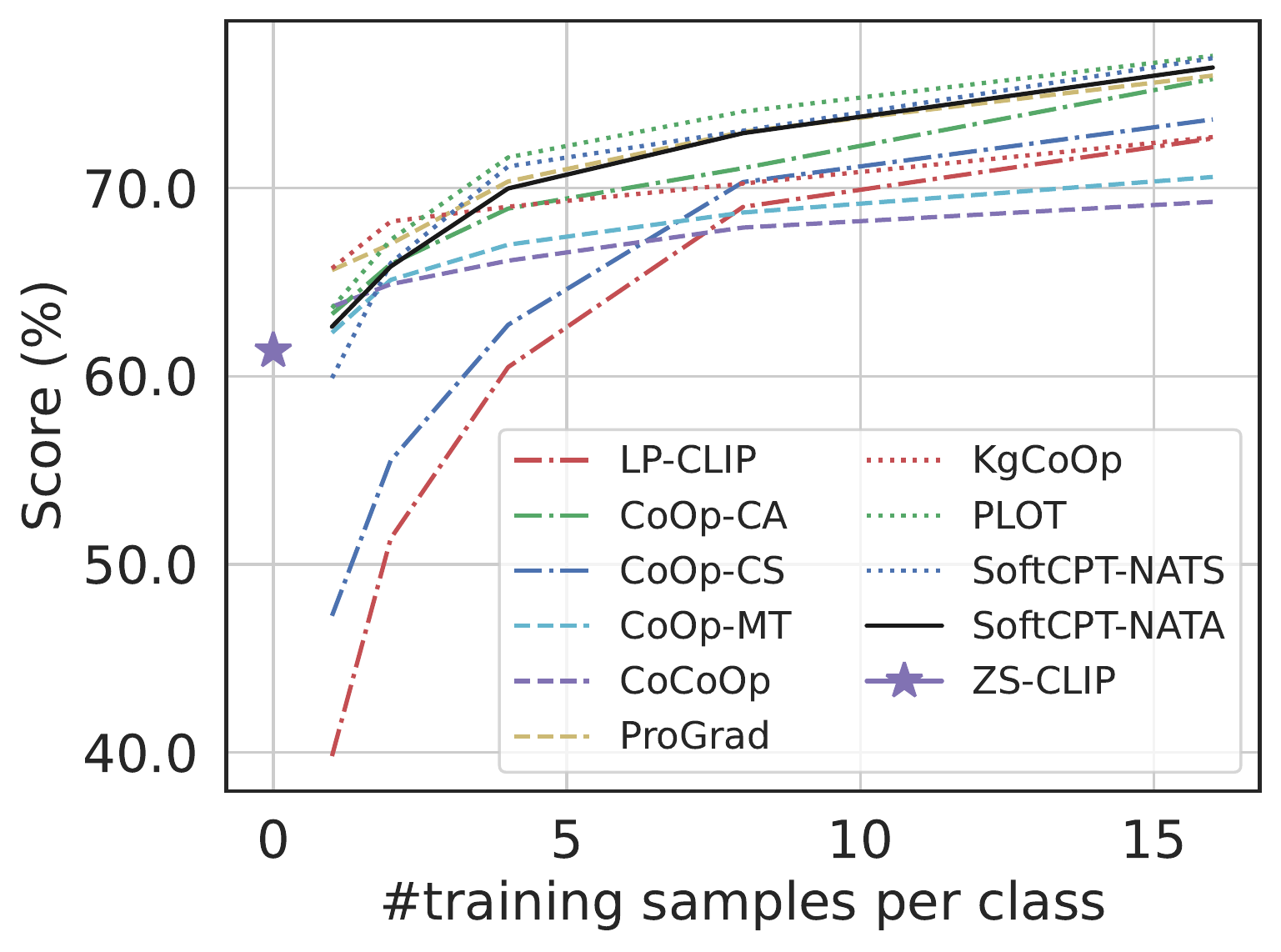}
		\caption{UCF101}
	\end{subfigure}
	\begin{subfigure}[t]{0.24\linewidth}
		\centering
		\includegraphics[width=1.6in]{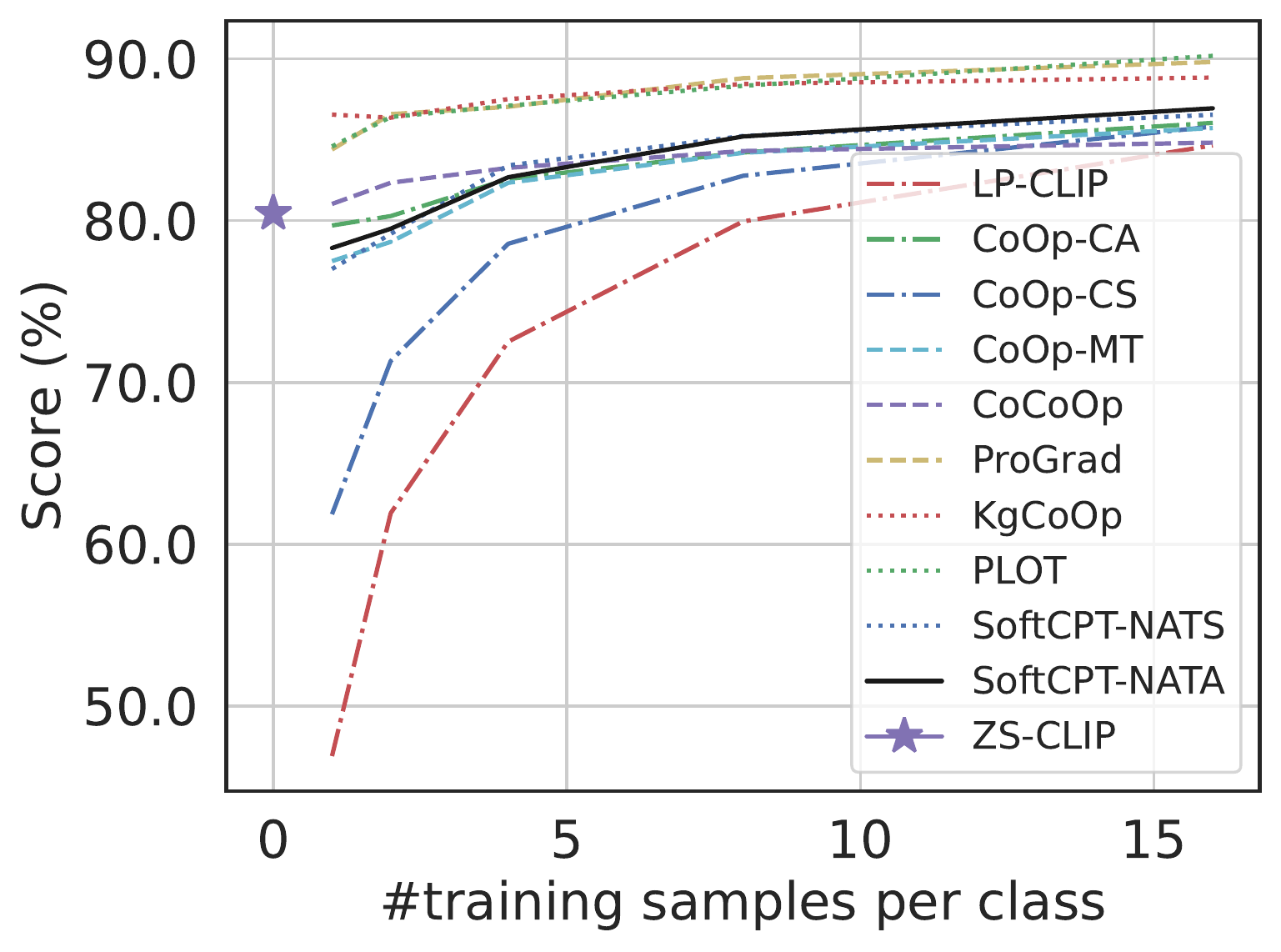}
		\caption{FruitVegetable}
	\end{subfigure}
	\begin{subfigure}[t]{0.24\linewidth}
		\centering
		\includegraphics[width=1.6in]{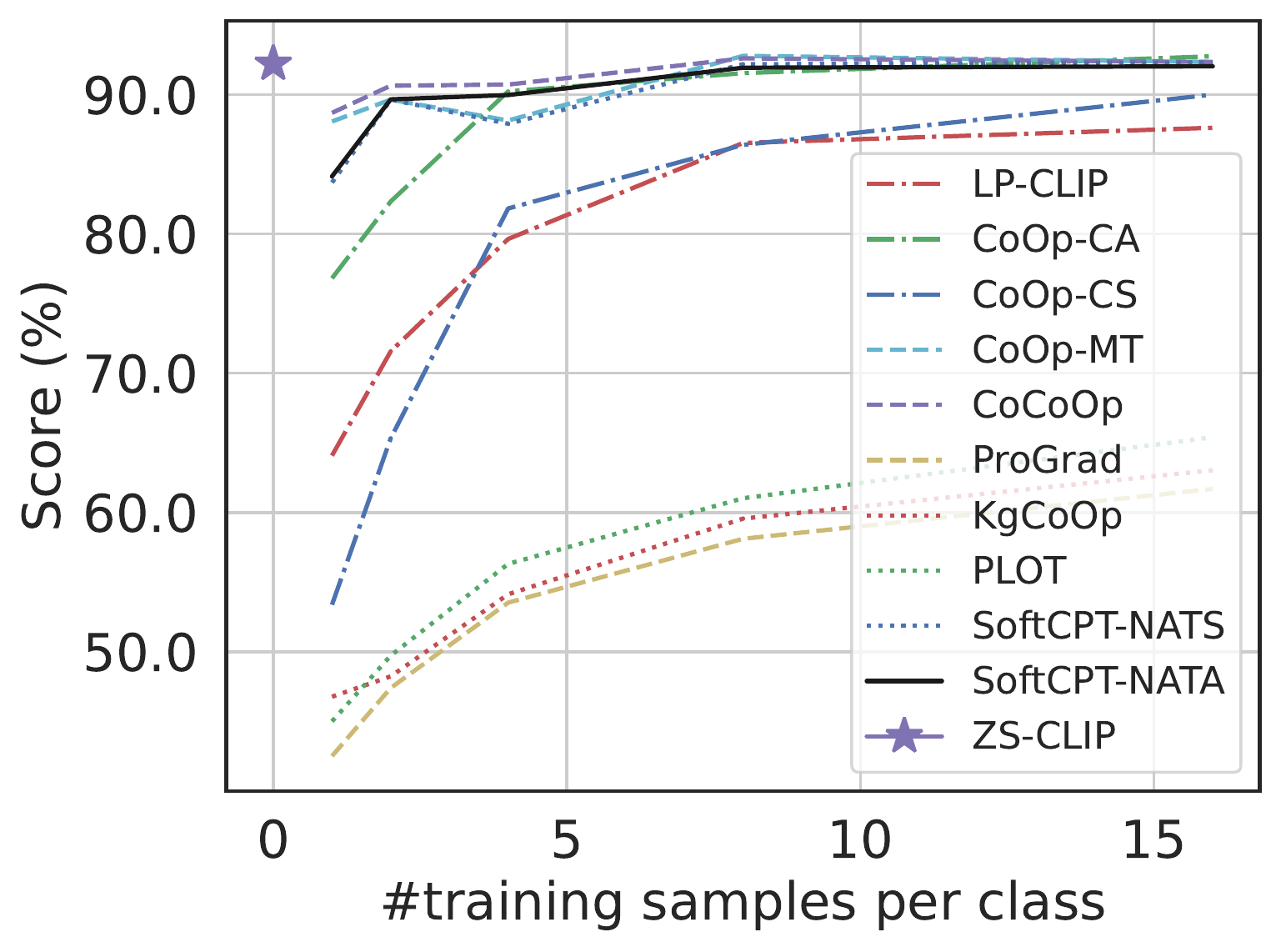}
		\caption{KaggleFlower}
	\end{subfigure}
	\\
	
	\begin{subfigure}[t]{0.24\linewidth}
		\centering
		\includegraphics[width=1.6in]{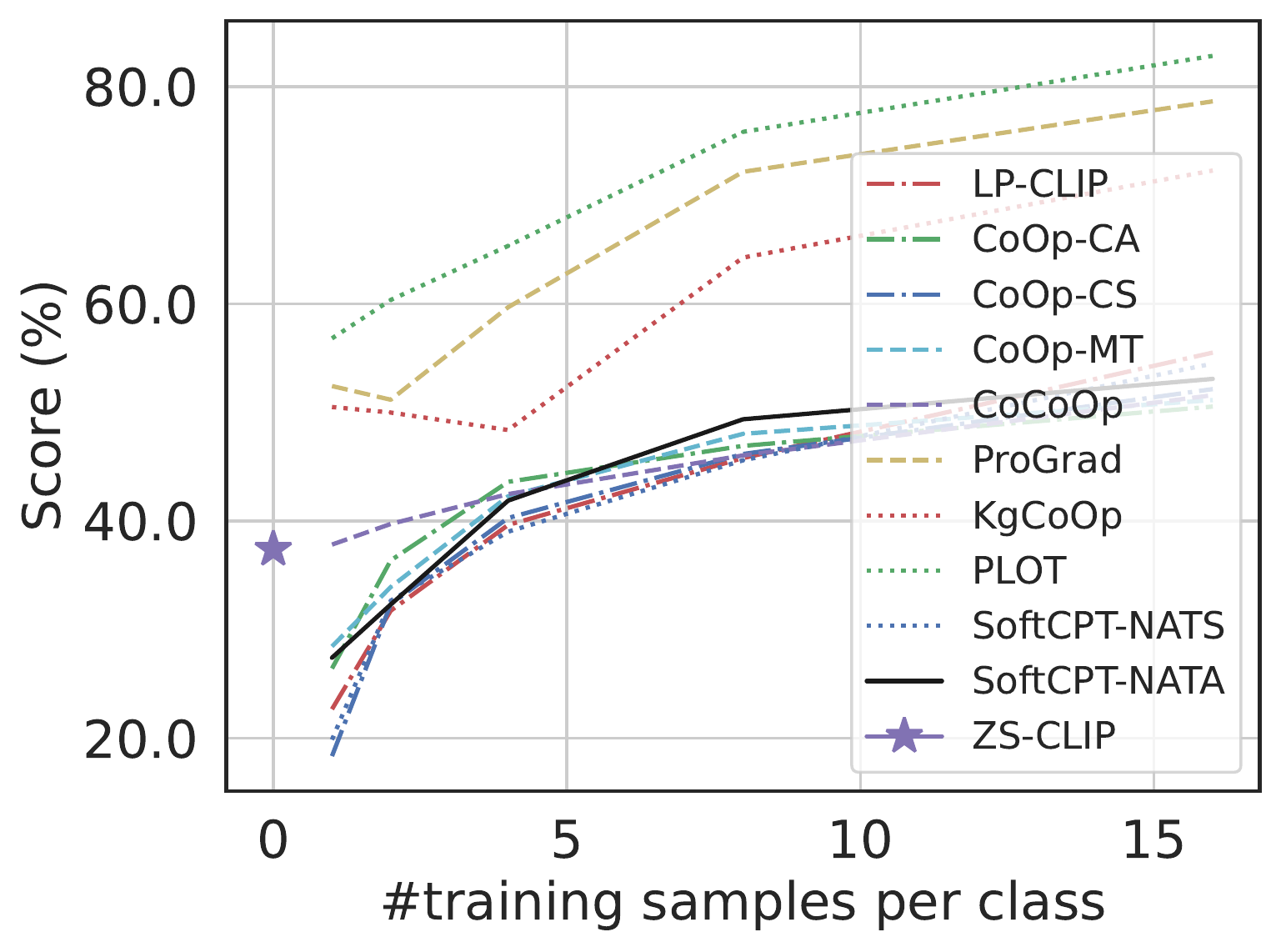}
		\caption{KaggleMushroom}
	\end{subfigure}
	\begin{subfigure}[t]{0.24\linewidth}
		\centering
		\includegraphics[width=1.6in]{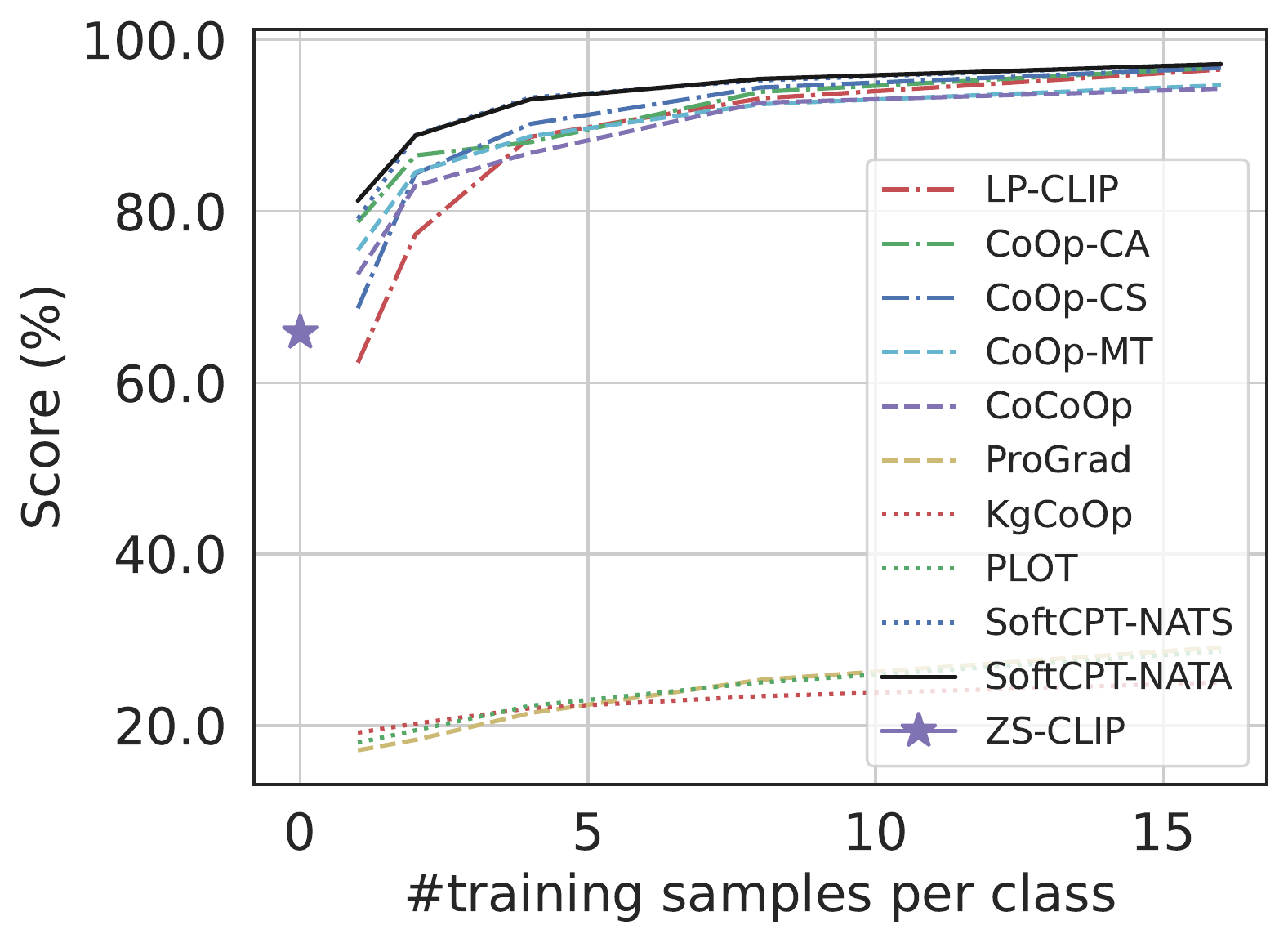}
		\caption{KaggleVegetable}
	\end{subfigure}
	\begin{subfigure}[t]{0.24\linewidth}
		\centering
		\includegraphics[width=1.6in]{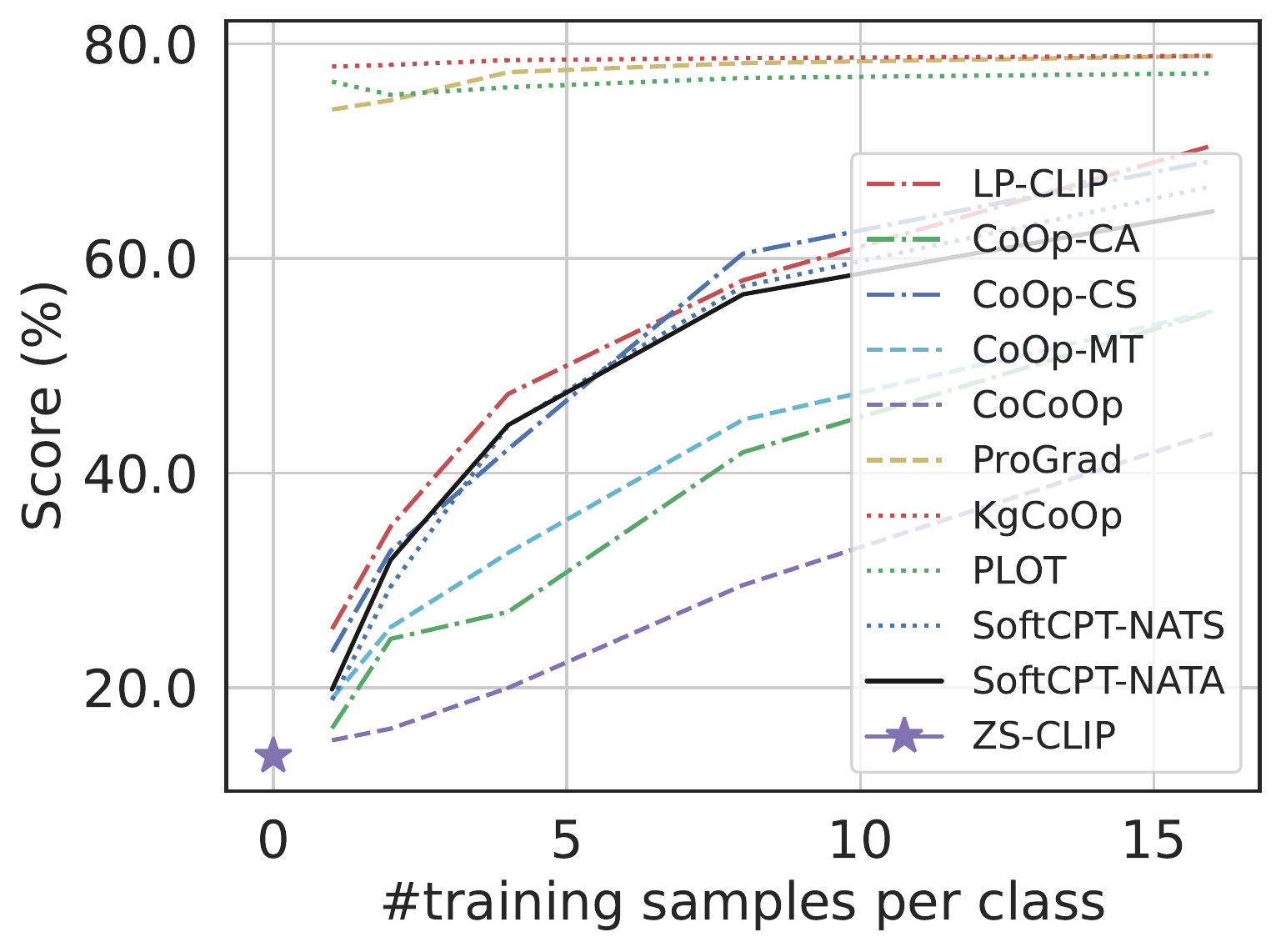}
		\caption{PlantSeedling}
	\end{subfigure}
	\begin{subfigure}[t]{0.24\linewidth}
		\centering
		\includegraphics[width=1.6in]{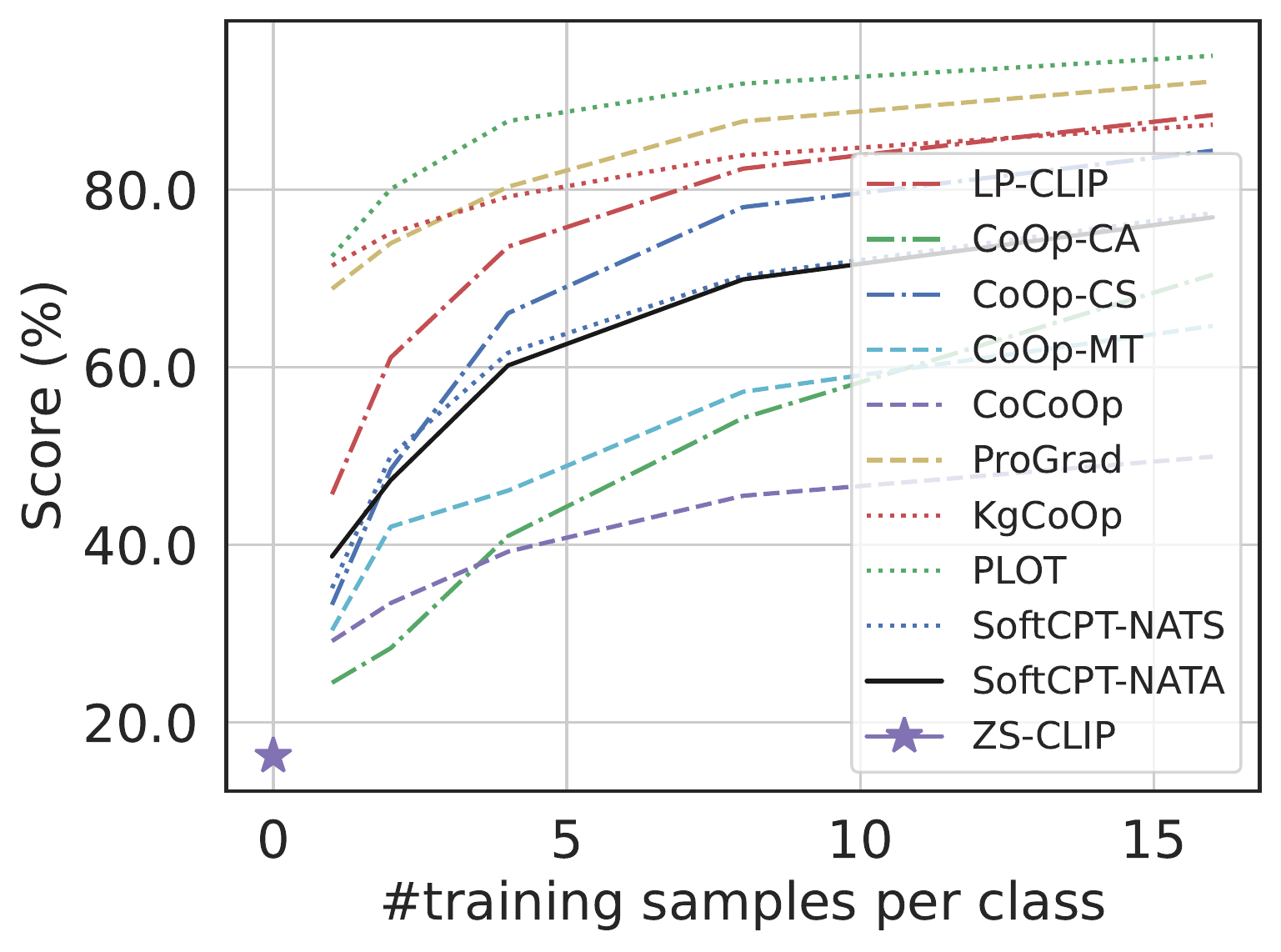}
		\caption{PlantVillage}
	\end{subfigure}
	
	\caption{Per-task results on General-10 (a)-(j) and Plant-6 (k)-(p).}
	\label{fig:mtcv1_per_task_results}
\end{figure*}

\begin{figure*}[t!]
	\begin{subfigure}[t]{0.24\linewidth}
		\centering
		\includegraphics[width=1.6in]{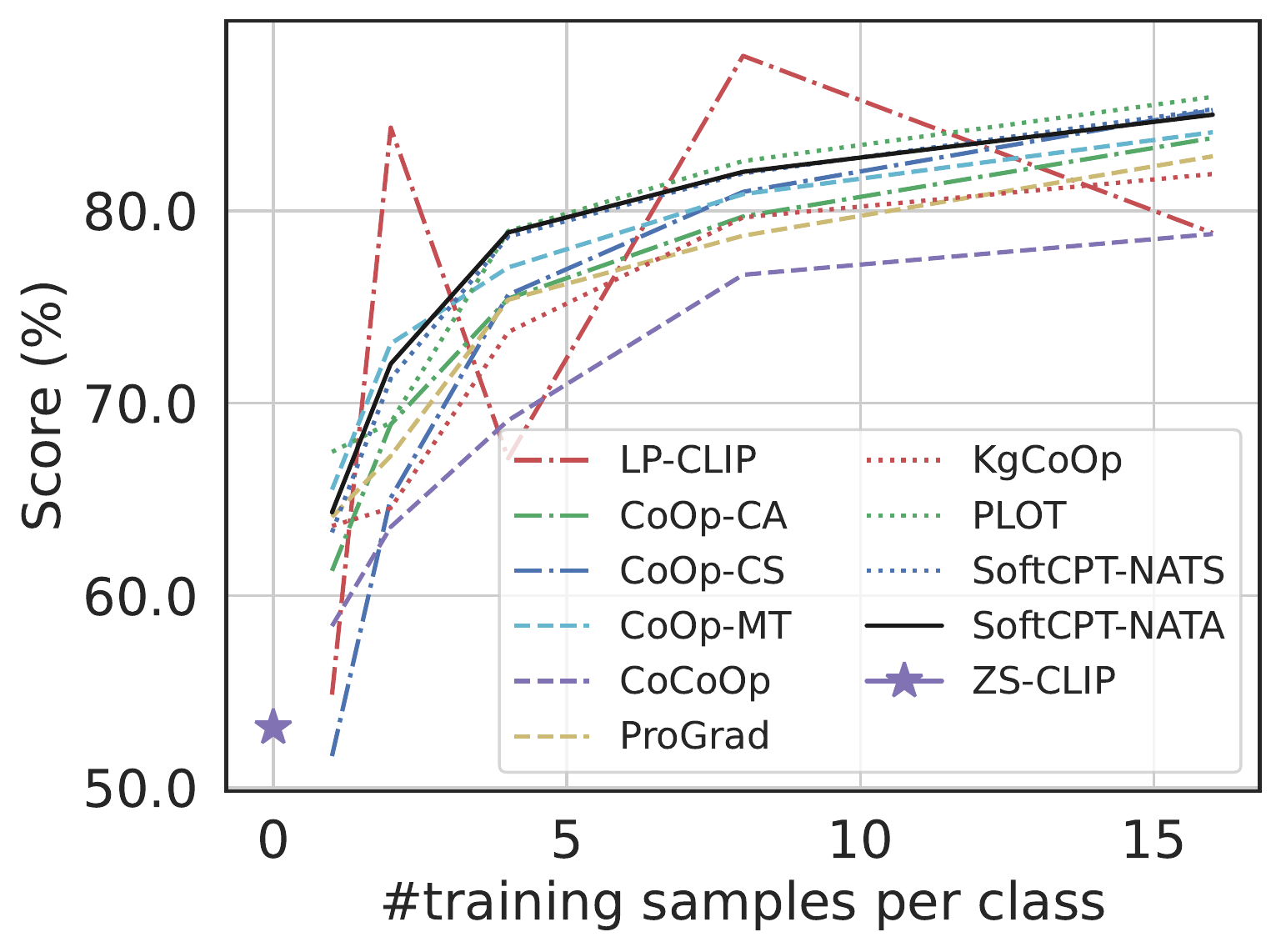}
		\caption{AID}
	\end{subfigure}
	\begin{subfigure}[t]{0.24\linewidth}
		\centering
		\includegraphics[width=1.6in]{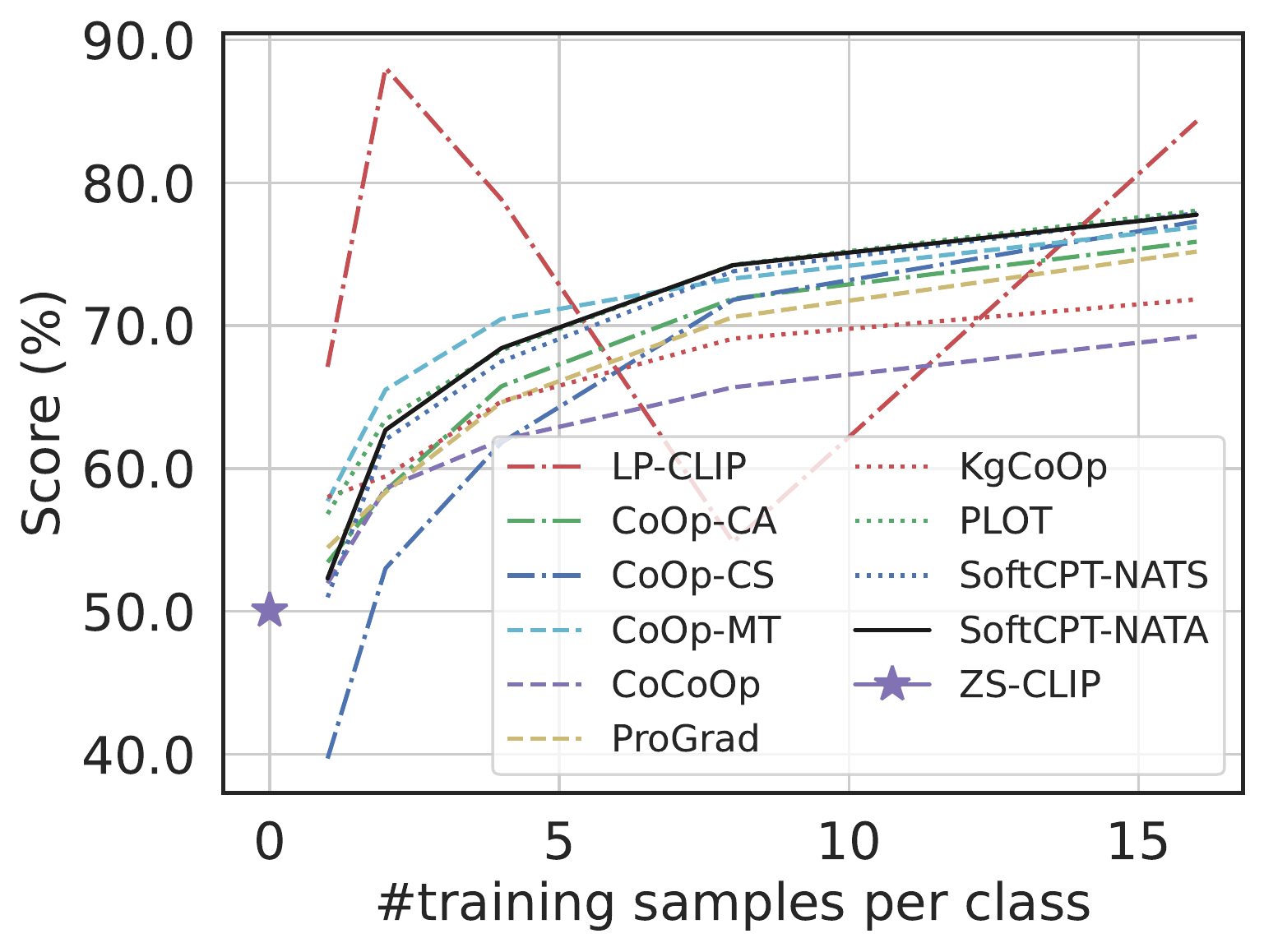}
		\caption{RESISC45}
	\end{subfigure}
	\begin{subfigure}[t]{0.24\linewidth}
		\centering
		\includegraphics[width=1.6in]{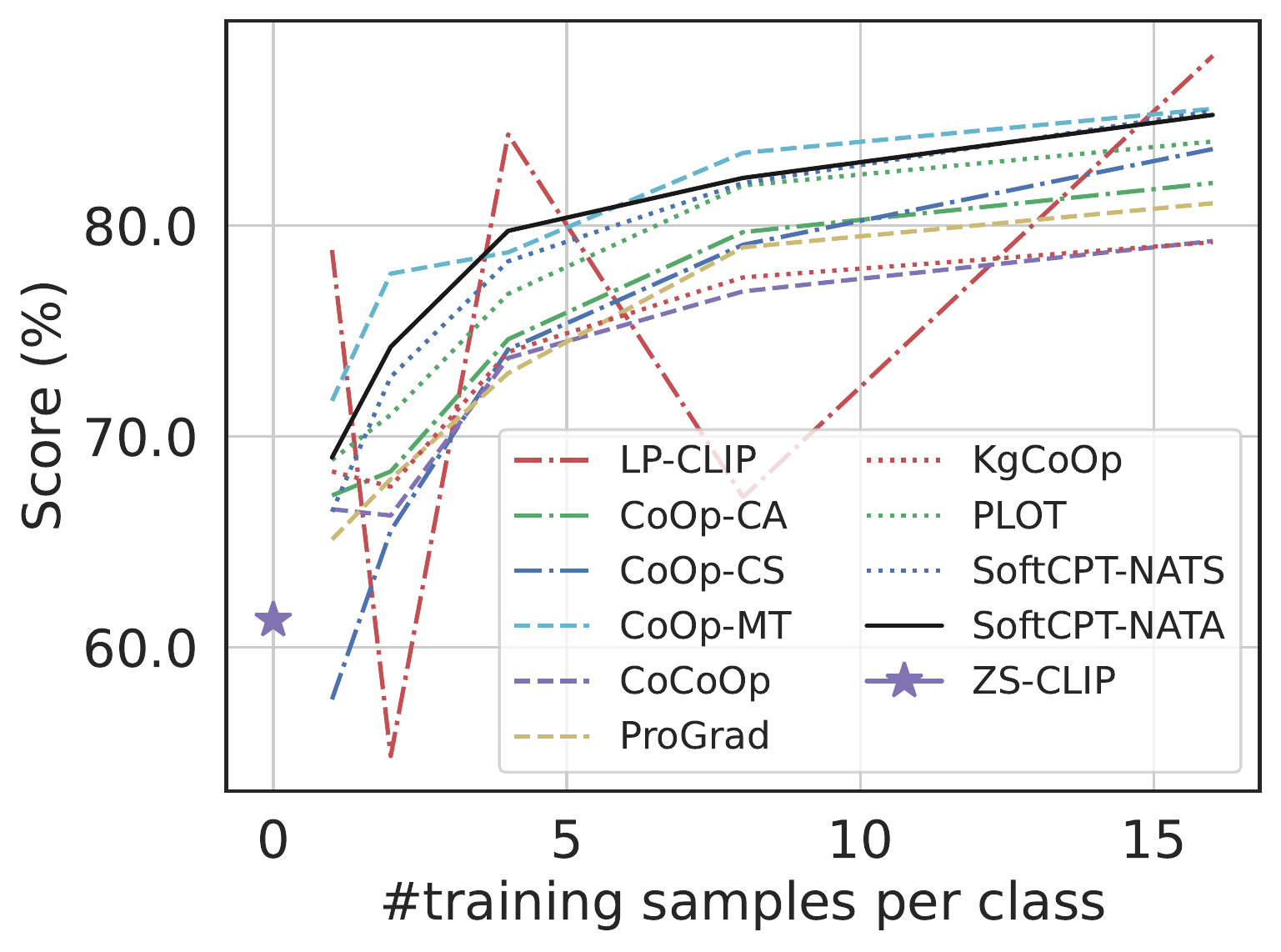}
		\caption{OPTIMAL}
	\end{subfigure}
	\begin{subfigure}[t]{0.24\linewidth}
		\centering
		\includegraphics[width=1.6in]{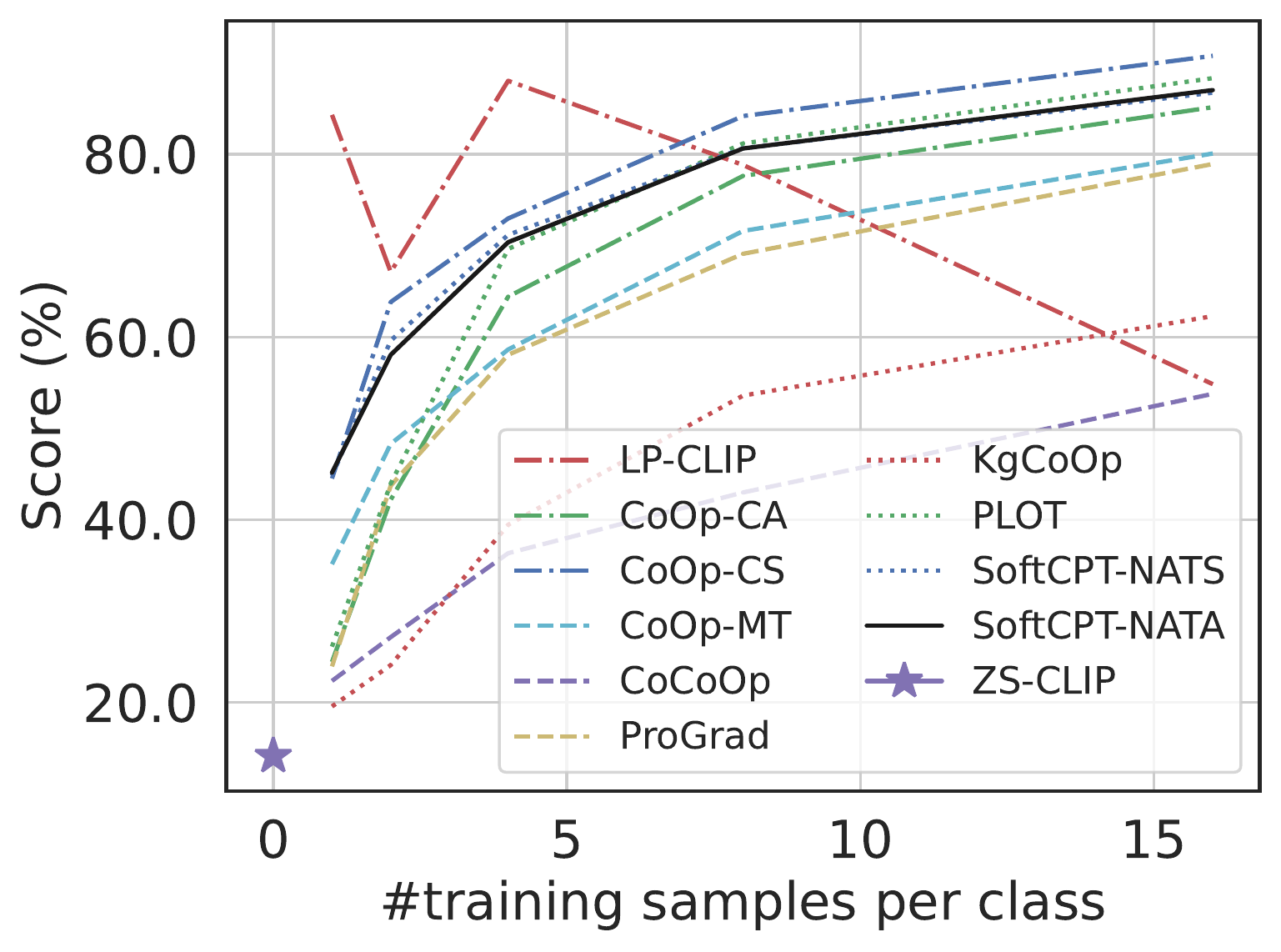}
		\caption{RSICB128}
	\end{subfigure}
	\\
	
	\begin{subfigure}[t]{0.24\linewidth}
		\centering
		\includegraphics[width=1.6in]{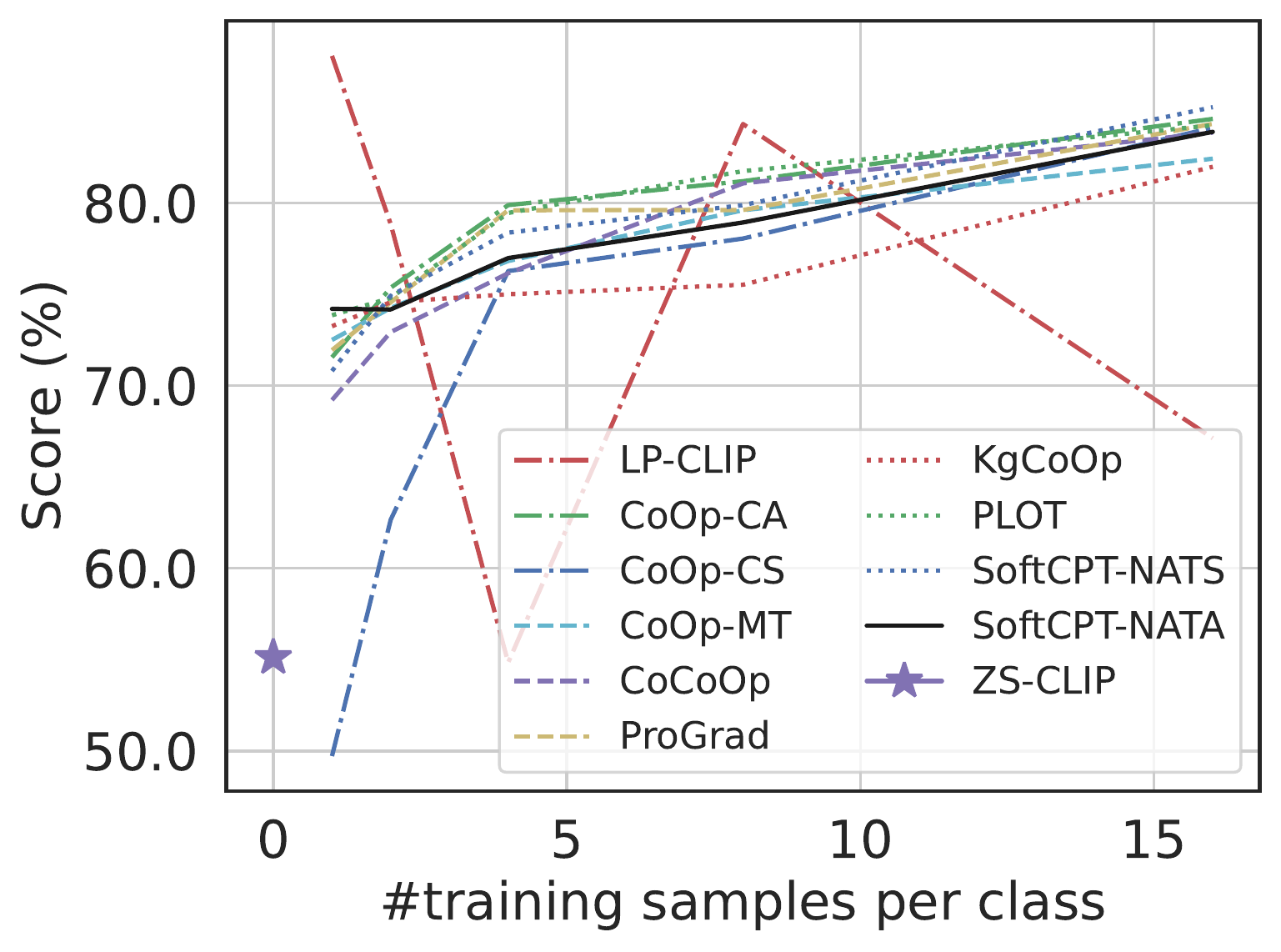}
		\caption{RSSCN7}
	\end{subfigure}
	\begin{subfigure}[t]{0.24\linewidth}
		\centering
		\includegraphics[width=1.6in]{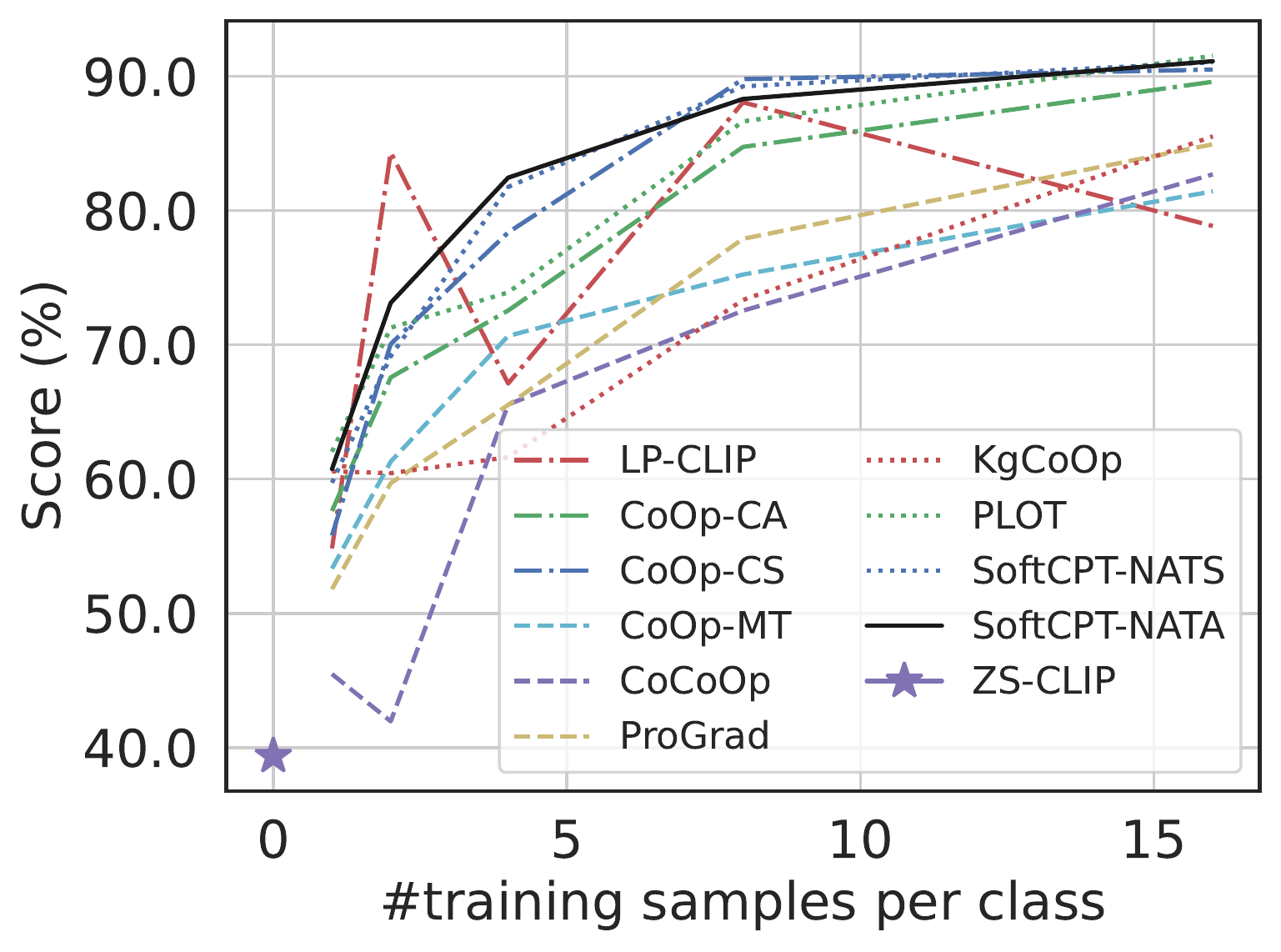}
		\caption{NaSC-TG2}
	\end{subfigure}
	\begin{subfigure}[t]{0.24\linewidth}
		\centering
		\includegraphics[width=1.6in]{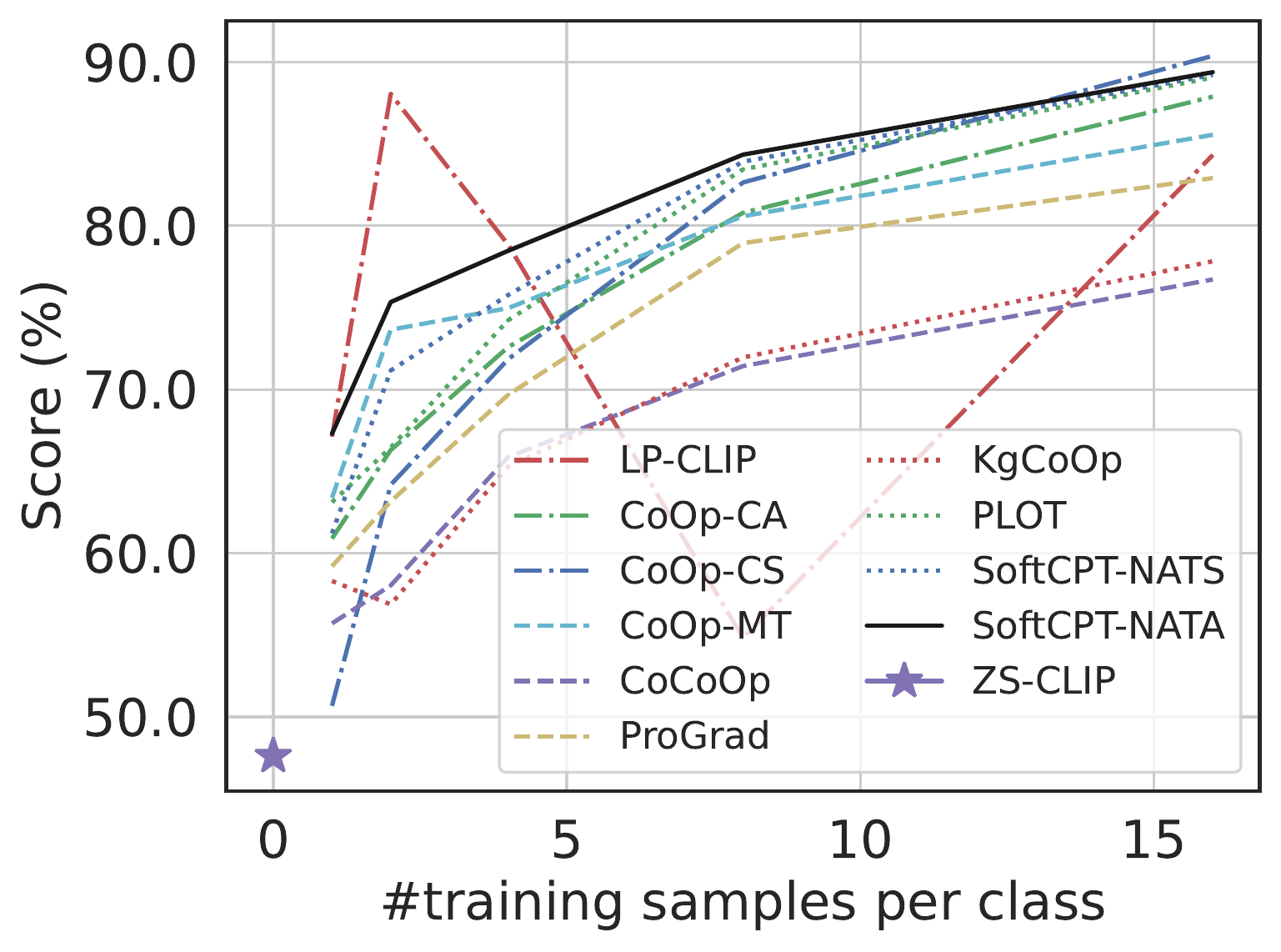}
		\caption{UCMerced}
	\end{subfigure}
	\begin{subfigure}[t]{0.24\linewidth}
		\centering
		\includegraphics[width=1.6in]{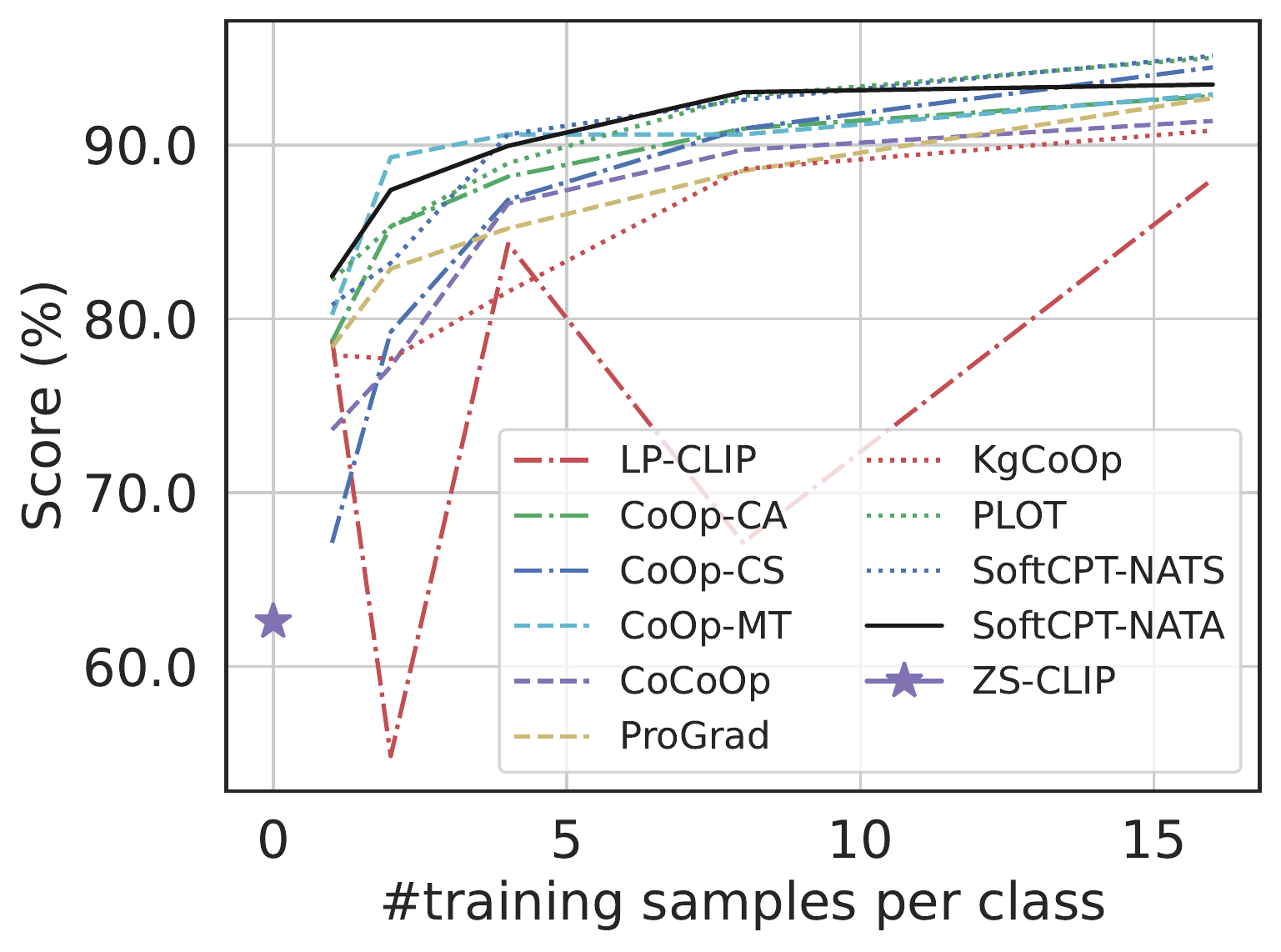}
		\caption{WHURS19}
	\end{subfigure}
	
	\caption{Per-task results on RS-8.}
	\label{fig:rsv1_per_task_results}
\end{figure*}

\textbf{SoftCPT-NATA vs CoOp-CA.} Compared to CoOp-CA, SoftCPT-NATA improves the average scores by \textbf{0.73\%}, \textbf{5.09\%}, \textbf{3.63\%} and \textbf{2.80\%} on General-10, Plant-6, RS-8 and Fashion-20, respectively. While comparing the scores with different shots, SoftCPT-NATA wins in almost all cases. This demonstrates the effectiveness of the proposed multi-task prompt tuning method. Note that, the improvement on General-10 is relatively small, which could be attributed to the loose relation between different tasks in this dataset. For example, the images in DTD~\cite{DTD} are texture related while the images in Food101~\cite{Food101} are food related, thereby exhibiting significant differences.

\textbf{SoftCPT-NATA vs CoOp-MT.} It can be seen that the simple hard sharing of prompt contexts (CoOp-MT) is exceeded by the proposed soft sharing (SoftCPT) in most cases. The underlying reason is that hard sharing is too restrict to scale with the increasing complexity of more tasks. Meanwhile, it is prone to introduce the so-called negative transfer~\cite{MT_DNN_Survey} for tasks that have weak relations to others. We can also observe that the performance gap between SoftCPT-NATA and CoOp-MT increases with more training samples. This implies that the limited representation ability of CoOp-MT is a distinct problem when the training data is adequate.

\textbf{SoftCPT-NATA vs SoftCPT*}. We can also find that SoftCPT-NATA surpasses SoftCPT* in most cases. This further proves the importance of introducing pre-trained language model for providing the prior information of task relatedness.

\textbf{Comparison to State-of-the-art Methods.} CoCoOp, ProGrad, KgCoOp, and PLOT are recent advancements that build upon the foundation of CoOp. When comparing SoftCPT-NATA to CoCoOp, our method demonstrates impressive average accuracy improvements of 3.19\%, 6.35\%, 10.93\%, and 3.40\% on the General-10, Plant-6, RS-8, and Fashion-20 datasets, respectively. Similarly, against ProGrad, SoftCPT-NATA achieves accuracy boosts of 0.11\%, 5.22\%, 5.79\%, and 1.96\% on the corresponding datasets. In the case of PLOT, while SoftCPT-NATA slightly lags behind on the General-10 dataset, this is understandable given the limited benefit from multi-task learning due to the weak relationship among the tasks. Nevertheless, on the three specialized datasets, SoftCPT-NATA surpasses PLOT in terms of mean accuracy, thereby validating the effectiveness of multi-task learning in these contexts. It's worth noting that PLOT incorporates multiple contexts through optimal transport, making it a strong baseline. Moreover, since SoftCPT and PLOT explore distinct optimization directions, there is potential for integrating these two techniques to further enhance performance.

\textbf{Linear Probe vs Prompt Tuning.} It is noteworthy that the simple linear probe (LP-CLIP) exhibits impressive performance on three specialized datasets, even surpassing CoOp, thus highlighting its robustness in resisting domain shifts. Nevertheless, it's crucial to acknowledge that linear probe utilizes the validation set for hyper-parameter tuning, whereas prompt tuning methods refrain from using it. Furthermore, despite its strong performance, linear probe is significantly outweighed by more contemporary prompt tuning techniques, including PLOT and our proposed method. In summary, prompt tuning methods can yield comparable or even superior results to linear probe, demonstrating the significance and value of prompt tuning.

\begin{figure*}[t!]
	\begin{subfigure}[t]{0.24\linewidth}
		\centering
		\includegraphics[width=1.6in]{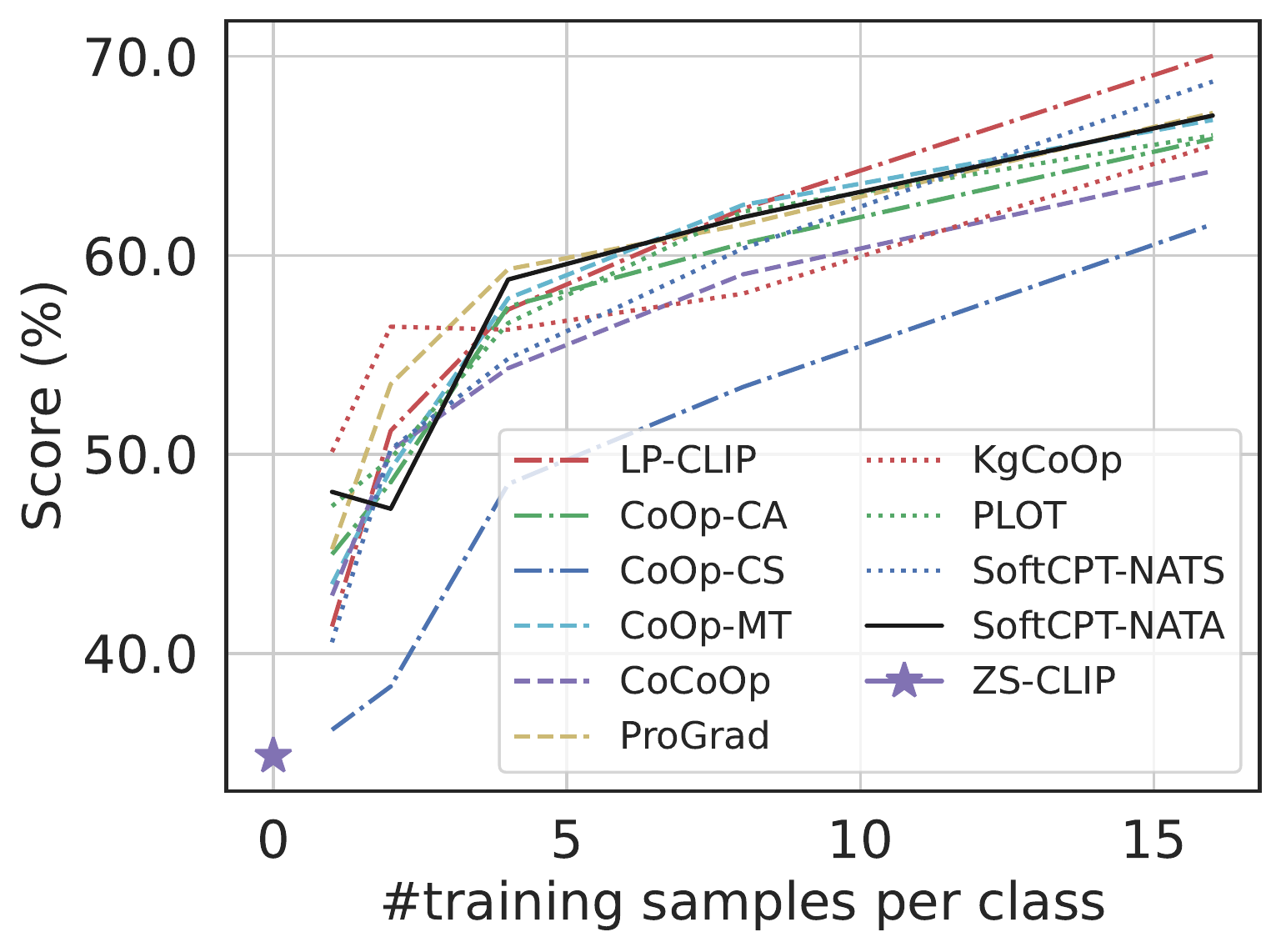}
		\caption{pants type}
	\end{subfigure}
	\begin{subfigure}[t]{0.24\linewidth}
		\centering
		\includegraphics[width=1.6in]{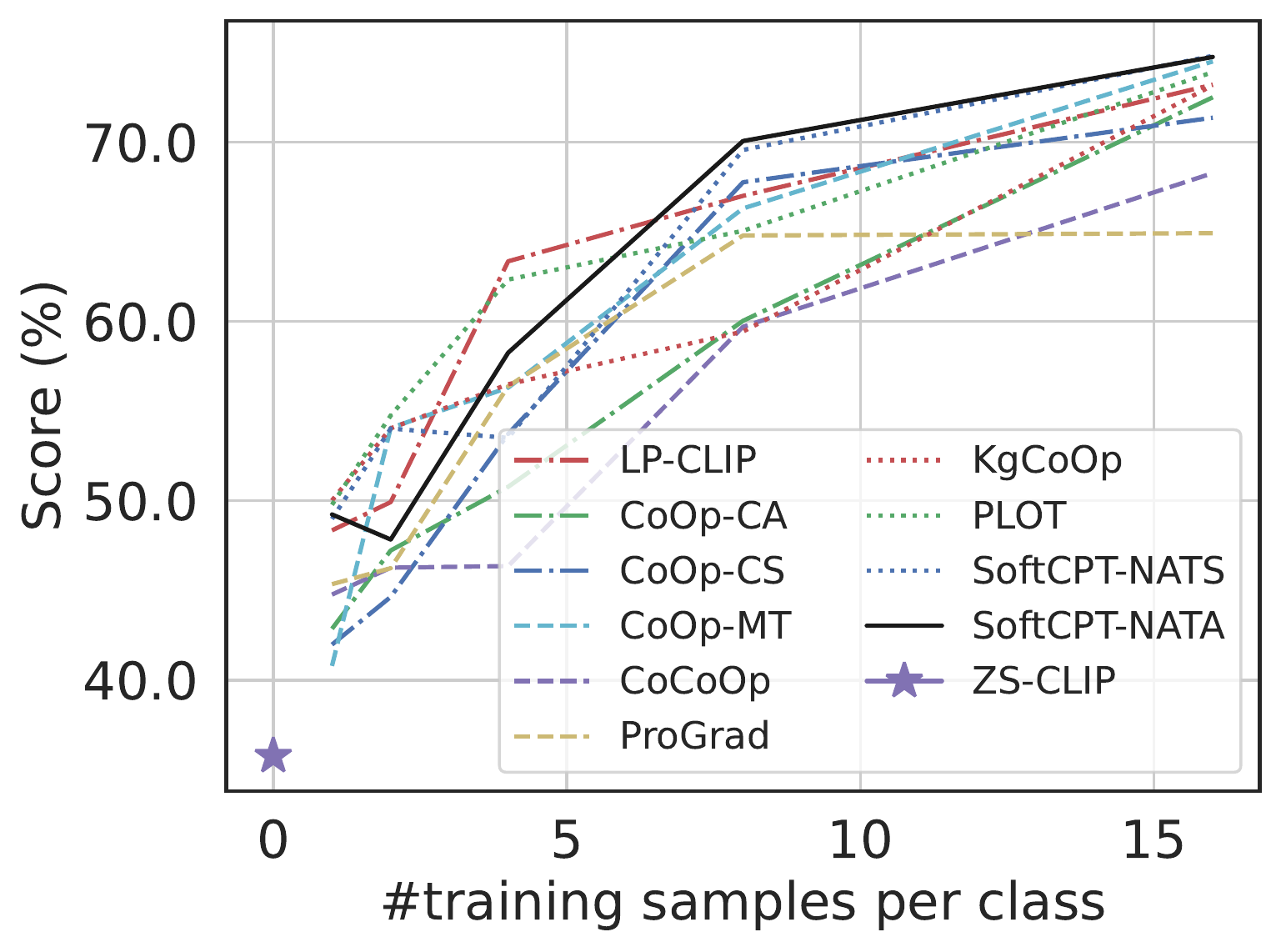}
		\caption{pants length}
	\end{subfigure}
	\begin{subfigure}[t]{0.24\linewidth}
		\centering
		\includegraphics[width=1.6in]{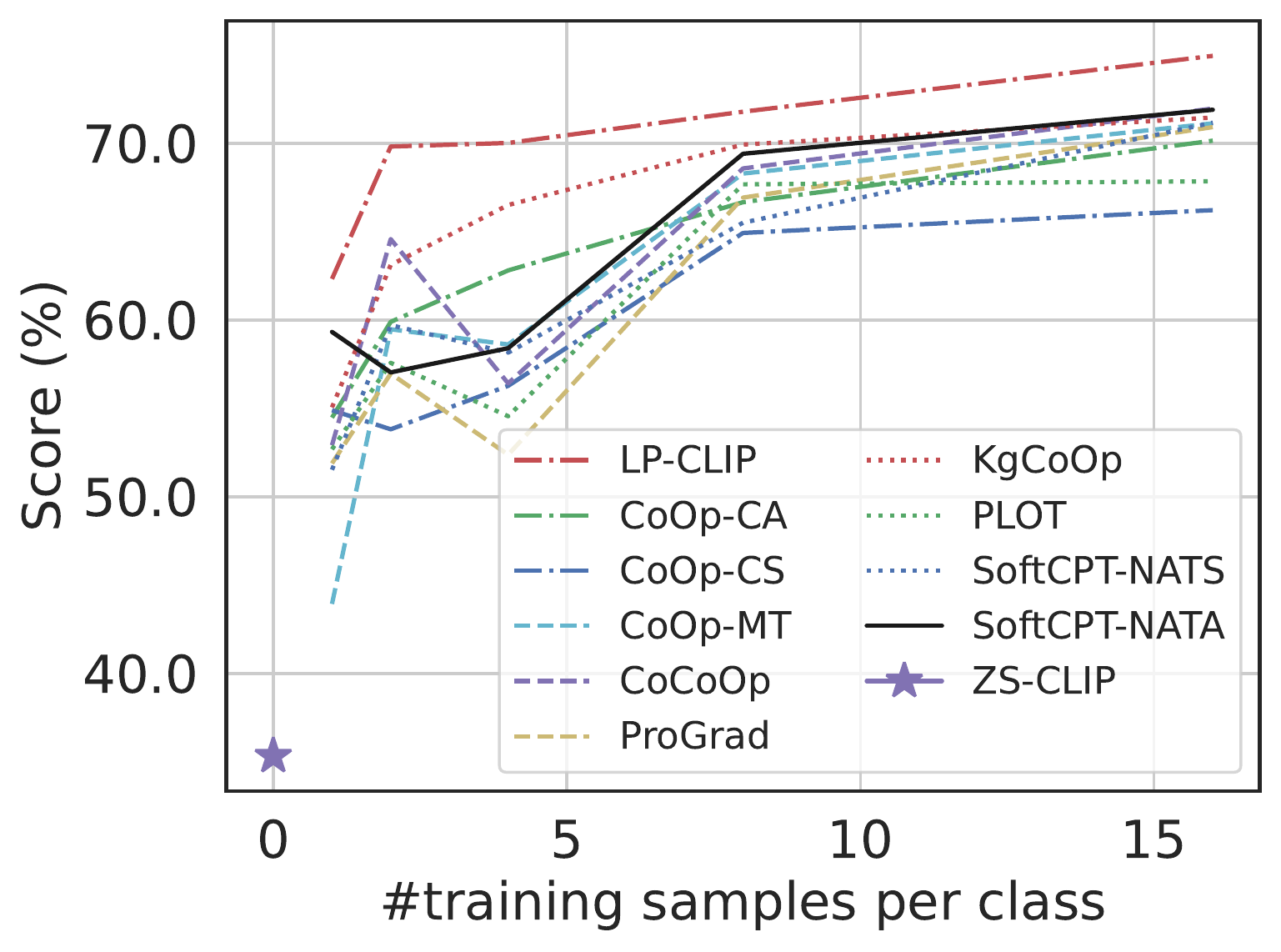}
		\caption{waist type}
	\end{subfigure}
	\begin{subfigure}[t]{0.24\linewidth}
		\centering
		\includegraphics[width=1.6in]{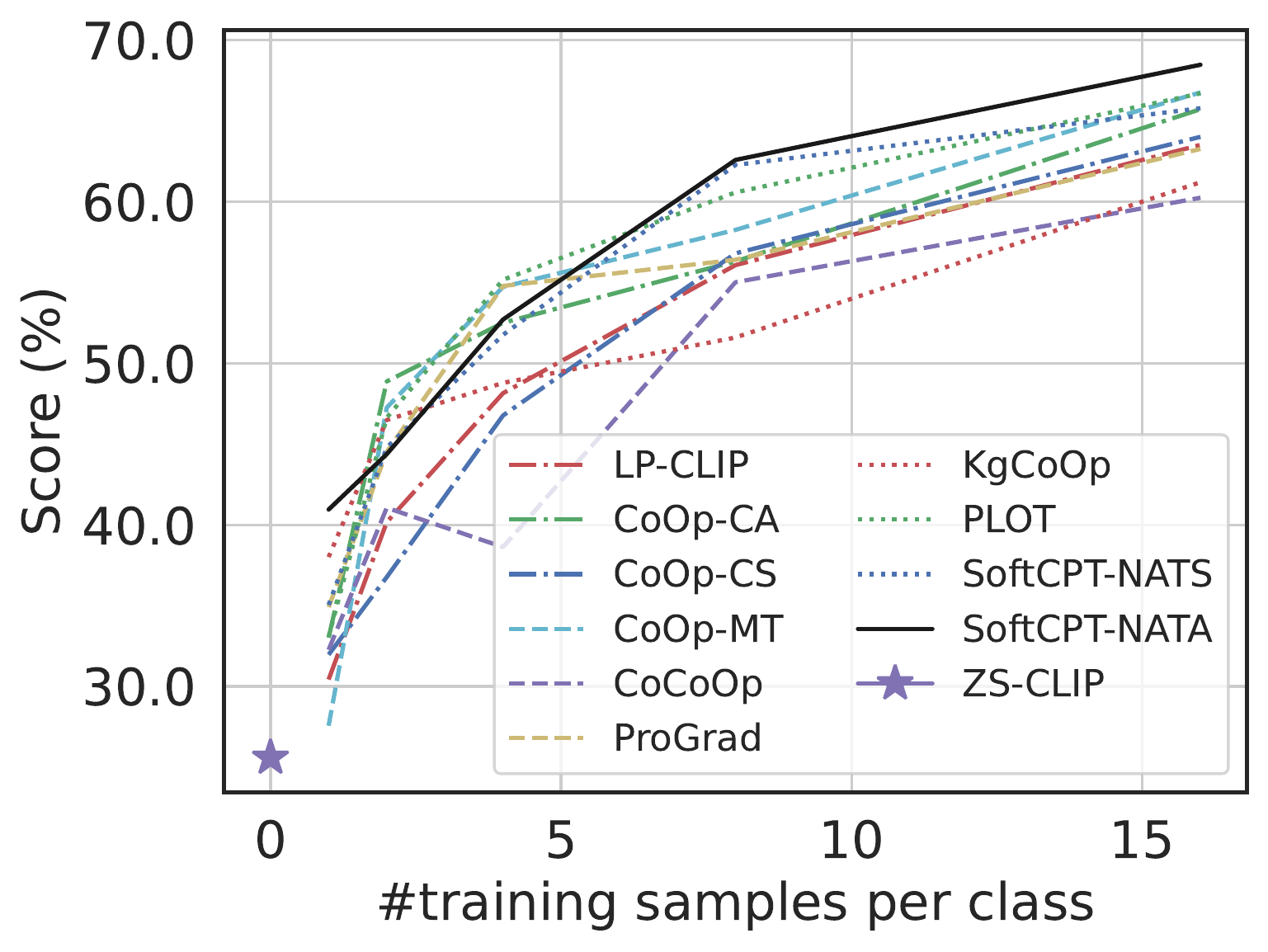}
		\caption{collar type}
	\end{subfigure}
	\\
	
	\begin{subfigure}[t]{0.24\linewidth}
		\centering
		\includegraphics[width=1.6in]{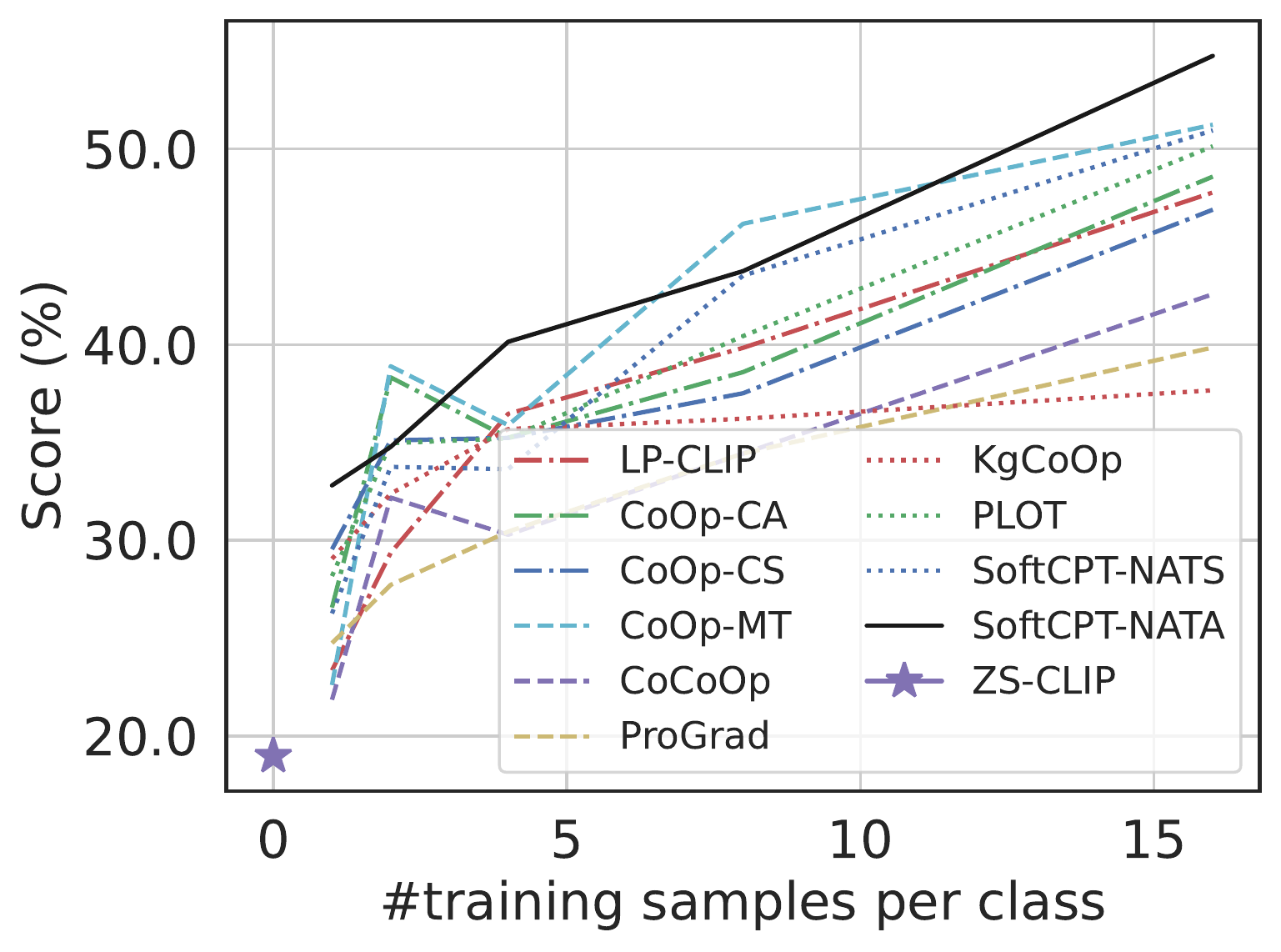}
		\caption{sleeve type}
	\end{subfigure}
	\begin{subfigure}[t]{0.24\linewidth}
		\centering
		\includegraphics[width=1.6in]{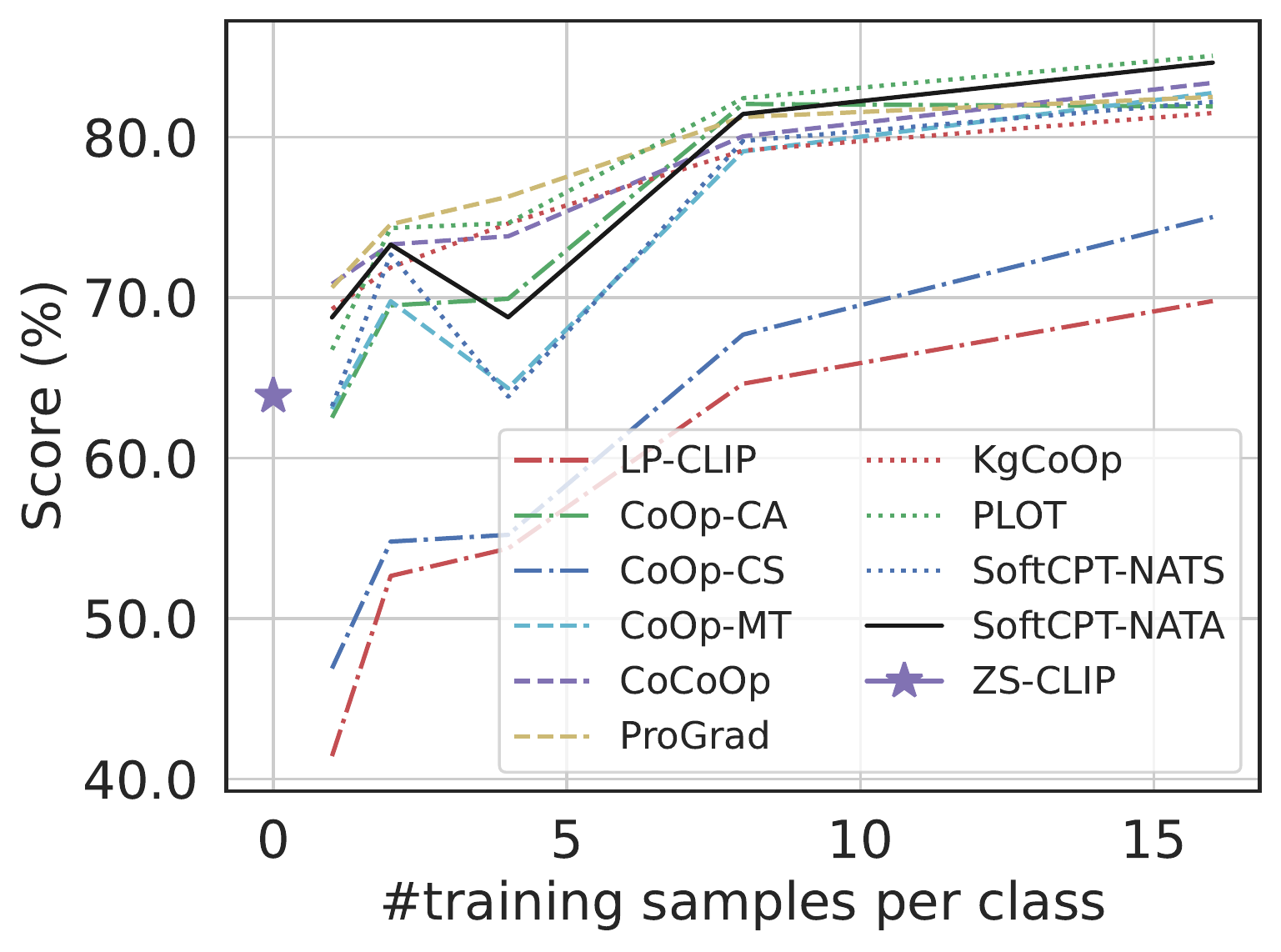}
		\caption{sleeve length}
	\end{subfigure}
	\begin{subfigure}[t]{0.24\linewidth}
		\centering
		\includegraphics[width=1.6in]{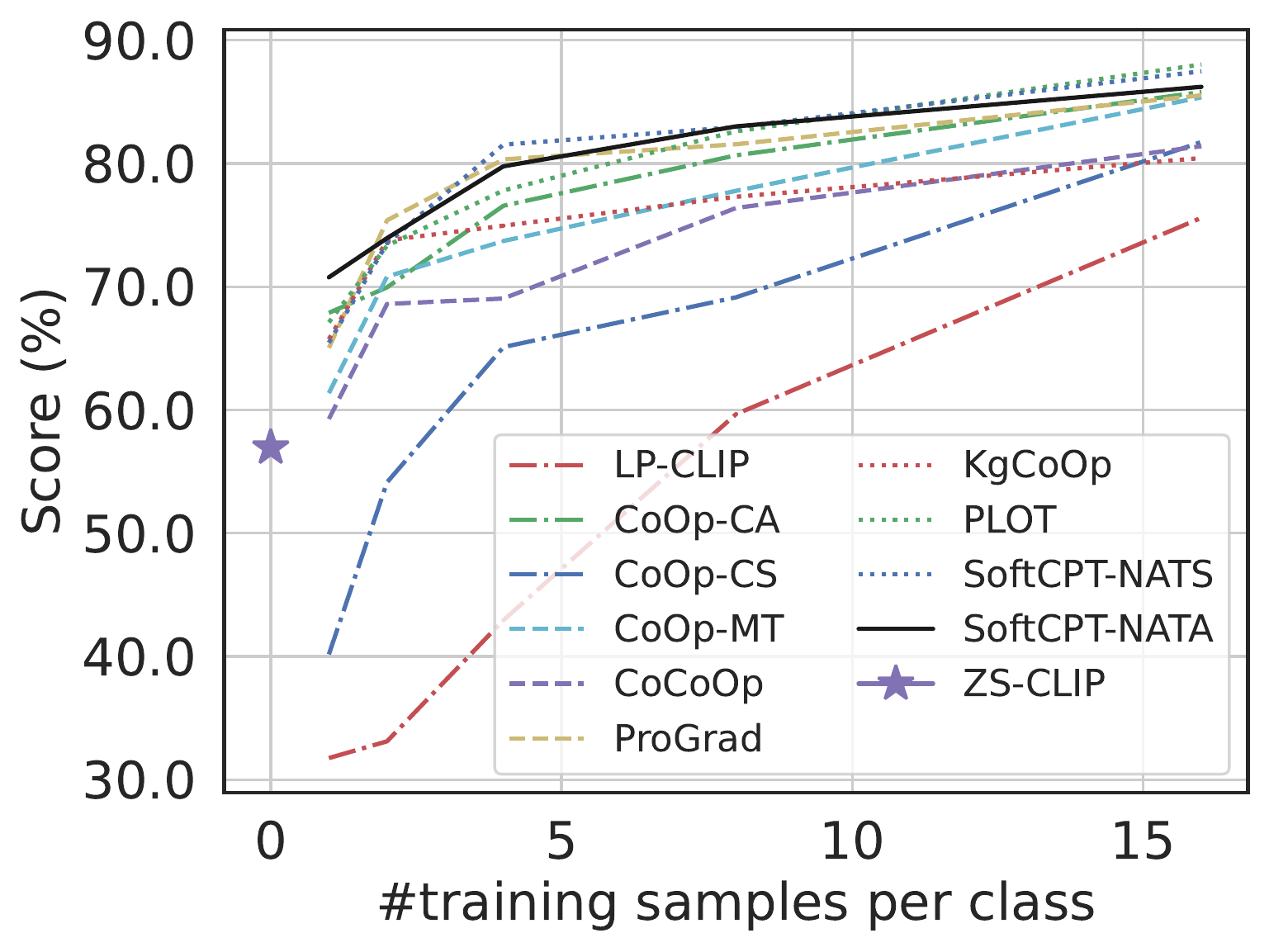}
		\caption{top pattern}
	\end{subfigure}
	\begin{subfigure}[t]{0.24\linewidth}
		\centering
		\includegraphics[width=1.6in]{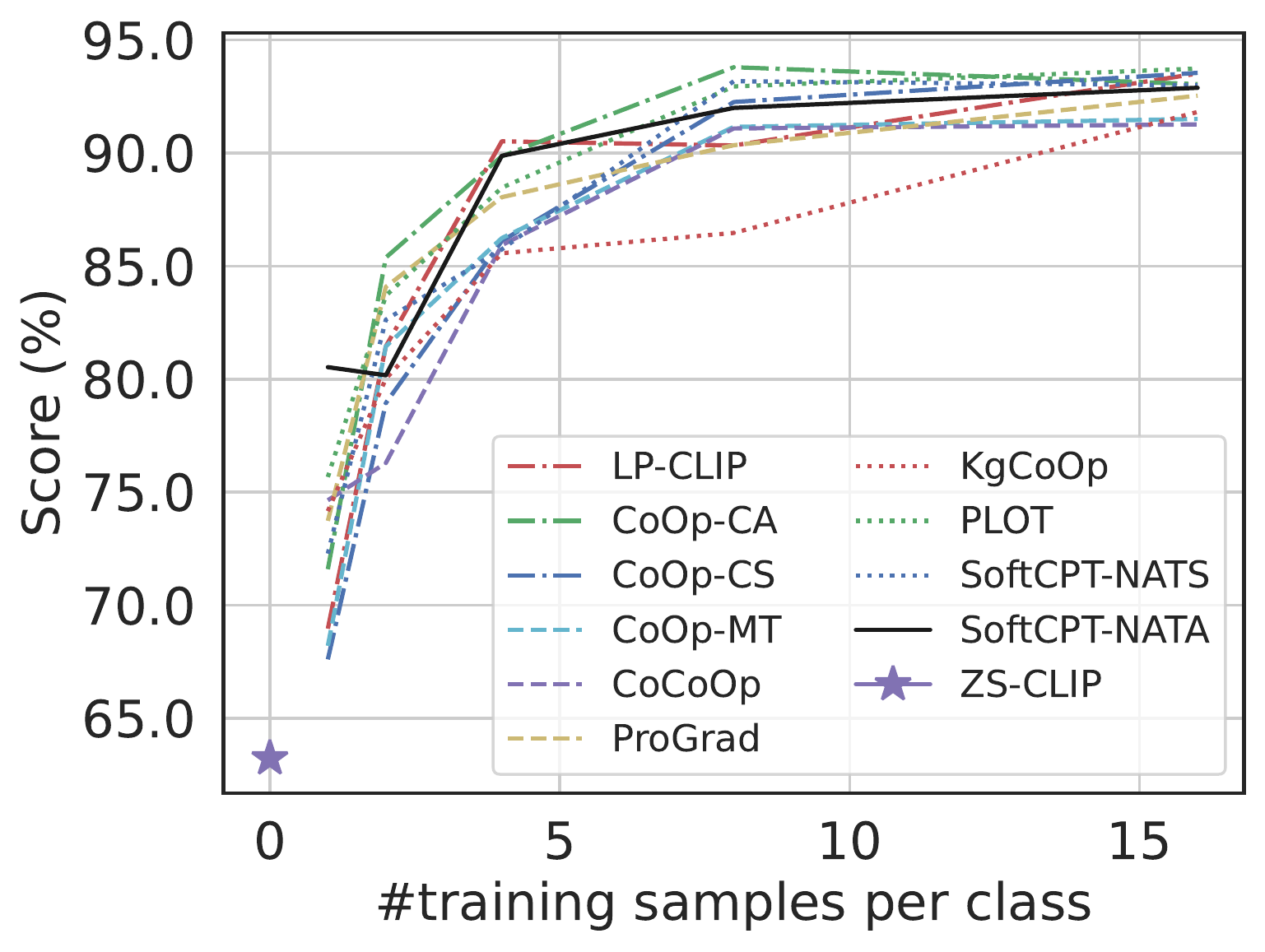}
		\caption{shoe material}
	\end{subfigure}
	\\

    \begin{subfigure}[t]{0.24\linewidth}
    	\centering
    	\includegraphics[width=1.6in]{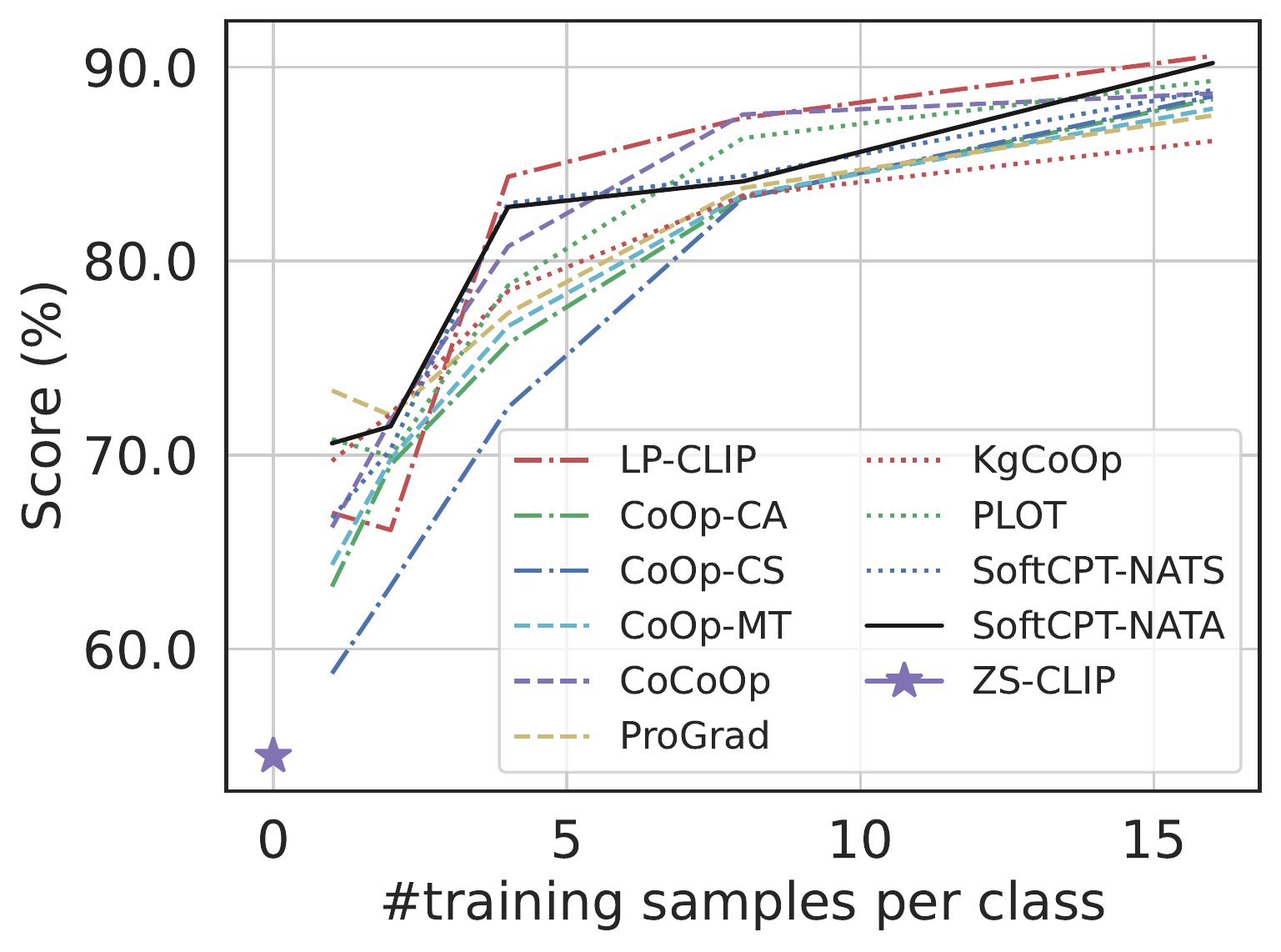}
    	\caption{shoe styleshoe style}
    \end{subfigure}
    \begin{subfigure}[t]{0.24\linewidth}
    	\centering
    	\includegraphics[width=1.6in]{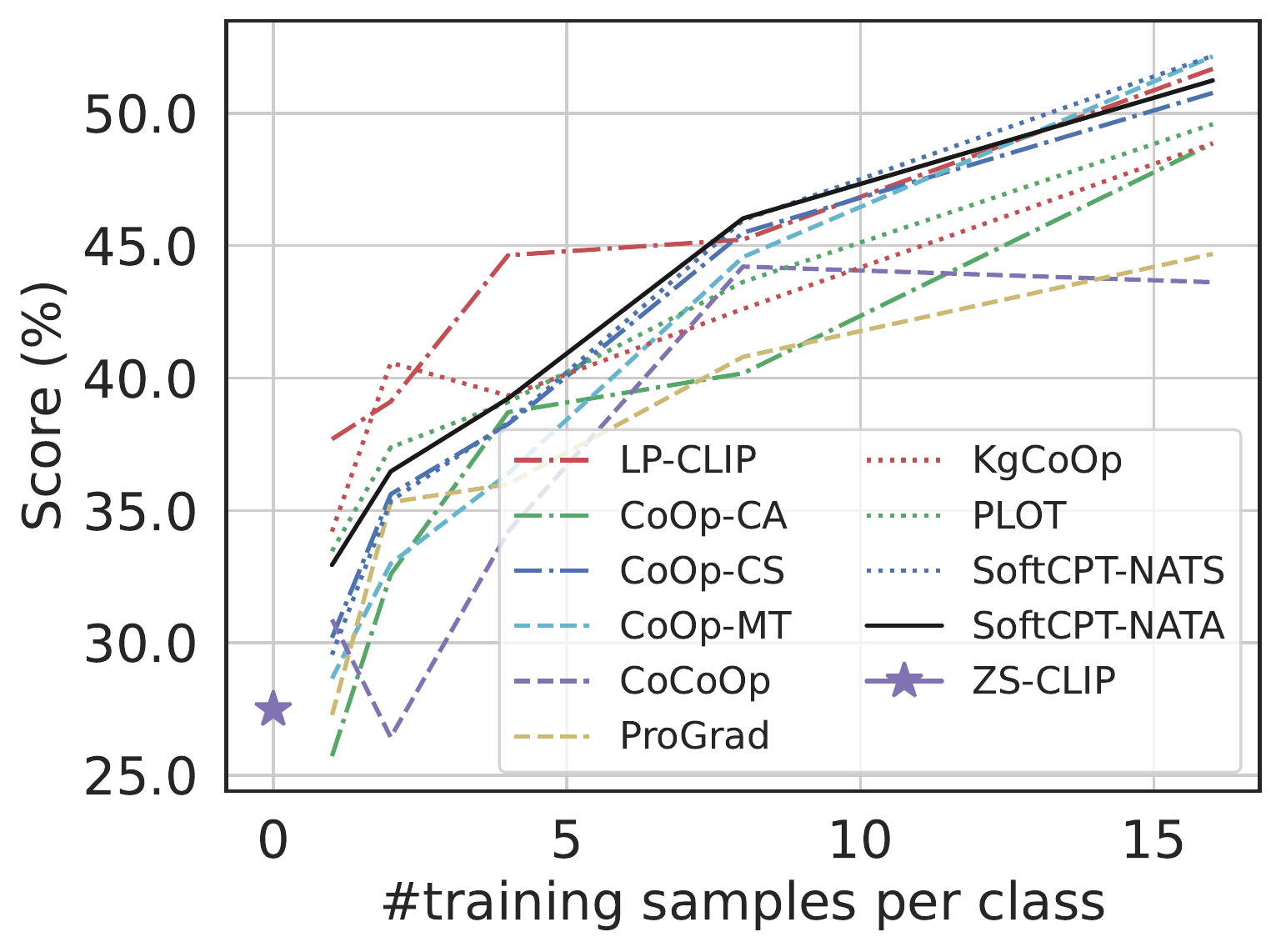}
    	\caption{heel shape}
    \end{subfigure}
    \begin{subfigure}[t]{0.24\linewidth}
    	\centering
    	\includegraphics[width=1.6in]{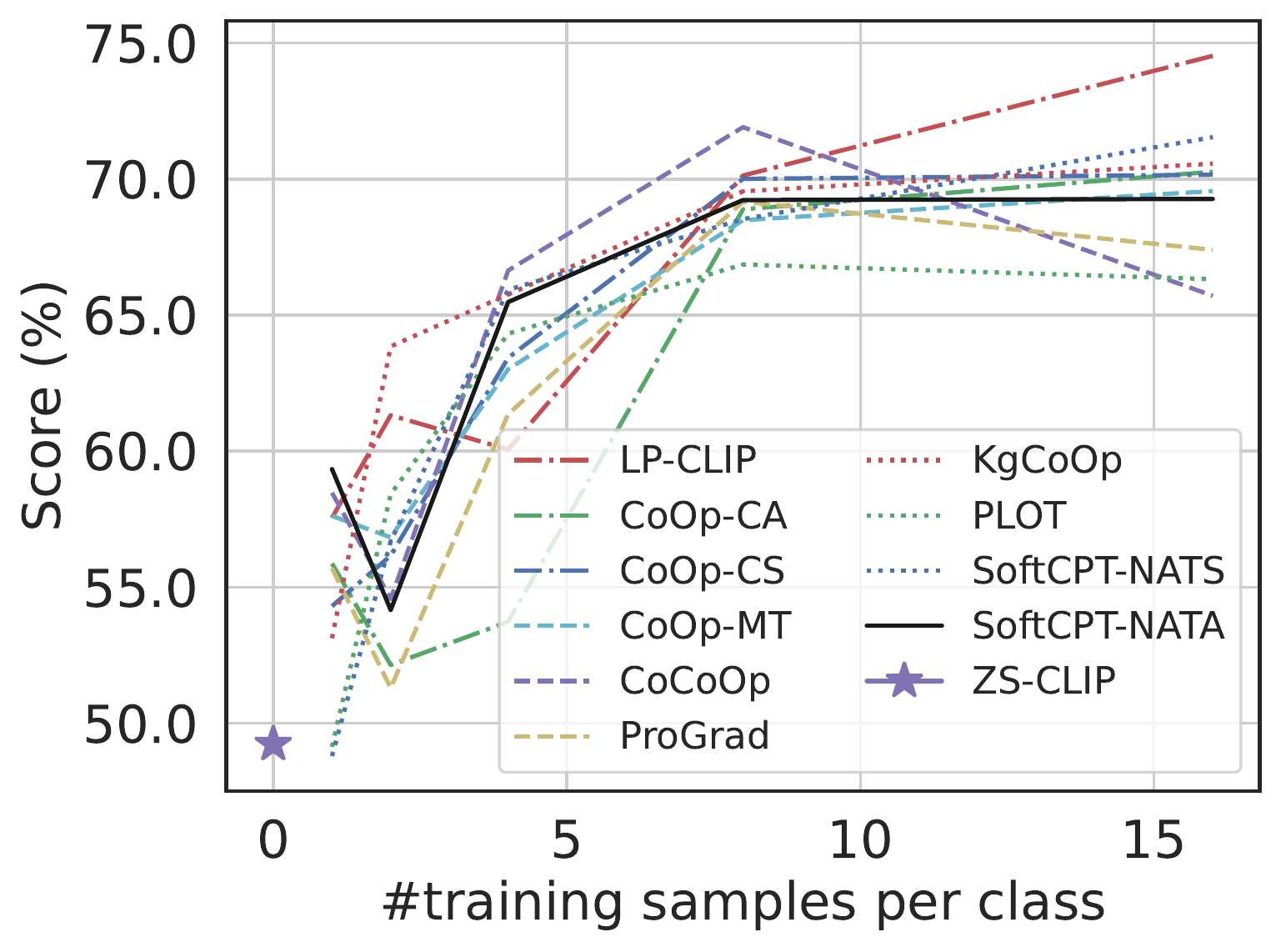}
    	\caption{heel thickness}
    \end{subfigure}
    \begin{subfigure}[t]{0.24\linewidth}
    	\centering
    	\includegraphics[width=1.6in]{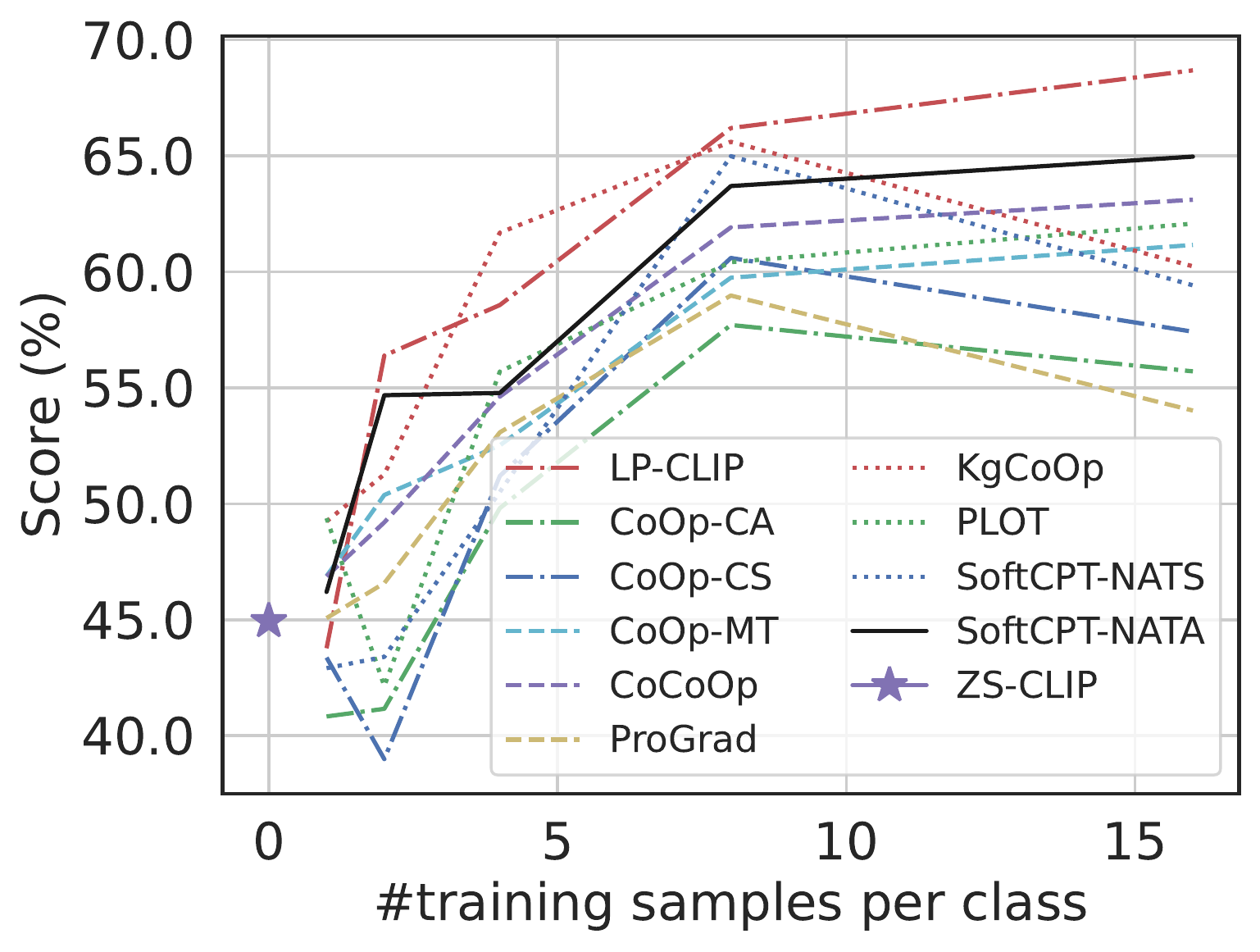}
    	\caption{heel height}
    \end{subfigure}
    \\
    
    \begin{subfigure}[t]{0.24\linewidth}
    	\centering
    	\includegraphics[width=1.6in]{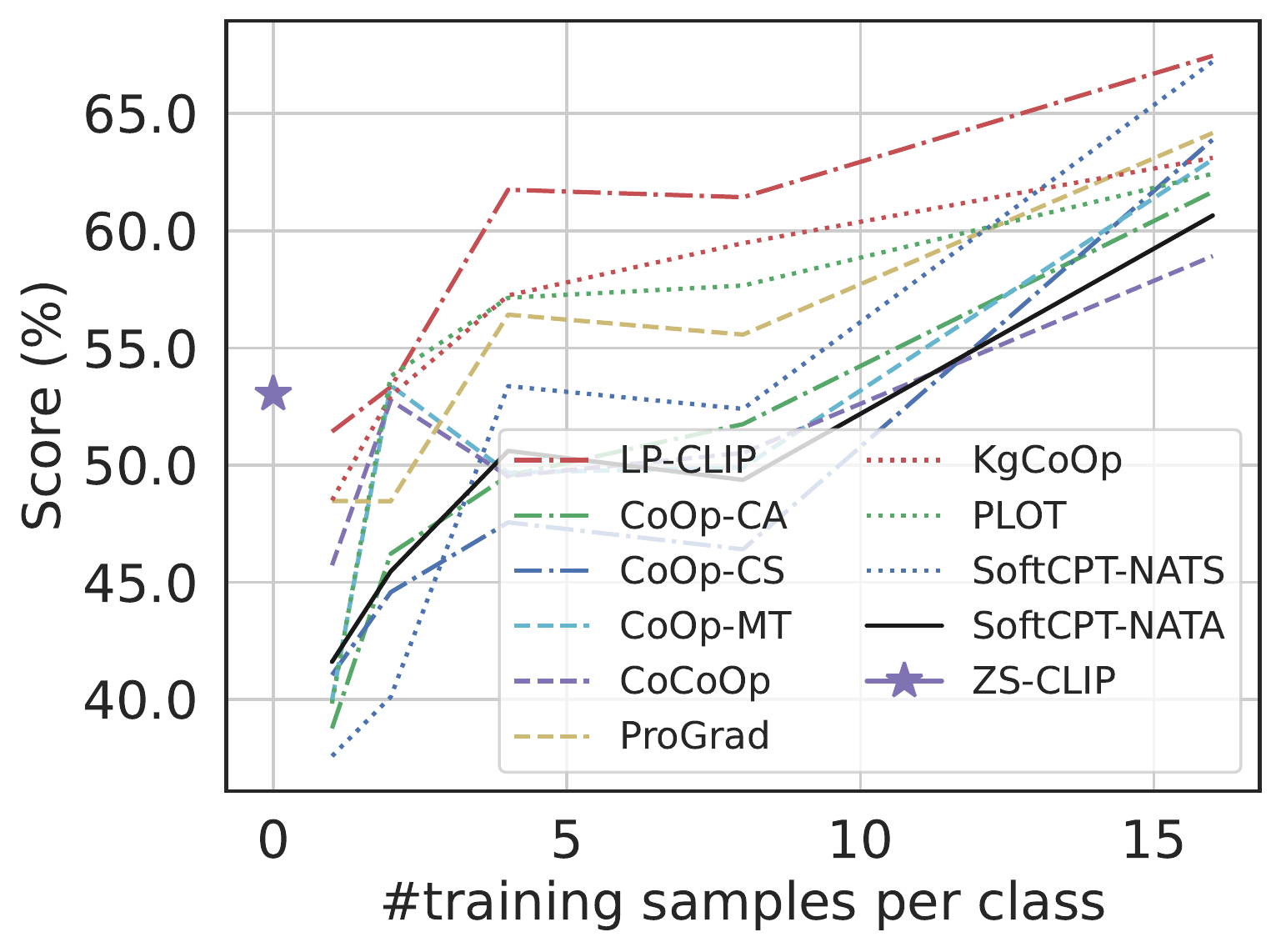}
    	\caption{upper height}
    \end{subfigure}
    \begin{subfigure}[t]{0.24\linewidth}
    	\centering
    	\includegraphics[width=1.6in]{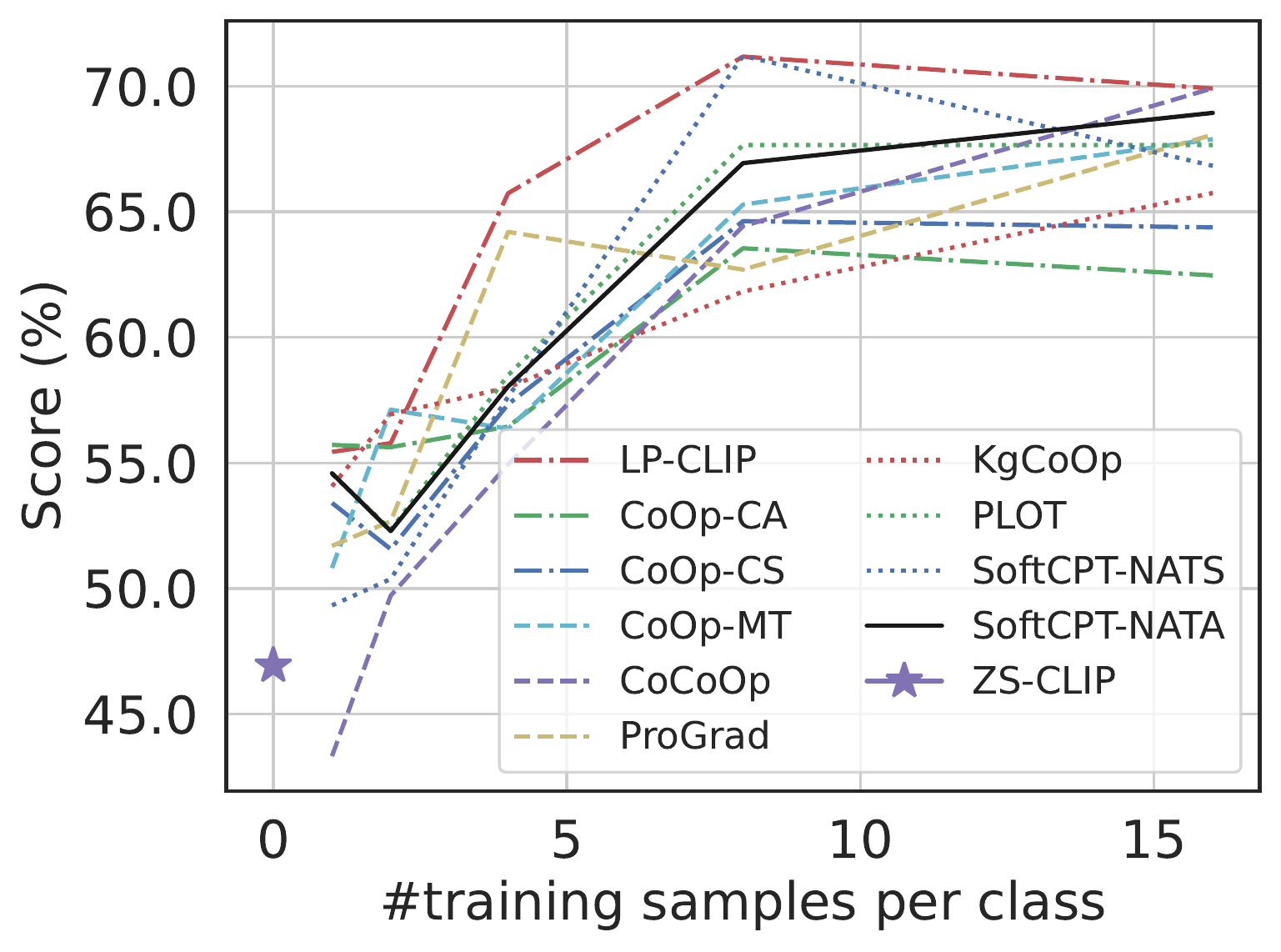}
    	\caption{toe cap style}
    \end{subfigure}
    \begin{subfigure}[t]{0.24\linewidth}
    	\centering
    	\includegraphics[width=1.6in]{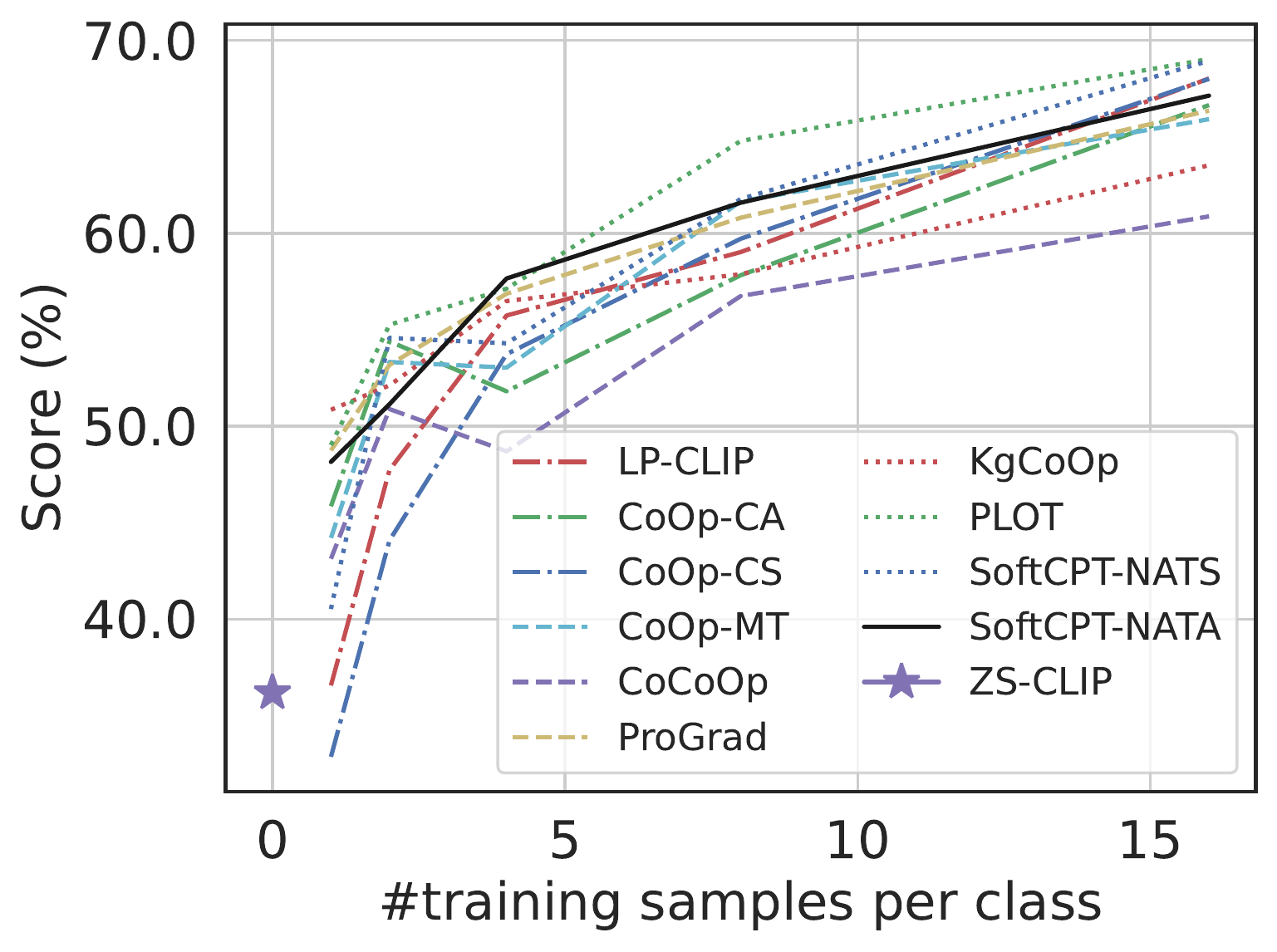}
    	\caption{hat style}
    \end{subfigure}
    \begin{subfigure}[t]{0.24\linewidth}
    	\centering
    	\includegraphics[width=1.6in]{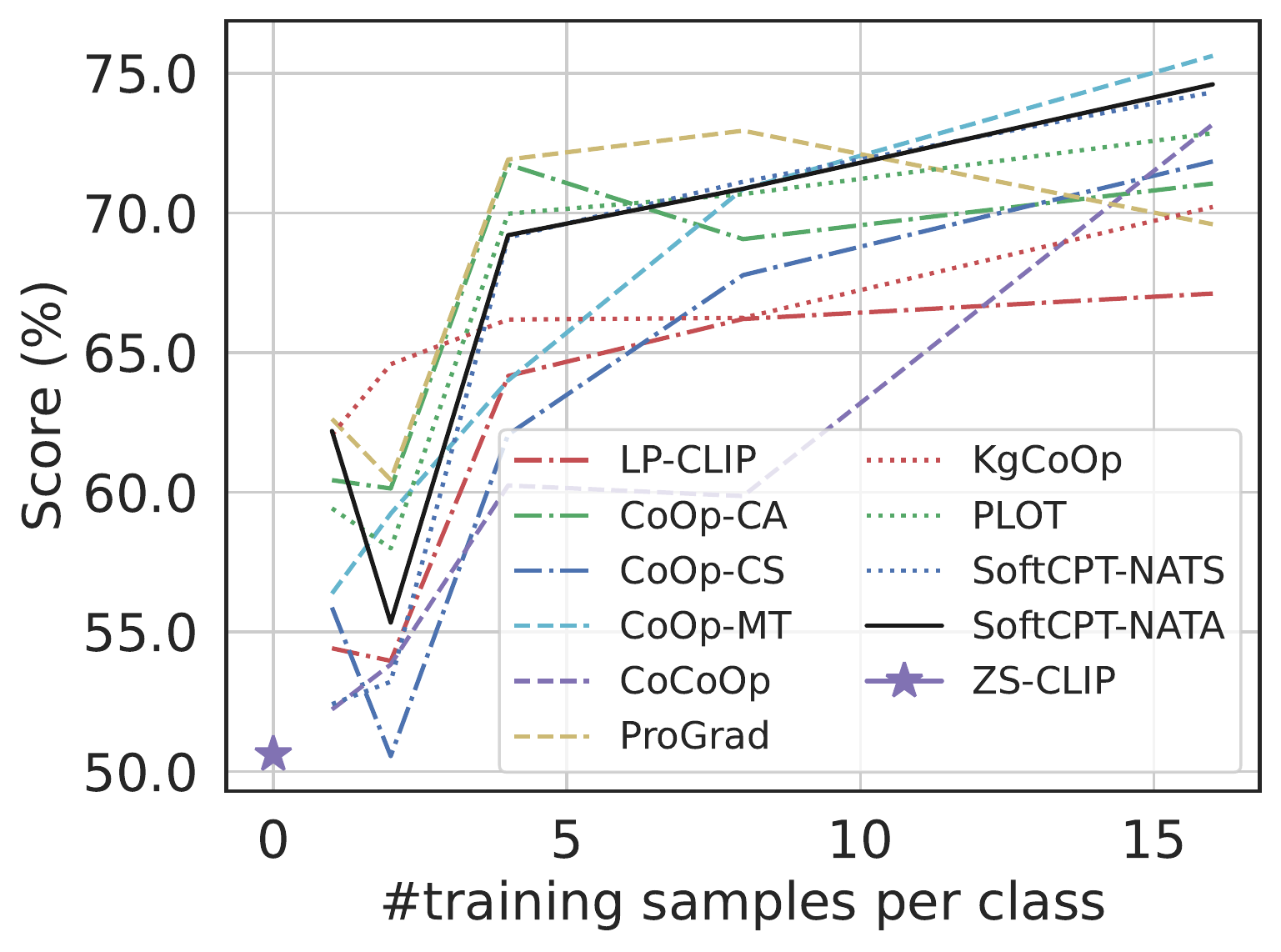}
    	\caption{socks length}
    \end{subfigure}
    \\

    \begin{subfigure}[t]{0.24\linewidth}
    	\centering
    	\includegraphics[width=1.6in]{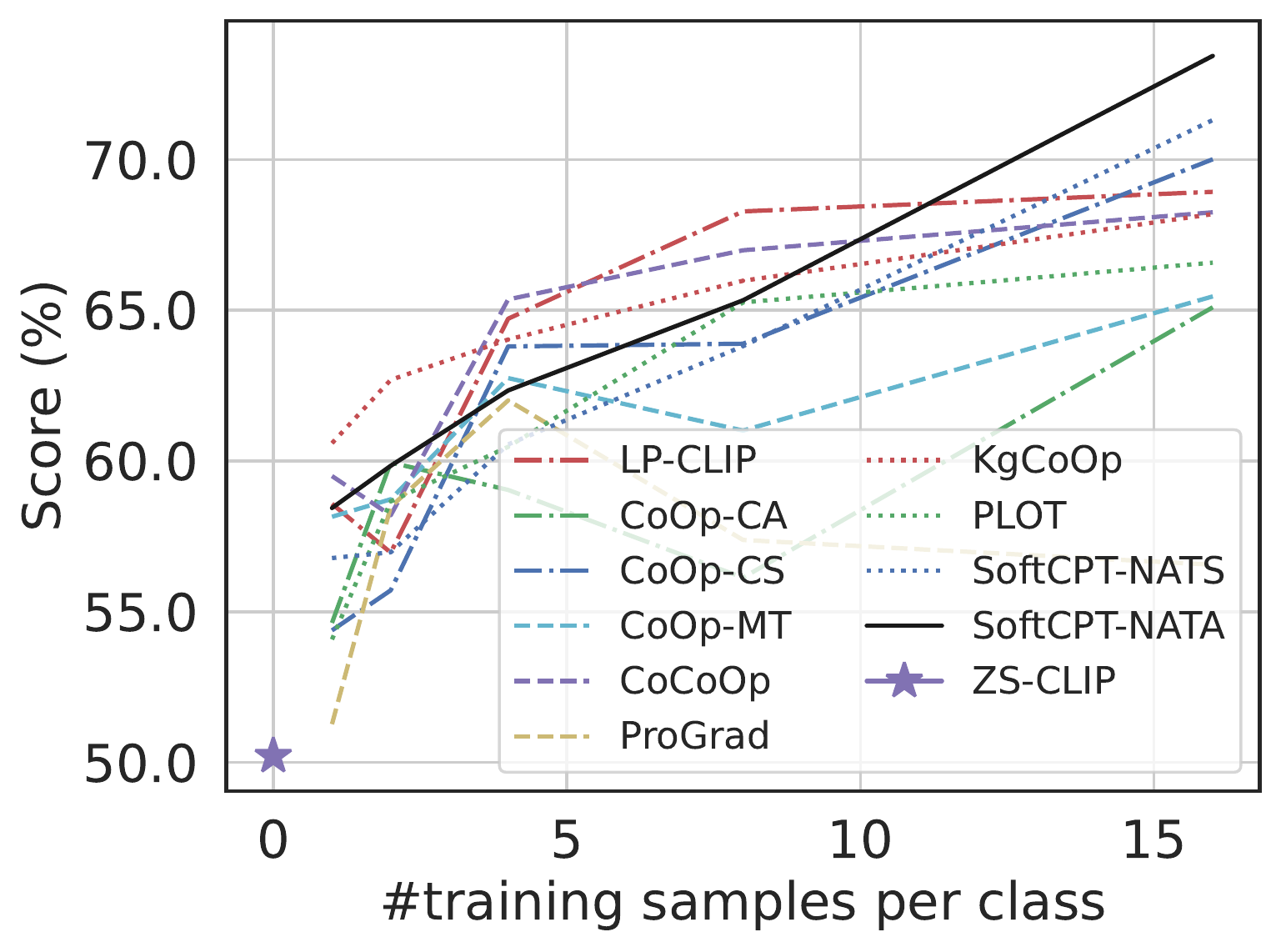}
    	\caption{socks type}
    \end{subfigure}
    \begin{subfigure}[t]{0.24\linewidth}
    	\centering
    	\includegraphics[width=1.6in]{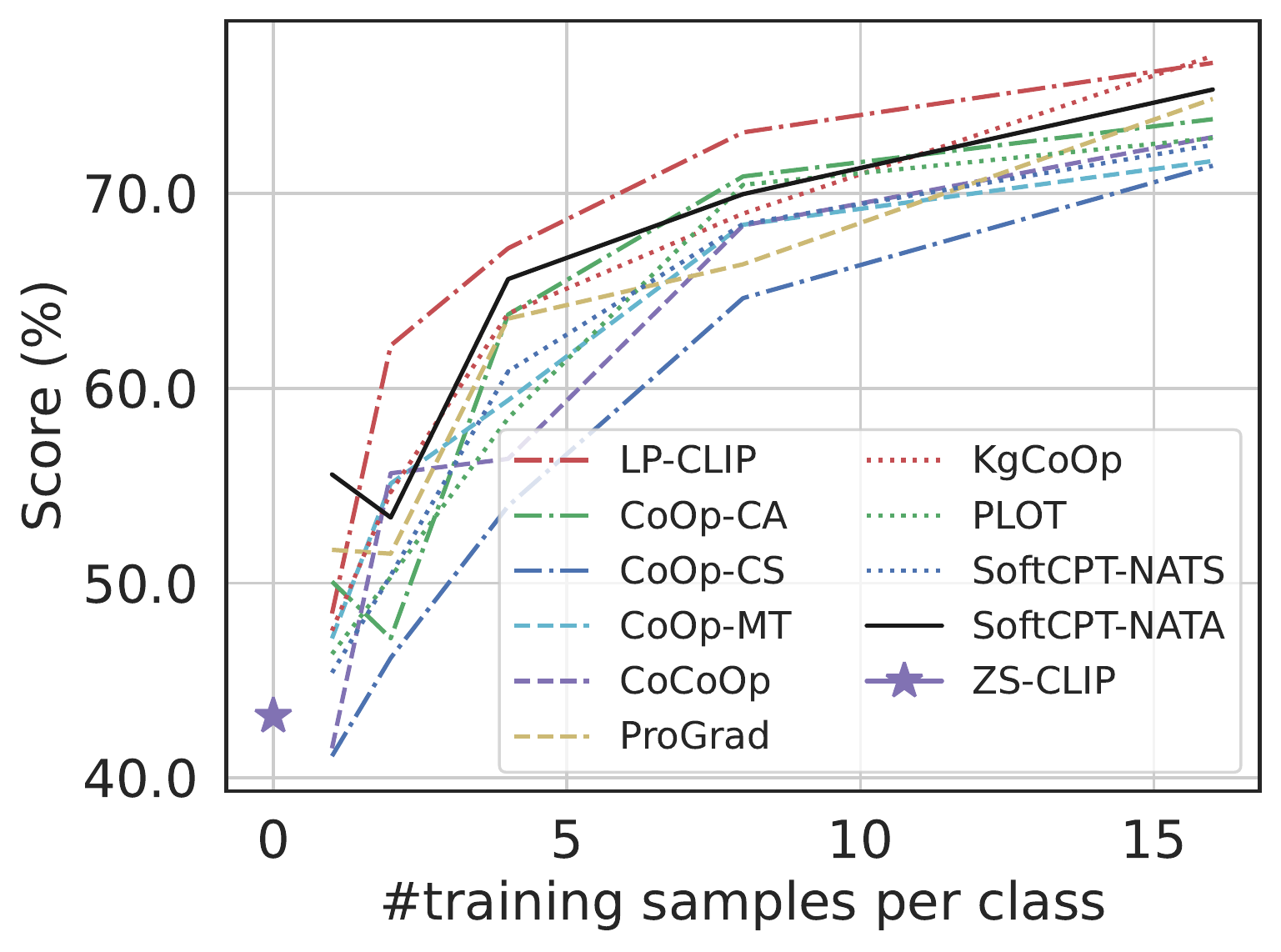}
    	\caption{skirt length}
    \end{subfigure}
    \begin{subfigure}[t]{0.24\linewidth}
    	\centering
    	\includegraphics[width=1.6in]{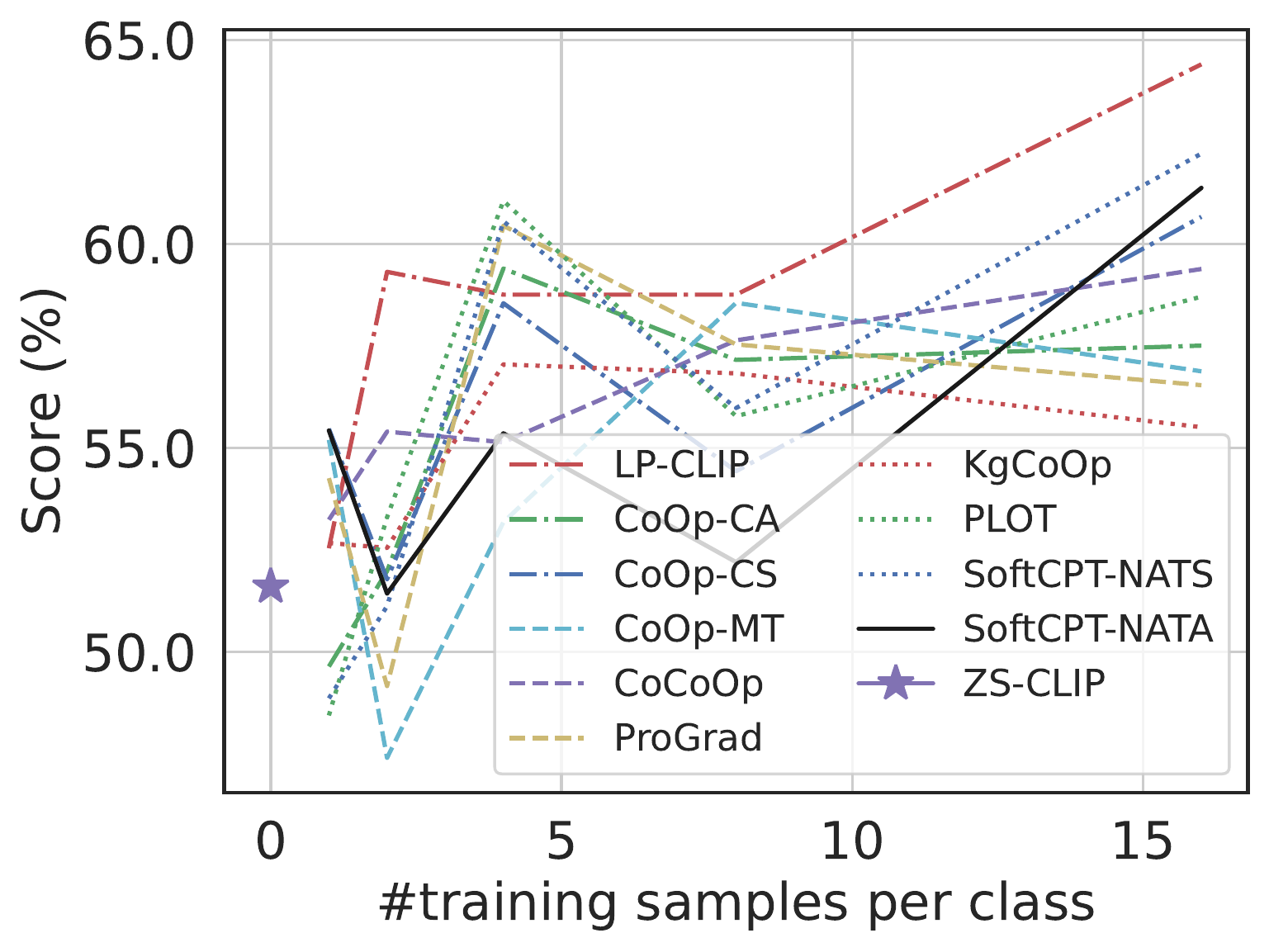}
    	\caption{number of button rows}
    \end{subfigure}
    \begin{subfigure}[t]{0.24\linewidth}
    	\centering
    	\includegraphics[width=1.6in]{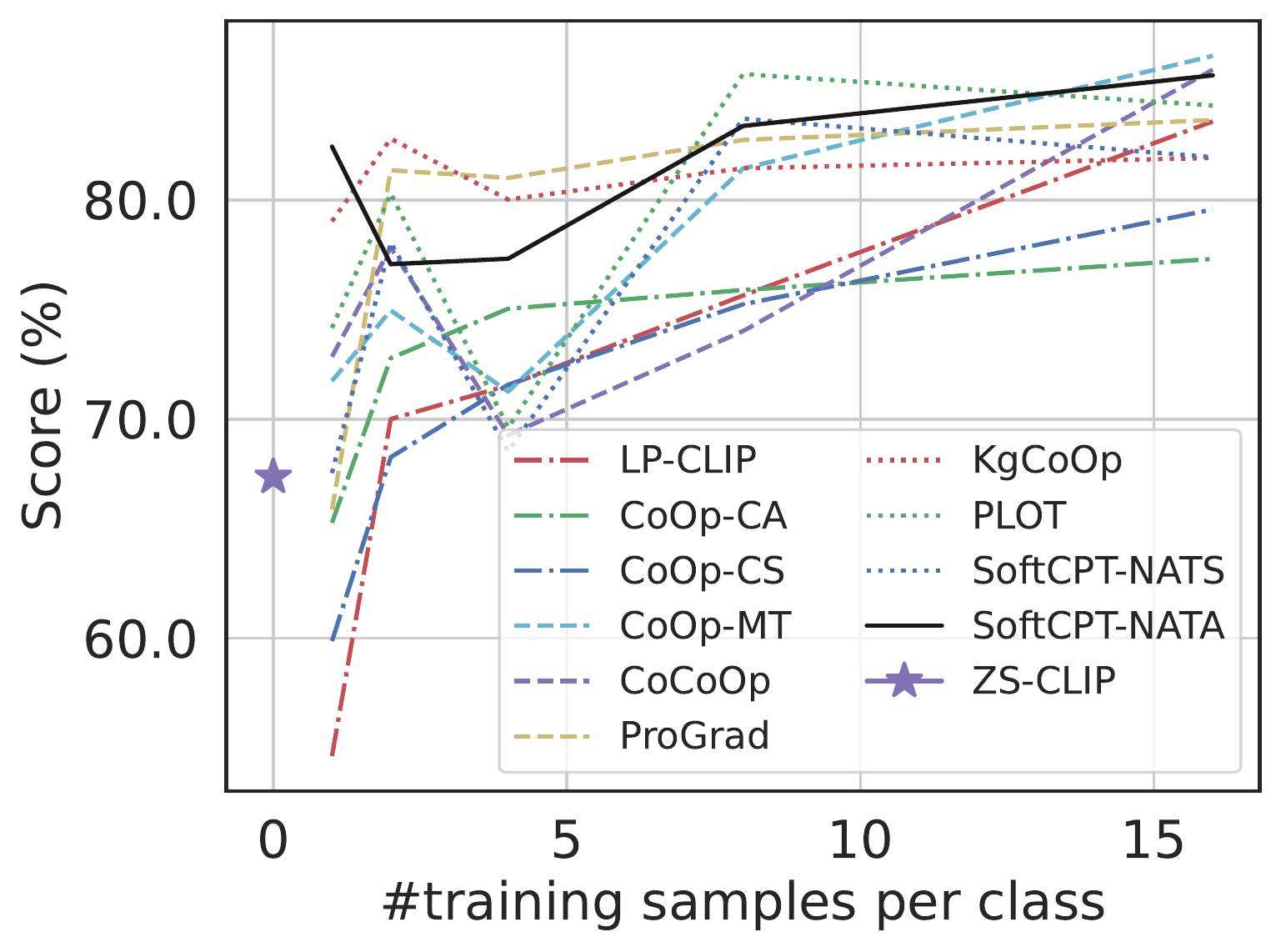}
    	\caption{underwear style}
    \end{subfigure}
	
	\caption{Per-task results on Fashion-20.}
	\label{fig:fashionv1_per_task_results}
\end{figure*}

\begin{figure*}[!t]
	\centering
	\begin{subfigure}[t]{0.4\linewidth}
		\centering
		\includegraphics[width=2.2in]{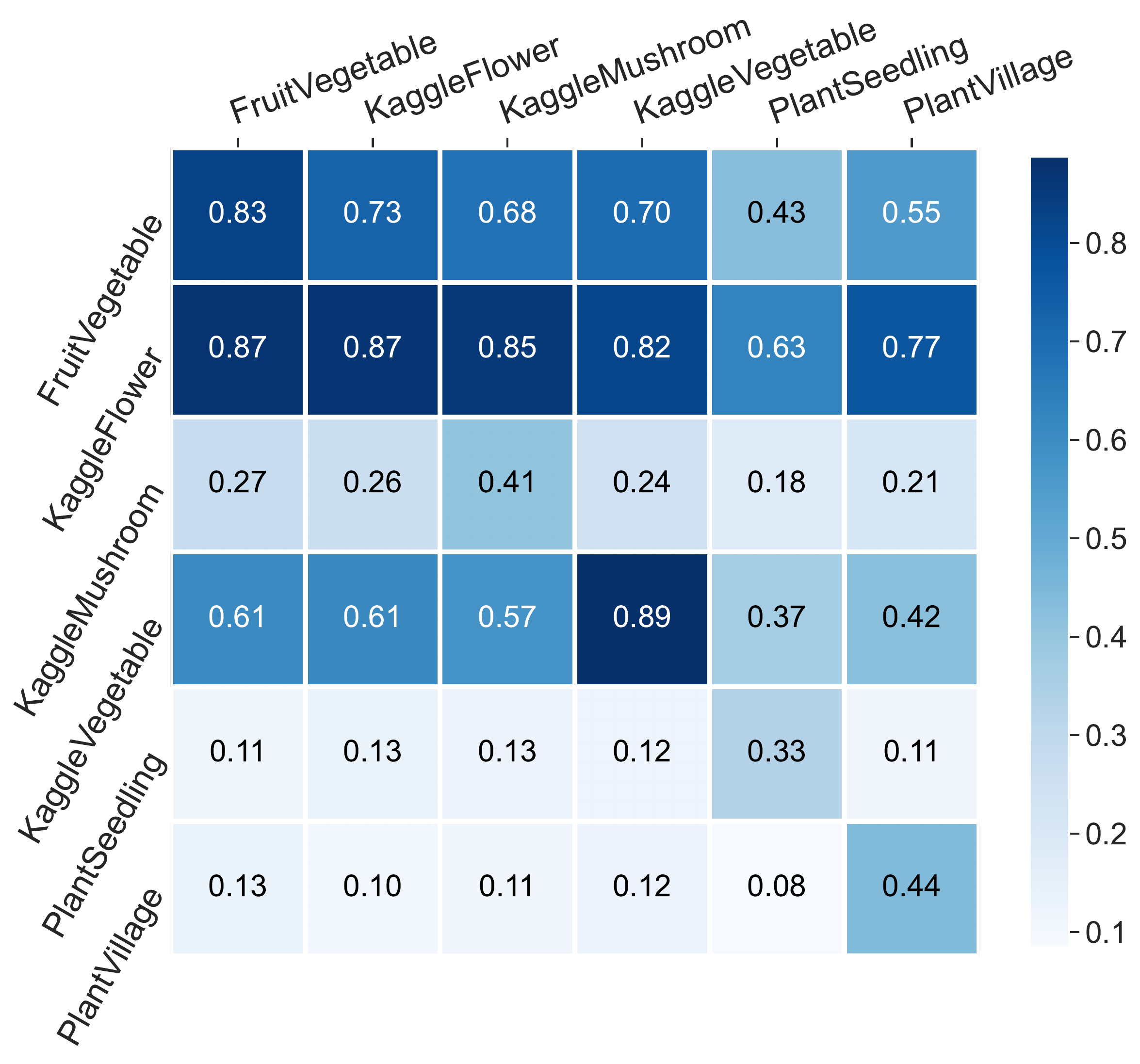}
		\caption{$\boldsymbol{S}$}
		\label{fig:corr_map_a}
	\end{subfigure}
	\begin{subfigure}[t]{0.4\linewidth}
		\centering
	    \includegraphics[width=2.2in]{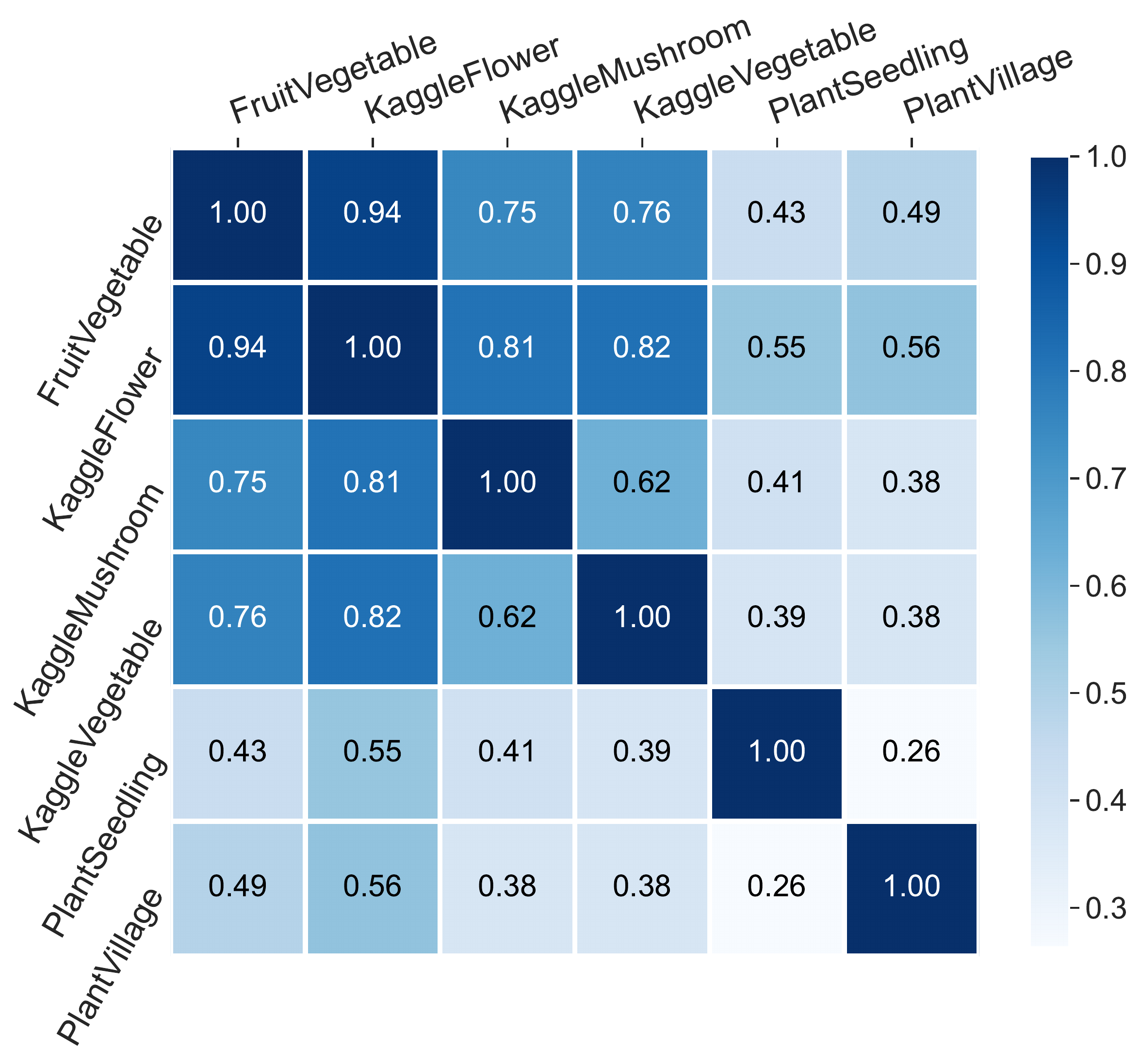}
	    \caption{$\boldsymbol{S}_\text{oracle}$}
	    \label{fig:corr_map_b}
	\end{subfigure}
	\\
	\begin{subfigure}[t]{0.4\linewidth}
		\centering
		\includegraphics[width=2.2in]{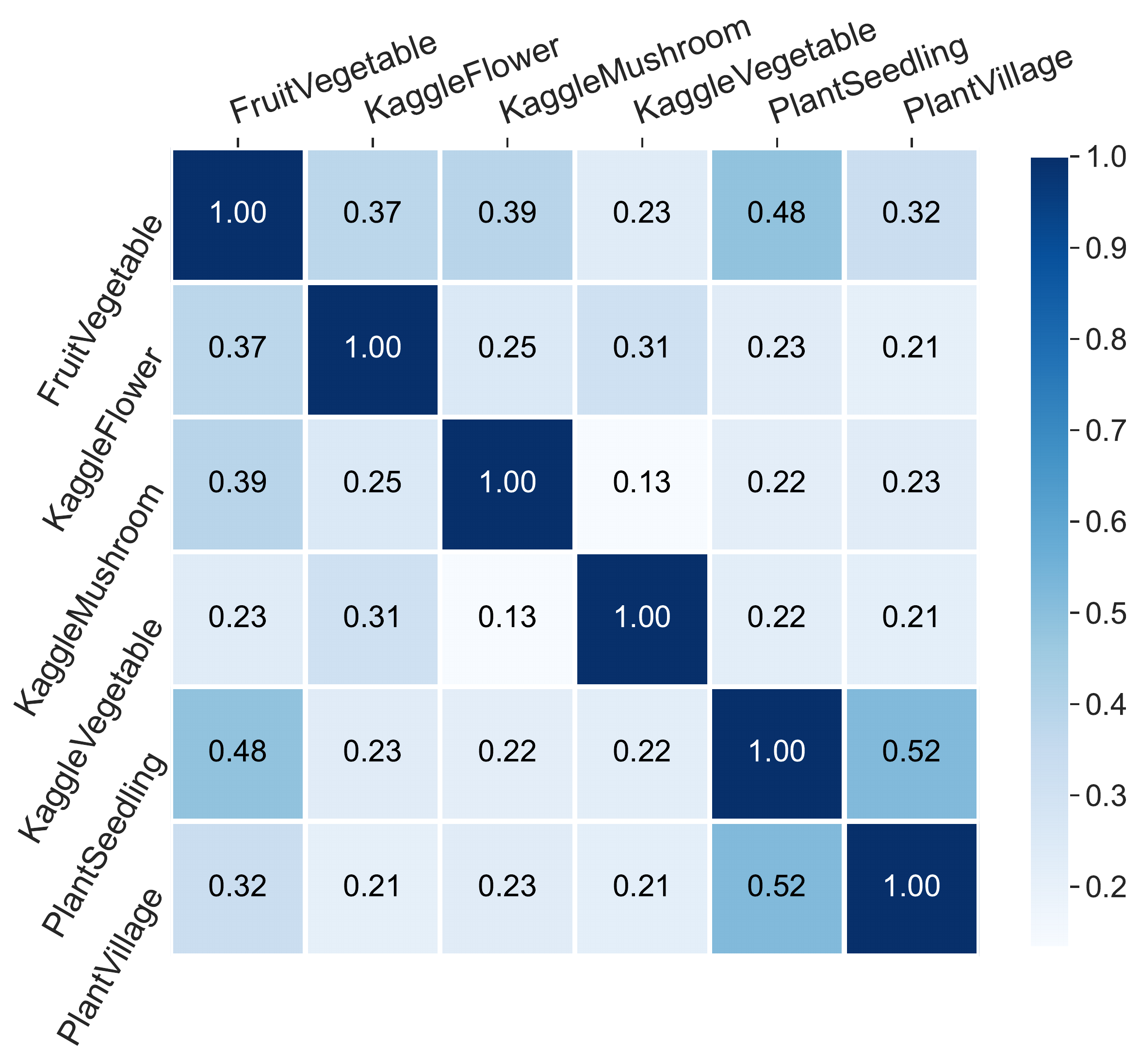}
		\caption{$\boldsymbol{S}_\text{st}$}
		\label{fig:corr_map_c}
	\end{subfigure}
	\begin{subfigure}[t]{0.4\linewidth}
		\centering
		\includegraphics[width=2.2in]{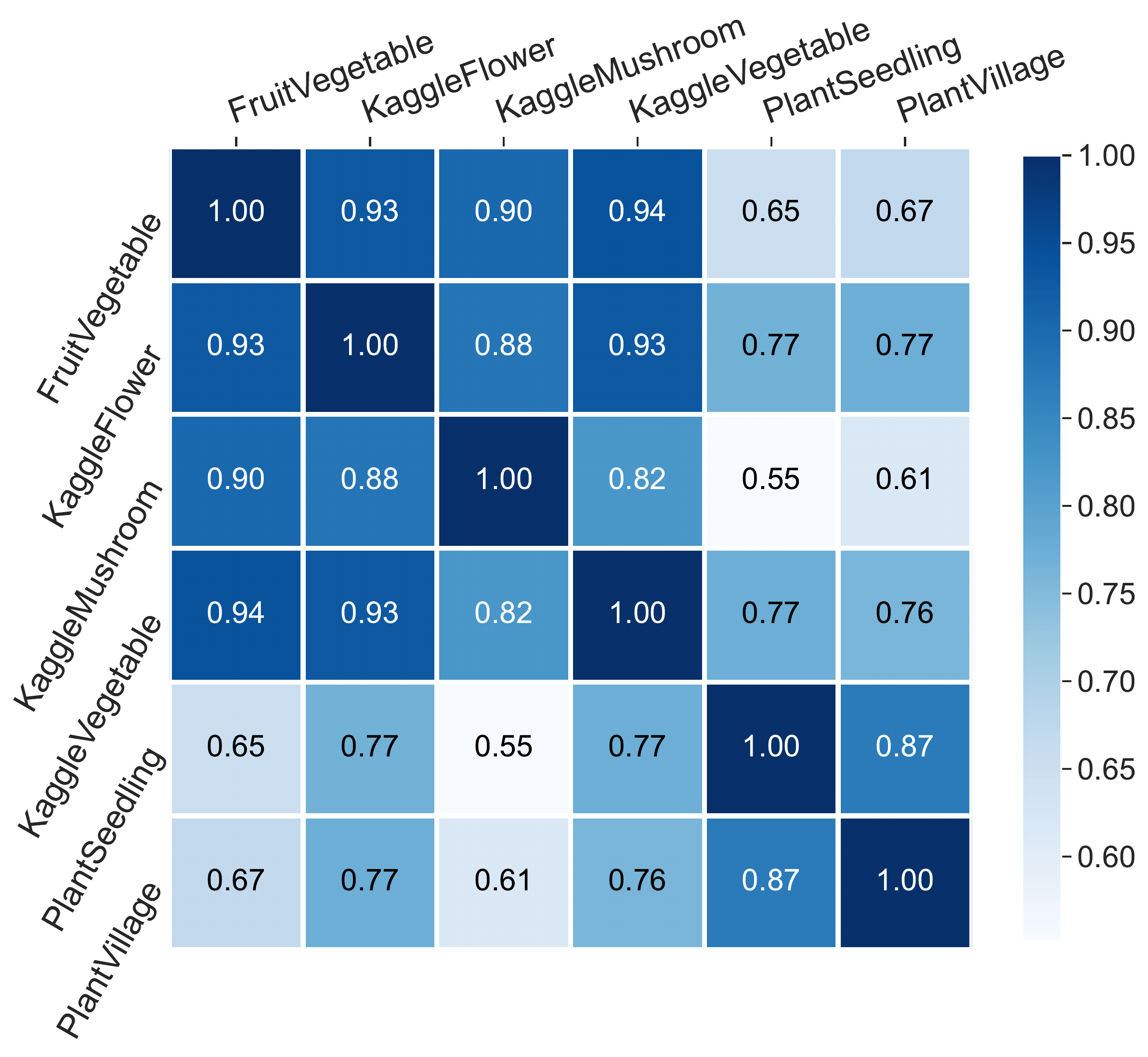}
		\caption{$\boldsymbol{S}_\text{mt}$}
		\label{fig:corr_map_d}
	\end{subfigure}

	\caption{Visualization of task correlation maps on Plant-6. See main text for the description of each map.}
	\label{fig:corr_map}
\end{figure*}

\textbf{Parameter Count and Training Time.} The parameter count and training time of different prompt tuning methods are listed in the second and third column in Table~\ref{tab:main_numerical_results}, respectively. The number of tunable parameters of SoftCPT-NATA is about 8M, which is larger than CoOp and other related methods. However, the relatively large parameter count of SoftCPT is not a disadvantage. The reasons are manifold. First, SoftCPT can maintain the training speed even though there are lots of parameters being optimized. Second, tunable parameter count has no negative impact on the inference speed and storage usage as the text features of class names can be pre-computed. Third, the number of parameters of SoftCPT-NATA does not increase with the number of categories or tasks. By contrast, the parameter count of CoOp-CA, ProGrad, KgCoOp and PLOT increases with the task number and the parameter count of CoOp-CS increases with the class number.

\textbf{Per-task Results.} The per-task results on four datasets are shown in Fig.~\ref{fig:mtcv1_per_task_results}, Fig.~\ref{fig:rsv1_per_task_results} and Fig.~\ref{fig:fashionv1_per_task_results}. It is evident that the performance of linear probe is quite unstable. On Caltech101, DTD, FGVCAircraft, KaggleVegetable and shoe style, it is one of the best performing methods. However, on Food101, Oxford-Pets, SUN397, UCF101, FruitVegetable, WHURS19 and sleeve length, it manifests the worst performance.

For fine-grained tasks, such as Flowers102, PlantSeedling and PlantVillage, CoOp-CS is better than CoOp-CA, while for coarse-grained tasks, such as Caltech101 and SUN397, CoOp-CA is better. This observation is consistent to that in CoOp~\cite{CoOp}. 

As for CoOp-MT, the growth rate of score with more training samples slows down compared to other methods, for example on DTD and EuroSAT, which implies the hard prompt sharing must be a bottleneck of activating the capability of pre-trained models. 

SoftCPT-NATA and SoftCPT-NATS exhibit consistently excellent performances on most tasks besides some fine-grained tasks, such as FGVCAircraft, KaggleMushroom and PlantVillage. To enhance performance on these fine-grained tasks, a promising approach would be to incorporate class-specific information in a more strategic manner. Nevertheless, how to trade off the performances between coarse-grained tasks and fine-grained tasks is a challenging yet interesting question.

\subsection{Visualization}
We visualize the task correlation map on Plant-6. Let us denote the aforementioned score matrix computed by transferring CoOp's prompts as $\boldsymbol{S}$. The normalized score matrix is given by $\boldsymbol{S}'=\boldsymbol{S}*\text{diag}(\boldsymbol{d})$, where $\text{diag}(\cdot)$ is an operator to construct diagonal matrix based on a vector, and the item of $\boldsymbol{d}$ is the inverse of the corresponding value on the diagonal position of $\boldsymbol{S}$. The symmetric task similarity matrix $\boldsymbol{S}_\text{oracle}$ is the mean between $\boldsymbol{S}'$ and its transpose. Let $\boldsymbol{S}_\text{st}$ ($\boldsymbol{S}_\text{mt}$) denotes the single-task (multi-task) task similarity matrix, whose elements are defined as pairwise cosine similarity between two corresponding prompt features, which are computed by encoding the prompt contexts in CoOp-CA (SoftCPT-NATA) with CLIP's text encoder. The correlation maps are shown in Fig.~\ref{fig:corr_map}. The correlation coefficient between upper triangle of $\boldsymbol{S}_\text{oracle}$ and $\boldsymbol{S}_\text{st}$ is -0.03, and that between $\boldsymbol{S}_\text{oracle}$ and $\boldsymbol{S}_\text{mt}$ is 0.32. This means that the proposed method can better reflect the task relatedness.

\section{Conclusion}
\label{sec:conclusion}
Vision-language models (VLMs) are recently adapted to few-shot image recognition tasks by optimizing prompt context. However, existing works of prompt tuning target on single-task learning, ignoring the relation between downstream tasks. Therefore, it is necessary to investigate if introducing multi-task learning to prompt tuning of VLMs is feasible and effective. 

For this aim, we proposed the soft context shared prompt tuning, SoftCPT, which adopts a shared meta network across all tasks to generate the context of a task based on its task description text. The meta network involves extracting task features by pre-trained language model and transforming the task features to needed context by a linear sub-network. The meta network was trained on the joint training set of multiple tasks. To validate the effectiveness of SoftCPT, a series of experiments were conducted. By analyzing the experimental results, we found that: 1) SoftCPT surpasses single-task based CoOp and the simple hard sharing of context, implying that multi-task learning if used in a proper manner does help to improve the overall performance of prompt tuning on all tasks; 2) SoftCPT achieves small improvements on generalized dataset but more significant improvements on specialized datasets, implying that to fully leverage the advantages of multi-task prompt tuning, it is better for the tasks to have strong relationship; 3) The proposed method using multi-task learning can achieve better results compared to state-of-the-art methods especially on specialized datasets. Based on these observations, we can conclude that multi-task learning should be helpful for prompt tuning of VLMs.

Our findings have some significant meanings. Firstly, in real-world applications, it is very common that different tasks have some relation. By leveraging the proposed method, it is able to further enhance the recognition accuracy, demonstrating its practical value. Secondly, our findings will help the computer vision community to realize the importance of multi-task prompt learning, thereby further promoting the development of combining multi-task learning and prompt learning in the adaption of vision-language models.

The proposed method is not without flaws. First, although SoftCPT can still improve the overall accuracy on the generalized dataset General-10, the accuracy boost is less significant. This is mainly due to the weak relation between different tasks. Second, SoftCPT uses a linear sub-network to convert task features to context. However, this linear layer has a relatively large number of parameters.

In the future, our primary focus will be on further enhancing multi-task prompt tuning, particularly for generalized datasets. Additionally, we aim to further compress the parameter count, which could potentially improve the data efficiency of prompt tuning.

{\small
\bibliographystyle{unsrt}
\bibliography{main}
}

\end{document}